%% file: thesis.tex
\documentclass[12pt]{gatechthesis}

\title{Efficient Adaptation of Reinforcement Learning Agents to Sudden Environmental Change}
\author{Jonathan Clifford Balloch}
\approvaldate{November 21, 2024}
\school{School of Interactive Computing}
\department{College of Computing}
\degree{Doctor of Philosophy} 
\bibliography{references}

\usepackage{amsmath,amssymb,amsfonts,amsthm}
\usepackage{hyperref}
\usepackage{multirow}
\usepackage{tikz}
\usepackage[ruled,vlined,noend]{algorithm2e}
\usetikzlibrary{positioning}
\usepackage{enumerate}
\usepackage{comment}
\excludecomment{toexclude} 

\newtheorem{definition}{Definition} 

\newtheorem{theorem}{Theorem}[section]

\SetKwComment{Comment}{$\triangleright$\ }{}

\makeatletter
  \newcommand\reduline{\bgroup\markoverwith{\textcolor{red}{\rule[-0.5ex]{2pt}{0.4pt}}}\ULon}
  \UL@protected\def\blueuwave{\leavevmode \bgroup 
    \ifdim \ULdepth=\maxdimen \ULdepth 3.5\p@
    \else \advance\ULdepth2\p@ 
    \fi \markoverwith{\lower\ULdepth\hbox{\textcolor{blue}{\sixly \char58}}}\ULon}
  \UL@protected\def\yellowdotuline{\leavevmode \bgroup 
    \UL@setULdepth
    \ifx\UL@on\UL@onin \advance\ULdepth2\p@\fi
    \markoverwith{\begingroup
       \lower\ULdepth\hbox{\kern.06em \textcolor{yellow}{.}\kern.04em}%
       \endgroup}%
    \ULon}
  \UL@protected\def\greendashuline{\leavevmode \bgroup 
    \UL@setULdepth
    \ifx\UL@on\UL@onin \advance\ULdepth2\p@\fi
    \markoverwith{\kern.13em
    \vtop{\color{green}\kern\ULdepth \hrule width .3em}%
    \kern.13em}\ULon}
\makeatother

\begingroup
\definecolor{rulecolor}{rgb}{1.0, 0.95, 0.95}

\usepackage{environ}

\newcommand{\jb}[1]{\textcolor{orange}{\small [Jonathan: #1]}}
\newcommand{\julia}[1]{\textcolor{green}{\small [Julia: #1]}}

\newcommand\Mark[1]{\textcolor{blue}{\small [Mark: #1]}}
\newcommand\MarkLeft[1]{\textcolor{blue}{$\leftarrow$Mark: #1}}

\begin{document}

\makeTitlePage{December}{2024}
\input{approvalPage.tex}
\makeEpigraph{Non progredi est regredi}{}
\makeDedication{
For my wife Lena, my precious daughter Mariana, my parents Susan and Hugh, my family, my dog WALL-E
You are my world, and without you I would be lost.}

\begin{frontmatter}
    \input{acknowledgments}
    \makeTOC
    \makeListOfTables
    \makeListOfFigures
\input{abbrevs}

\input{summary}
\end{frontmatter}

\begin{thesisbody}
    \input{chapters/1_intro_chapt}

\input{chapters/2_prelims_chapt}

    \input{chapters/3_novelty_chapt}
    \input{chapters/4_transx_chapt}
    \input{chapters/5_dops_chapt}
    \input{chapters/6_worldcloner_chapt}

    \input{chapters/7_cbwm_chapt}

    \input{chapters/8_conclusion_chapt}
    \input{chapters/appendix.tex}
    \makeBibliography
    \input{vita}
\end{thesisbody}

\end{document}

%% file: approvalPage.tex

\begin{approvalPage}{6}


\committeeMember{Dr. Mark O. Riedl (Advisor)}{School of Interactive Computing}{Georgia Institute of Technology}
\\
\committeeMember{Dr. Seth A. Hutchinson}{School of Interactive Computing}{Georgia Institute of Technology}
\\
\committeeMember{Dr. Harish Ravichandar}{School of Interactive Computing}{Georgia Institute of Technology}
\\
\committeeMember{Dr. Sehoon Ha}{School of Interactive Computing}{Georgia Institute of Technology}
\\ 
\committeeMember{Dr. Michael L. Littman}{Computer Science Department}{Brown University}

\end{approvalPage}

%% file: acknowledgments.tex
\begin{acknowledgments}

I would like to thank the members of my thesis committee for their help in preparation of this work -- my advisor Mark Riedl, without whom I would be utterly doomed; Harish Ravichandar, who has been a fount of insight and guidance starting in my second year as a postdoc and through all the changes I experienced; Seth Hutchinson, whose perspective is invaluable and with whom I have great discussions (when I can catch him!); Sehoon Ha, whose work I have had to admire only from afar until I was fortunate enough to convince him to be on this committee; and Michael Littman, who I met before all the rest way back in 2015 when I toured Brown as a prospective PhD student and who has always been an inspiration to me for his contributions in machine learning and especially reinforcement learning. Thank you all.

Special thanks are due to the friends and colleagues who made this work possible. Zhiyu Lin, you have been there in person, as a collaborator, and as someone just to bounce ideas off of for many if not most of the times I needed you. Julia Kim, you are my light of Earendil, guiding me and lighting dark places when all other lights go out; without your help I would have gotten lost many times over. Jessica Inman, Bob Wright, Becky Peng, Upol Ehsan, and Spencer Frazier, thank you all for being such supportive collaborators. James Smith, Andrew Silva, and Cusuh Han, I couldn't ask for better friends, collaborators, and sounding boards; thank you for putting up with me. To everyone else I work with in EI+HCAI Lab, and to those in Irfan Essa's EYE Lab, Sonia Chernova's RAIL Lab, and RoboGrads: one hundred times thank you. You are the community that made this possible for me. 


\end{acknowledgments}

%% file: abbrevs.tex
\newacronym{starlabs}{STAR Labs}{Scientific and Technological Advanced Research Laboratories}
\newacronym{uv}{UV}{ultraviolet}
\newacronym{s}{s}{state}
\makeListOfAcronyms

%% file: summary.tex
\begin{summary}

Real-world autonomous decision-making systems, from robots to recommendation engines, must operate in environments that change over time. While deep reinforcement learning (RL) has shown an impressive ability to learn optimal policies in stationary environments, most methods are data intensive and assume a world that does not change between training and test time. As a result, conventional RL methods struggle to adapt when conditions change. This poses a fundamental challenge: how can RL agents efficiently adapt their behavior when encountering novel environmental changes during deployment without catastrophically forgetting useful prior knowledge? This dissertation demonstrates that efficient online adaptation requires two key capabilities: (1) prioritized exploration and sampling strategies that help identify and learn from relevant experiences, and (2) selective preservation of prior knowledge through structured representations that can be updated without disruption to reusable components.

We first establish a formal framework for studying online test-time adaptation (OTTA) in RL by introducing the Novelty Minigrid (NovGrid) test environment and metrics to systematically assess adaptation performance and analyze how different adaptation solutions handle various types of environmental change. We then begin our discussion of solutions to OTTA problems by investigating the impacts of different exploration and sampling strategies on adaptation. Through a comprehensive evaluation of model-free exploration strategies, we show that methods emphasizing stochasticity and explicit diversity are most effective for adaptation across different novelty types. Building on these insights, we develop the Dual Objective Priority Sampling (DOPS) strategy. DOPS improves model-based RL adaptation by training policy and world models on different subsets of data, each prioritized according to the different learning objectives. By balancing the trade-off between distribution overlap and mismatched objectives, DOPS achieves more sample-efficient adaptation while maintaining stable performance. 

To improve adaptation efficiency with knowledge preservation, we develop WorldCloner, a neurosymbolic approach that enables rapid world model updates while preserving useful prior knowledge through a symbolic rule-based representation. 
WorldCloner demonstrates how structured knowledge representation can dramatically improve adaptation efficiency compared to traditional neural approaches. 
Finally, we present Concept Bottleneck World Models (CBWMs), which extend these insights into an end-to-end differentiable architecture. 
By grounding learned representations in human-interpretable concepts, CBWMs enable selective preservation of unchanged knowledge during adaptation while maintaining competitive task performance. 
CBWMs provide a practical path toward interpretable and efficient adaptation in neural RL systems.

Together, these contributions advance both the theoretical understanding and practical capabilities of adaptive RL systems. 
By showing how careful exploration and structured knowledge preservation can enable efficient online adaptation, this work helps bridge the gap between current RL systems and the demands of real-world applications where change is constant and adaptation essential.

\end{summary}

%% file: chapters/1_intro_chapt.tex
\chapter{Introduction }
\label{chapt:intro}


People often imagine a future in which intelligent, autonomous agents such as robots can help us throughout our daily lives, not just with constrained, isolated tasks. 
In the last decade, deep reinforcement learning (RL) has been used to develop increasingly capable agents for solving complex decision-making tasks such as board games~\cite{silver2017alphagozero,schrittwieser2020muzero}, video games~\cite{vinyals2019grandmaster,berner2019dota,badia2020agent57}, recommender systems~\cite{afsar2022reinforcement}, industrial HVAC control~\cite{evans2016deepmind,chen2019gnu}, and tokamak control in nuclear fusion research~\cite{degrave2022magnetic}. 
In many of these applications, RL agents outperform planning-based agents and classic control agents, and in some cases even outperform humans reliably~\cite{badia2020agent57}. 
As such, RL shows great promise for developing intelligent agents that interact with the world. 

Many real world problems occur in ``open-world'' environments, where dynamics, objects, and the behavior of other agents can change in unexpected ways. 
In an imaginary future full of helpful intelligent agents, the decision making policies of these agents need to be able to accommodate these changes just as humans do.
Consider the task of commuting from home to work. 
People often take the same general route to work everyday;
after some initial practice and guidance, we can be confident that we are on the optimal route. 
However, what if something changes our typical route, such as the start of a new construction project? 
Without outside help, if the commuter never tried any other routes to work, could not remember the other routes, or could not differentiate the disrupted and unaffected parts of the route, adapting to the change would be highly random and inefficient. 
What if an intersection that previously took many minutes or light cycles to clear has become more efficient? 
If the commuter never considers other routes to work and never adapts the route, commuting will be significantly less efficient, wasting time and energy. 


In spite of its recent successes, deep RL solutions remain, like many neural network-based solutions, brittle to shifts in the distribution of inputs and outputs. 
After the widely successful Starcraft RL agent AlphaStar played against humans only a handful of times, the human player MaNa was able to find a strategy for which the AlphaStar bot had no answer~\cite{evans2016deepmind}. 
Similarly, researchers were able to find a strategy that beat KataGo, an open source reimplementation of AlphaGo, 14 out of 15 times~\cite{wang2023adversarial}. 
In neither of these two cases did the RL agent learn to adapt to the new strategies. 
This is not unusual: for most deployed models trained with deep learning, deployment is considered ``test time,'' not ``training time,'' meaning the model is prevented from adapting during deployment. 
As such, neither of these agents was designed explicitly to respond to novel situations by learning during deployment.


In modeling an environment as a stochastic process such, as an Markov decision process (MDP), these novel changes in an environment fall broadly into the umbrella of non-stationary processes.~\cite{littman1991adaptation,sutton2018reinforcement} 
As this dissertation discusses in greater detail in Chapter~\ref{chapt:background}, process non-stationarity can take many forms, and designing learning agents to respond to all of the ways in which an environment can change is often intractable. 
As a result, theoretical solutions for responding to any non-stationary phenomena often require simplifying assumptions about the environment that make them difficult to apply to complex, real-world scenarios.
In an effort to reduce the scope of non-stationary phenomena in this dissertation to study change of more complex systems like robots,  we constrain set of non-stationary environment changes considered to two environments related by a \textit{novelty}:

\begin{definition} 
A \textbf{novelty} is a sudden, unanticipatable, previously unseen change in an interactive environment that represents the transformation from a source domain, task, or environment to a target domain, task, or environment. 
\end{definition}

Fast and sample-efficient response to novel environment changes is always desirable and can be essential, for example, when human safety is involved~\cite{REZAPOUR2021312}. 
All intelligent agents, whether artificial or biological, are capable of responding to change by being \textit{robust}, \textit{adaptable}, or a combination of both. 
Robustness-based solutions are attractive because they are systematically simple to implement. 
Adding robustness to an agent can often be as simple as exposing the agent to a wider variety of possible scenarios during training~\cite{tobin2017domain}. 
By preparing for many variations of a scenario in advance, even if the agent experiences a change in deployment that it never saw during training, it will generally be less sensitive to change. 
However, the changes an agent can learn to be robust to is often limited. 
In complex decision making applications and environments, an intractable amount of data, time, and compute may be required for robustness good enough to make adaptation unnecessary. 
Further, while robustness allows agents to handle changing circumstances with innate behavioral capabilities that apply to more than the original task,  adaptation allows agents to make permanent changes in its behavior that better match the changed environment. 

Adaptation-based solutions are attractive because they are independent of the type of change that may occur. 
This explains why animals such as humans have evolved to be good at adapting to novel circumstances \cite{braun2009learning}. 
What's more, behavioral science and neuroscience show that operant conditioning and reinforcement are critical for adaptation in biological intelligence~\cite{berridge2000reward,reale2001temperament}, suggesting that RL is a potential path forward for adapting AI agents to change. 
The simplest way to adapt models to changing data, as seen in the use of production recommender systems~\cite{45530,VERACHTERT2023100455} and other RL research domains, is often to simply retrain the agent offline once performance has sufficiently dropped. 
These systems can be retrained \textit{tabula rasa}, or ``from scratch,'' but it is desirable to \textit{adapt} the prior model by using it to initialize the new training process, and it is necessary if the retraining data is limited. 
However, even in this simple scenario, adaptation is complex in practice, potentially requiring knowledge about how much data is changing and how much performance drop should trigger retraining. 

Another issue with offline retraining is that it is not always the case that an agent can be taken offline for training.
Taking deployed agents, such as robots, offline for a training update is not always possible.
Moreover, offline training implies that sufficient training data representing the post-novelty environment has already been gathered in advance of training.
If offline data had been collected, it implies that a deployed system was operating suboptimally in a novel environment during the course of that data acquisition, which is undesirable and can be dangerous.
As with intelligent biological agents, online adaptation---also referred to as during ``deployment'' or ``test-time'' adaptation---can be an efficient way to adapt to novelty while interacting with the new environment.
However, adapting online adds additional challenges, chief among them is \textit{catastrophic inference}~\cite{mccloskey1989catastrophic}, also known as catastrophic forgetting.
Catastrophic inference is the complete loss of performance that occurs when a trained parameterized model is trained on data distributed differently than its original training data, such as new training data limited only the changed environment.
Catastrophic inference occurs in offline adaptation as well; however, in online adaptation there is a greater need for the agent to reduce or avoid it altogether as the performance drop from an interacting agent can be more unsafe.
As a result of these complexities, adapting task-specific learned agents has great potential, but in practice is very challenging and sample-inefficient.

\section{Thesis Statement}


This dissertation focuses on demonstrating that adapting models to changing worlds online with reinforcement learning requires evaluating and improving the way RL algorithms approach exploration and sampling, and the way the knowledge from model priors are represented and adapted. 
Thus, this dissertation investigates the following thesis statement: 

\begin{quote}
    \noindent\textbf{
    To efficiently adapt online to changes in the environment, reinforcement learning agents must (1) use exploration and sampling strategies that prioritize task-agnostic interactions and learning data to reduce distribution shift, and (2) identify and selectively preserve reusable prior knowledge in symbolic and learned representations.
    }
\end{quote}

\noindent This thesis statement can be more effectively investigated by further decomposing it into two subclaims:
\begin{enumerate}
    \item If an agent can identify the nature of environmental change through exploration, then the agent is more likely to rapidly adapt to the new optimal goal trajectory. 
    \item If an agent can distinguish which parts of its representations are consistent before the environment changes, then the agent can adapt more efficiently than without prior knowledge by limiting the amount representations change during adaptation.
\end{enumerate}


The work of this dissertation validates these two subclaims by researching solutions to two corresponding subproblems. 
The first subproblem focuses on framing exploration in RL toward discovering specific information important to efficient adaptation.
Exploration in conventional RL problems with a stationary MDP largely serves as a source of diversification in sampling, as the agent's greedy pursuit of reward is how the environment is sampled otherwise. 
A discussion of conventional uses of exploration is covered in Section~\ref{sec:background:rl:sample}.
This dissertation demonstrates that exploration can benefit efficient adaptation RL in two major ways: 
(1)~exploration in model-free RL can improve an agent's ability to adapt by incentivizing diversity and stochasticity, which is discussed in Chapter~\ref{chapt:transx}, and (2)~exploration in model-based RL can improve an agent's ability to transfer by sampling the data appropriate to the different learning objectives of the policy and world model, which is discussed in Chapter~\ref{chapt:dops}.

The second subproblem focuses on selective reuse of prior knowledge for transfer learning. 
In deep RL, prior knowledge is not inherently preserved when adapting a prior policy or model to a new task. 
Moreover, not all prior knowledge \textit{should} be preserved: ``incorrect'' prior concepts related to the novelty need to change while preserving unrelated concepts. 
The forgetting that occurs in neural networks' entangled latent knowledge is a source of great inefficiency. 
Most of the evidence for adaptation behavior in humans shows that knowledge reuse is critical to success~\cite{reale2001temperament,braun2009learning}. 
Focusing on model-based reinforcement learning where the ``world model'' is trained to represent the environment dynamics, this dissertation examines two knowledge-preserving representations for efficient adaptation: (1) preserving knowledge in the world model by representing knowledge symbolically, and (2) enforcing disentanglement of knowledge by structuring and grounding a latent bottleneck in the world model. 
As discussed in Chapter~\ref{chapt:knowledge}, knowledge can be structured in a way that makes it well suited to preservation without necessarily trading off task performance. 
Finally in Chapter~\ref{chapt:cbwm}, the insights of Chapter~\ref{chapt:knowledge} are applied to the development of an end-to-end differentiable deep RL world model, where a bottleneck architecture enforces disentanglement of world model knowledge. 
Our results show that this disentangled world model can better facilitate knowledge reuse with little to no impact on overall performance, and is a means of studying how much knowledge forgetting occurs in adaptation.

\subsection{Outline}

I will begin by presenting the background information necessary to situate the contributions of this dissertation in prior work in Chapter~\ref{chapt:background}. 
I will then discuss my work in formalizing and evaluating the study of online test time adaptation to novelty in sequential decision making problems in Chapter~\ref{chapt:novgrid}, which serves as a foundation for the remaining chapters of the dissertation. 

Following this, the core contributions of the dissertation are split according to the described subproblems of exploration (Chapters~\ref{chapt:transx} and \ref{chapt:dops}) and knowledge preservation (Chapters~\ref{chapt:knowledge} and \ref{chapt:cbwm}).
Starting with the investigation of the impacts of exploration, Chapter~\ref{chapt:transx} describes the work comparing the impact of different characteristics of exploration on online task transfer problems, and how the effects of these characteristics vary depending on the type of novelty and nature of the environment. 
Chapter~\ref{chapt:dops} extends the findings of Chapter~\ref{chapt:transx} by examining the challenges of prioritized sampling of observations in model-based RL, and proposes a solution for improving the adaptation efficiency of  model-based agents. 

Transitioning to the investigation of preserving unaffected knowledge to improve adaptive efficiency, Chapter~\ref{chapt:knowledge} describes the work on enabling a world model to only change necessary knowledge in the face of novelty by implementing a neuro-symbolic model-based RL approach, where the policy is a neural network implementation while the world model is represented as a rules-based induction model. 
Chapter~\ref{chapt:cbwm} then describes the work on representing knowledge as a supervised ``context-bottleneck'' through which decision making gradients must pass. 
By using a context bottleneck enforced knowledge disentanglement, the results allow us to quantify how much knowledge is preserved during adaptation. 

Finally in Chapter~\ref{chapt:conclusion}, I review the contributions and impact of the research efforts outlined in this dissertation and suggest several promising future research directions that seem most exciting in light of the dissertation's contributions.

%% file: chapters/2_prelims_chapt.tex
\chapter{Situating the Work}\label{chapt:background}

This chapter provides the technical background for understanding the prior and proposed work of this thesis, and an overview of related work investigating similar problems and solutions to this work. 
Specifically, this chapter provides an overview of both foundational and state-of-the-art reinforcement learning, and provides additional details and context for the specific aspects of reinforcement learning that this dissertation examines. 
Additionally, this chapter will provide a technical foundation of transfer learning in deep neural networks, the subdomain of test time adaptation, and the specific challenges associated with test time adaptation when learning from data in non-stationary sequential decision making environments. 

\section{Reinforcement Learning}\label{sec:background:rl}

Finding optimal solutions to sequential decision making problems interactively is fundamental to the development of intelligent autonomous agents. 
Unlike non-interactive machine learning where optimal solutions are learned by trying to predict labels (supervised learning) or features (unsupervised learning) of a dataset, reinforcement learning finds optimal solutions by interacting with an environment and maximizing expected future reward in an environment where more reward is associated with better task performance. 
This makes reinforcement learning both powerful and broadly applicable. 
However, successful application of reinforcement learning to a specific problem depends on implementation details including whether one learns a model of the environment or of the policy directly, whether the policy being updated is the same policy interacting with the environment, which interactions are prioritized and which samples are used for learning, whether learning occurs online or offline, and how to trade-off using the current solution with searching for better, yet undiscovered solutions. 

\subsection{Fundamentals of Reinforcement Learning}
\label{sec:background:rl:fund}

Reinforcement learning fundamentally assumes that there exists a repeated interaction between \textit{agents} and \textit{the environment}, and a \textit{task} associated with the agent maximizing a specific reward in that environment.

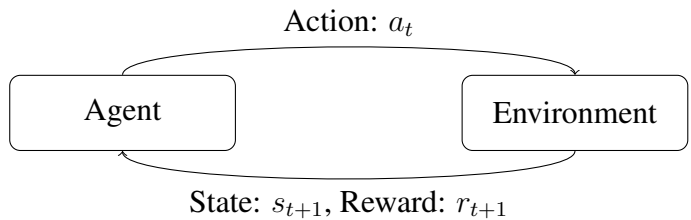
\begin{figure}[ht]
\centering
\begin{tikzpicture}[node distance=3cm, auto]
    \node [draw, rounded corners, minimum width=3cm, minimum height=1cm] (agent) {Agent};
    \node [draw, rounded corners, minimum width=3cm, minimum height=1cm, right=of agent] (env) {Environment};
    \draw [->] (agent) .. controls +(up:1cm) and +(up:1cm) ..  node[above] {Action: $a_t$} (env);
    \draw [->] (env) .. controls +(down:1cm) and +(down:1cm) .. node[below] {State: $s_{t+1}$, Reward: $r_{t+1}$} (agent);
\end{tikzpicture}
\caption{The agent-environment interaction that is fundamental to reinforcement learning. }
\label{fig:rl_loop}
\end{figure}

Reinforcement learning typically assumes that the learning problem, also referred to as the environment, is modeled as a specific type of stochastic process called a Markov decision process (MDP). 
An MDP is represented as a 4-tuple.
\begin{equation}
    M=\langle \mathcal{S}, \mathcal{A}, \mathcal{R}, \mathcal{P}\rangle 
\end{equation}
\label{eq:mdp}

These quantities are defined as follows:
\begin{itemize}
    \item $\mathcal{S}$ is the space of environment \textit{states}.
    \item $\mathcal{A}$ is the space of \textit{actions} that an agent can take. 
    \item $\mathcal{R}: \mathcal{S} \times \mathcal{A} \rightarrow \mathbb{R}$ is the \textit{reward function} that maps states, and sometimes actions, to a scalar reward that RL solutions seek to maximize. 
    \item  $\mathcal{P}: \mathcal{S} \times \mathcal{A} \times \mathcal{S} \rightarrow [0,1]$ is the \textit{transition function}, also referred to as the dynamics model or world model, and defines the distribution of next states conditioned on the current state and selected action. 
\end{itemize}
Often an alternative definition of an MDP is the 5-tuple $M=\langle \mathcal{S}, \mathcal{A}, \mathcal{R}, \mathcal{P}, \gamma\rangle$ to include $\gamma \in [0,1)$, which is the discount factor that determines the importance of future rewards. 
We exclude $\gamma$ for the remainder of this dissertation because, while it is necessary for some aspects of theoretical guarantees of reinforcement learning approaches, its value is in practice often tuned depending on the solution. 

The key assumptions in the MDP formulation are that the environment is fully observable (the agent has complete information about the current state), the environment is Markovian (the next state depends only on the current state and action, not on past states or actions), and that the environment is \textit{stationary} meaning that the transition and reward functions do not change over time. 
While these assumptions may seem restrictive, many real-world problems can be approximated as MDPs and still solved by reinforcement learning algorithms, and many more can be approximated by variants on MDPs where some of these assumptions can be relaxed, such as the observability assumption with partially-observable MDPs (POMDPs). 
This work focuses on investigating learning problems where the stationarity assumption is relaxed and the MDP is non-stationary due to a large, unexpected change in the environment.

The goal in RL is to find an optimal \textit{policy} $\pi^*: \mathcal{S} \rightarrow \mathcal{A}$ that, given any state $s_t \in \mathcal{S}$ at time $t$, can select the action that maximizes the likelihood of discounted cumulative reward, also known as the \textit{return}, $G_t$:
\begin{align*}
    \pi^* &= \arg\max_{\pi} \mathbb{E}\left[\sum_{t=0}^\infty G_t \mid s_0, \pi\right] \\
    G_t &= r_{t+1} + \gamma r_{t+2} + \dots = \sum_{k=0}^{\infty} \gamma^k r_{t+k+1} \text{ where } r_t = \mathcal{R}(s_t,a_t). 
\end{align*}
\textit{Credit assignment}, the problem of figuring out how much certain states and actions contribute to the final reward, is one of the most challenging aspects of learning models for sequential decision making. 
Key to addressing credit assignment in reinforcement learning is defining how valuable it is to visit or take action in a state, defined respectively by the state value function $V_{\pi}(s_t)$---or simply \textit{value function}---and state-action value function---or simply \textit{Q-function} $Q_{\pi}(s_t, a_t)$:
\begin{align*}
    V_{\pi}(s_t) &= \mathbb{E}_{\pi}[G_t \vert s_t \in \mathcal{S}] \\
    Q_{\pi}(s_t, a_t) &= \mathbb{E}_{\pi}[G_t \vert s_t \in \mathcal{S}, a_t \in \mathcal{A}].
\end{align*}
The policy, value function, and Q-function quantities are all related to each other through the \textit{Bellman Expectation Equations}, which describe the recursive functions by which these two value quantities update with respect to discounted future rewards:
\begin{align*}
V_{\pi}(s) &= \sum_{a \in \mathcal{A}} \pi(a \vert s) Q_{\pi}(s, a) \\
Q_{\pi}(s, a) &= R(s, a) + \gamma \sum_{s' \in \mathcal{S}} P_{ss'}^a V_{\pi} (s') 
\end{align*}

As a result, many RL methods optimize for the optimal policy $\pi^*$ indirectly by finding a value function associated with the optimal policy, $Q_{\pi}^*(s,a)$ or $V_{\pi}^*(s) = \max_{a \in \mathcal{A}} Q_{\pi}^*(s,a)$, and evaluating to find the best action at a given state (e.g. $\pi^*(s) = \underset{a \in \mathcal{A}}{\mathrm{argmax}} Q_{\pi}^*(s,a)$). 
If the transition function $\mathcal{P}$ is known, dynamic programming can be applied iteratively to find the optimal solution using techniques such as value or policy iteration~\cite{sutton2018reinforcement}.
However, in many cases, information about the dynamics and environment are unknown; in these situations, we can either directly learn the optimal policy or value function without a world model by approximating its impact on future rewards, or we can learn a world model. 
As we will describe in Section~\ref{sec:background:rl:mfmb}, the world model and policy can also both be learned.

\subsection{Model-free and Model-based Approaches to Reinforcement Learning}
\label{sec:background:rl:mfmb}

When the dynamics model of the environment is not known, the possible reinforcement learning approaches can be separated into two main categories: (1)~\textit{model-free} and (2)~\textit{model-based} reinforcement learning. 
In model-free RL, the optimal policy is learned directly by learning which action maximizes the expected return in the environment. 
In model-based reinforcement learning, the agent's experiences are primarily used to learn the world model; the learned world model can then be used to find the optimal policy, or the optimal policy can be learned concurrently with the help of world model learning.

\subsubsection{Model-Free Reinforcement Learning}
\label{sec:background:rl:mfmb:mf}

Q-learning is one of the most fundamental approaches to model-free reinforcement learning, and the most common ``value-based'' method, meaning and RL method that learn an optimal value function.
In Q-learning, the return is estimated with Temporal Difference (TD) learning~\cite{sutton1988td,tsitsiklis1996analysis}, where the update to the Q-function sum between the estimated next-state return $R_{t+1}+\gamma \max _{a^{\prime} \in \mathcal{A}} Q\left(s_{t+1}, a\right)$, also called the ``TD-target,'' and the estimated value $Q\left(s_t, a_t \right)$~\cite{sutton2018reinforcement}. 
Rewriting to isolate the mixing weight $\alpha$, more commonly referred to as the \textit{learning rate}, TD-update becomes the the change to the estimated value from the ``TD-error'' $\delta_{TD}$: 
\begin{equation}
Q(S_t, A_t) \leftarrow Q(S_t, A_t) + \alpha \underset{\mathrm{TD-Error}}{(R_{t+1} + \gamma \max_{a \in \mathcal{A}} Q(S_{t+1}, a) - Q(S_t, A_t))}
\label{eq:background:td}
\end{equation}
TD learning approximates the exact return with respect to existing value estimates in a practice known as ``bootstrapping''~\cite{sutton2018reinforcement}. 
Although more biased than learning from full episodes, bootstrapping can be far more sample-efficient, especially if rewards are sparse, and allows learning from individual transitions.   
One main advantage of estimating the return with TD learning is that it only requires a single transition, as opposed to other approximation methods, such as Monte Carlo methods, that requires a terminating sequence of many transitions (also called an ``episode''). 
Q-learning also has the benefit of being able to learn ``off-policy,'' meaning that it does not choose the next action according to the current policy, but instead according to the action with the maximum next Q-value. 
These two advantages mean that Q-learning is able to, and in fact benefits from, maintaining a large buffer prior environment transitions and update its Q-function on a random sample of this past experience. 
This sampling advantage and its trade-offs are explained in greater detail in Section~\ref{sec:background:rl:sample}.

Today, in \textit{deep} reinforcement learning, function approximators (most often a multilayer neural network composed of parameters $\theta$) are used to approximate the desired function or distribution. 
The parameters of the neural network are most often updated by minimizing a differentiable loss function with stochastic gradient descent (SGD). 
In a Deep Q-Learning Network (DQN)~\cite{mnih2015human}, a neural network is used to approximate the Q-function, $Q_* \approx Q(S, A ; \theta)$, with the loss function: 

\begin{equation}
\mathcal{L_Q}(\theta)=\mathbb{E}_{\left(s, a, r, s^{\prime}\right) \sim U(B)}\left[\left(r+\gamma \max _{a^{\prime}} \hat{Q}\left(s^{\prime}, a^{\prime}\right)-Q(s, a ; \theta)\right)^2\right]
\end{equation}

\noindent where $U(B)$ is a uniform distribution for sampling random transition tuples from an experience replay buffer $B$.  
Critically, $\hat{Q}(\cdot)$ is a frozen copy of the Q-function called the ``target network,'' which acts as a surrogate target Q-function and is updated less frequently to reduce correlations between the action selection and TD-target estimate Q-functions~\cite{mnih2015human}. 

One of the main downsides of value-based methods is that they mostly assume deterministic policies. 
The Q-learning update function specifically is also intractable to compute over continuous action spaces, and DQN is a classic example of the ``deadly triad'' in RL~\cite{sutton2018reinforcement,van2018deadly}, where a combination of off-policy updates, bootstrapping, and function approximation leads to unpredictable instability in the learning process.
``Policy-based'' model-free methods offer an alternative that, as the name suggests, approximate the optimal policy $\pi^*_{\theta}(a)$ instead as a distribution $\pi^*_{\theta}(s \vert a)$ for which a return-based objective $J$ can be defined. 
In the undiscounted case, this makes the objective function at time $t$ for updating parameters $\theta$:
\begin{equation*}
        J(\theta) = G_t = \sum_{k=0}^{\infty} r_{t+k+1} = \sum_{k=0}^\infty  V_{\pi_{\theta}}(s_t) \pi_{\theta}(a \vert s_{t+k+1})
\end{equation*}
Most policy-based methods are \textit{policy gradient} methods which depend on the \textit{policy gradient theorem}~\cite{williams1992simple}. 
The policy gradient theorem states that, for the on-policy case where both state and action distributions follow the policy being learned, the gradient of $J(\theta)$ over an entire episode can be approximated in expectation as:
\begin{align*}
\nabla_\theta J(\theta) 
&\propto \sum_{s \in \mathcal{S}} p_{\pi}(s) \sum_{a \in \mathcal{A}} Q_{\pi}(s, a) \nabla_\theta \pi_{\theta}(a \vert s)  \\
&= \mathbb{E}_\pi \left[Q_\pi(s, a) \nabla_\theta \ln \pi_\theta(a \vert s)\right]
\end{align*}
Here, $p_{\pi}(s)$ denotes the on-policy state distribution when following policy $\pi$. 

The REINFORCE algorithm~\cite{williams1992simple}, one of the original policy gradient methods, calculated the gradient by estimating $Q_\pi(s, a)$ using complete Monte Carlo rollouts. 
REINFORCE, however, can produce high-variance estimates the harms the efficiency of learning. 
The two most common ways to reduce this variance is to (1) subtract from the $Q_\pi(s, a)$ surrogate return a \textit{baseline}, a corrective term usually assigned to be the value function $V(s)$, and (2) to learn a function for the surrogate return $A_{\omega}(s,a) = Q_\pi(s, a)-V(s)$ with TD learning along with the policy. 
$A(s,a)$ is referred to as the \textit{advantage}, and this method of learning both the policy and the value function is called Actor-Critic, where the learned value function is called the critic.

In this dissertation, the model-free methods primarily considered are on-policy actor-critic policy gradient methods, as they naturally work well with deep neural network function approximators and in a wide variety of environment types. 
However, since on-policy actor-critic methods cannot make use of an experience replay buffer in the same way as DQN, off-policy policy gradient methods and exploration play a very important role in their success, as discussed further in Section~\ref{sec:background:rl:sample:expl}. 

There have been many innovations in the space of deep policy gradients. 
By having distributed parallel workers, algorithms like A2C, A3C~\cite{mnih2016asynchronous}, and IMPALA~\cite{espeholt2018impala} very effectively reduce the high variance of Monte Carlo rollouts by averaging the gradients from multiple ``worker'' actors to update to a central actor policy, and then after some time reset the worker policies back to the central policy. 
By changing a policy gradient method to off-policy---using a different sampling policy than the target policy---methods such as DDPG~\cite{lillicrap2016ddpg}, D4PG~\cite{barth2018distributed}, TD3~\cite{wu2017scalable}, and SAC~\cite{haarnoja2018soft}  greatly improve the sample efficiency and performance of continuous control tasks by taking advantage of the deterministic policy gradient theorem, using experience replay buffer, and ``soft'' updates that constrain how fast functions can change. 

The recent state of the art policy methods that continue to set the standard across the widest set of RL problems, however, are methods that pursue a similar idea of constraining the speed of change using the notion of a \textit{trust region}. 
For on-policy actor critic methods with multiple distributed workers, the notion of a trust region comes from the functional reality that even on-policy distributed RL methods with multiple workers experience some differences, or ``staleness,'' between the worker and an older central policy. 
As a result, trust region methods like TRPO~\cite{schulman2015trust} and ACKTR~\cite{wu2017scalable} show that distributed actor critic frameworks are greatly stabilized (and therefore reach improved policies) when the optimization function is constrained by limiting the amount the worker and central policies may diverge. 
PPO~\cite{schulman2017proximal} simplifies this further by simply bounding or ``clipping'' the ratio $m(\theta)$ of the ``old'' and current policies instead of calculating a formal divergence between the old and current policies:
\begin{equation}
J^{\text {clip}}(\theta)=\mathbb{E}\left[\min \left(m(\theta) \hat{A}_{\theta_{\text {old}}}(s, a), \operatorname{clip}(m(\theta), 1-\epsilon, 1+\epsilon) \hat{A}_{\theta_{\text {old }}}(s, a)\right)\right]
\end{equation}
PPO has the benefit of not requiring any complex gradient calculations and, as a result, is significantly simpler computationally while retaining most of the stability and performance benefits of methods like TRPO.
It is for all the reasons listed here that this dissertation makes regular use of PPO both as a baseline and a ``starting point'' to which our approaches are added.

\subsubsection{Model-based Reinforcement Learning}
\label{sec:background:rl:mfmb:mb}

In \textit{model-based} reinforcement learning (MBRL), the transition and reward functions $P$ and $R$ are modeled and used to plan agent action sequences to reach the goal and to develop a policy based on the effectiveness of those plans. 
So as to avoid confusion with the general term ``model,'' this document will refer to the approximation of the transition and reward functions together in model-based RL as a ``world model.''
Traditional model-based algorithms such as Dyna~\cite{sutton1991dyna} interact with the environment (often according to some fixed policy like random action) to gather data and use those data to learn the world model.
Then using this world model Dyna-like methods then execute a planning process to derive an optimal action, or---as in Dyna-Q~\cite{sutton1991dyna}---learn an optimal policy or Q-function with which the agent can solve the task. 
``Value-expansion'' methods such as AlphaZero~\cite{silver2017alphagozero} and MuZero~\cite{schrittwieser2020muzero} estimate the value of each state as in model-free approaches, and then using a learned or rule-based transition model simulate many outcomes from that state to estimate the value. 
Then after reaching a terminal state, the transition model-based value estimates are used to correct  policy and value models.  
Sampling-based methods like similar approaches in the optimal control literature, learn models with which they can continuously plan using a local optimization method like Model-Predictive Control to dictate the policy. Such methods include PETS~\cite{ecoffet2021first} and PlaNet~\cite{hafner2019planet}.

Model-building control systems~\cite{schmidhuber1991curious} are model-based techniques inspired by Dyna-style algorithms. 
They set themselves apart by learning a transition model through interactions with the environment, while simultaneously learning the policy~\cite{schmidhuber2015learning, ha2018recurrent, hafner2019dream}. 
This ``interleaved'' learning is distinct from learning in ``phases'' of first optimizing the world model, and then optimizing policy as in the original Dyna work. 
The Dreamer family of methods~\cite{hafner2019dream,hafner2020dv2,hafner2023dv3} is one such method; learning a world model based on the Recurrent State Space Machine (RSSM) architecture first proposed in PlaNet~\cite{hafner2019planet}, is one example of this approach.
As the work presented in this dissertation makes repeated use of Dreamer MBRL methods, we will go into greater detail on the specific Dreamer architecture and learning procedure. 

\subsubsection{Dreamer World Model}

The main purpose of the RSSM is to model the dynamics in an encoded latent space that has both stochastic and deterministic components.
The RSSM can be decomposed into three main components: (1) a deterministic \textit{recurrent trajectory model} $f_{\phi}$ that predicts a latent trajectory encoding $h_t$ given the prior action $a_{t-1}$, prior encoded stochastic state $z_{t-1}$, and prior deterministic trajectory encoding $h_{t-1}$, (2) an \textit{observation encoder model} $e_{\phi}$ that predicts a latent stochastic encoding $z_t$ of the current observation $x_t$ and the trajectory encoding $h_t$, and (3) a stochastic dynamics prediction model $g_{\phi}$ that predicts the stochastic state encoding $\hat{z}_{t}$ from solely the current deterministic trajectory encoding $h_t$. 

\begin{equation}
\begin{array}{c}
\text { RSSM }\left\{
\begin{array}{ll}
\text { (Deterministic) Recurrent trajectory model: } & h_t=f_\phi\left(h_{t-1}, z_{t-1}, a_{t-1}\right) \\
\text { (Stochastic) Observation encoder model: } & z_t \sim e_\phi\left(z_t \mid h_t, x_t\right) \\
\text { (Stochastic) Dynamics prediction model: } & \hat{z}_t \sim g_\phi\left(\hat{z}_t \mid h_t\right) \\
\end{array}\right . \\
\end{array}
\end{equation}

These RSSM components, combined with the observation prediction model, a reward prediction model, and a discount prediction model, form what Dreamer methods collectively refer to as the \textit{world model}.

\begin{equation}
\begin{array}{c}
\text { World Model }\left\{
\begin{array}{ll}
\text { RSSM: } & h_t, z_t=\mathcal{P}_\phi\left(h_{t-1}, z_{t-1}, a_{t-1}\right) \\
\text { Observation prediction model: } & \hat{x}_t \sim d_\phi\left(\hat{x}_t \mid s_t\right) \\
\text { Reward prediction model: } & \hat{r}_t \sim \mathcal{R}_\phi\left(\hat{r}_t \mid s_t\right) \\
\text { Discount prediction model: } & \hat{\gamma}_t \sim \Gamma_\phi\left(\hat{\gamma}_t \mid s_t\right) .
\end{array}\right . \\
\end{array}
\end{equation}
\label{eq:dreamer:wm}

\noindent Where the \textit{model state} $s_t = \{h_t, z_t\}$ is the concatenation of the deterministic and stochastic hidden states. All of these component models are implemented as neural networks and, as all of the world model's components are updated jointly, $\phi$ is used to describe their combined parameter. 
The trajectory recurrent model is implemented as an RNN such as a Gated Recurrent Unit (GRU)~\cite{cho2014properties}. 
For image inputs the observation encoder and prediction models are implemented as a Convolutional Neural Networks
(CNN)~\cite{lecun1995convolutional}, and a Multi-Layer Perceptrons (MLP) for non-images. 
Finally, all of the prediction models are implemented as MLPs.

While the trajectory recurrent model is deterministic, the other models sample their outputs by considering the outputs of their networks to parameterize multivariate distributions. 
The observation encoder and dynamics prediction model both parameterize a categorical distribution~\cite{hafner2020dv2}, the discount prediction model parameterizes a Bernoulli distribution, and the reward and observation prediction models parameterize a Gaussian distribution (with unit and parameterized variances, respectively).
Further implementation details for different efforts can be found in the Appendix.

For the world model learning process, image and reward prediction are supervised by ground truth data from the environments, and discount prediction is supervised by a fixed hyper parameter of 0 on terminal steps and 0.999 for non-terminal steps within an episode.
All components of the world model are optimized jointly using a weighted sum of the negative log-likelihood losses for image prediction, reward prediction, and discount prediction, and the Kullback–Leibler (KL) divergence between the dynamics prediction $g_{\phi}$ and observation encoder $e_{\phi}$ samples. 
\begin{equation}
\begin{aligned}
\mathcal{L}(\phi) =& \mathbb{E}_{e_\phi\left(z_{1: T} \mid a_{1: T}, x_{1, T}\right)}\left[ \sum_{t=1}^T \mathcal{L}_{NLL}(\phi) + \beta_{WM} \underbrace{ \mathbb{D}_{\text{KL}}\left[e_\phi (z_t \mid h_t, x_t) \| g_\phi(z_t \mid h_t)\right]}_{ \text {KL loss }}\right] \\
\mathcal{L}_{NLL}(\phi) =& - \left( \underbrace{\ln p_\phi\left(x_t \mid h_t, z_t\right)}_{\text { observation prediction }} + \underbrace{\ln p_\phi\left(r_t \mid h_t, z_t\right)}_{\text { reward prediction }} + \underbrace{\ln p_\phi\left(\gamma_t \mid h_t, z_t\right)}_{\text { discount prediction }} \right)
\end{aligned}
\end{equation}
\label{eq:dreamer:wm_loss}

In the variants of Dreamer including and following DreamerV2~\cite{hafner2020dv2}, the Gaussian stochastic latent is replaced with a categorical latent and learned using approximate ``straight-through gradients''~\cite{bengio2013estimating}. 


\noindent In practice, the expectation is approximated as an average over a batch of samples drawn from a replay buffer. 
Further implementation details for different efforts can be found in the Appendix.

\subsubsection{Latent Space Actor Critic}

For learning a behavior model, Dreamer uses a latent space actor critic policy gradient method trained entirely in the world model's ``imagination,'' only using rollouts predicted by the world model. 
This allows it to be trained completely in parallel with the world model, reduce interactions with the environment, and operate in a space that is a strictly Markovian representation. 
As in a traditional actor critic method, the actor aims to maximize the expected return $G_t = \sum_{\tau \geq t} \hat{\gamma}^{\tau-t} \hat{r}_\tau$.
The actor and critic are implemented as MLP neural networks, where the actor's action output parameterizes a categorical distribution, and the critic's value output is deterministic. 

\begin{equation}
\begin{aligned}
a_t \sim& \pi_{\theta}\left(a_t \mid s_t\right) \\
v_{\psi}\left(s_t\right) \approx& \mathbb{E}_{\phi,\theta} \left[G_t\right] .
\end{aligned}
\end{equation}
\label{eq:dreamer:actorcritic}

The actor uses model states, $s_i = \{h_i, z_i\}$, as input, and is used on-policy to generate rollouts for learning. 
However, differently from world model learning the stochastic state here comes from the dynamics prediction model. 
For each initial latent state in a set $[\boldsymbol{s_0}]$ (often the same as the batch drawn from the replay buffer for world model learning), learning rollouts of horizon $H$ are created iteratively from the actor and world model: 

\begin{enumerate}
\setlength{\itemsep}{0pt}
\setlength{\parskip}{0pt}
\setlength{\parsep}{0pt}
    \item sample reward and discount values from their respective prediction models,
    \item sample and action from the actor (on-policy),
    \item calculate the next model state using the trajectory recurrent model and dynamics prediction model.
\end{enumerate}

For critic learning, the target return is approximated using the TD($\lambda$) method to help balance bias and variance. 
An extension of the ``1-step'' TD-learning method described in Section~\ref{sec:background:rl:mfmb:mf}, the target in the TD($\lambda$) method, called a $\lambda$-return $V_t^\lambda$, is defined recursively as the sum of the reward and the weighted average of future returns. The critic TD($\lambda$)-error and $\epsilon^\lambda$ are therefore:
\begin{align}
    V_t^\lambda =& r_t+\hat{\gamma}_t \left((1-\lambda) v_\psi\left(s_{t+1}\right)+\lambda V_{t+1}^\lambda\right) 
    \quad {\text{with}} \quad 
    V_H^\lambda = v_\psi\left(s_H\right) \\ 
    \epsilon^\lambda =& v_{\psi}(s_t) - \operatorname{sg}\left(V_t^\lambda\right)
\end{align}
\label{eq:dreamer:tdlambda}

The critic loss is the expectation of the mean squared error (MSE) loss of the $\lambda$-error:

\begin{equation}
\mathcal{L}(\psi) = \mathbb{E}_{\phi, \theta}\left[\sum_{t=1}^{H-1} \frac{1}{2} \left(\epsilon^\lambda\right)^2 \right]
\end{equation}
\label{eq:dreamer:critic_loss}

For actor learning, the actor approximates the expected return using a weighted mixture of stochastic backprop through the TD($\lambda$) expected return---better continuous environments---and REINFORCE---better for discrete environments. 
Additionally, the actor loss is regularized by the policy entropy $\mathrm{H}$, which encourages exploration. 
This gives the loss functions for actor as: 
\begin{equation}
\mathcal{L}(\theta) = - \mathbb{E}_{\phi, \theta}\left[\sum_{t=1}^{H-1}\beta_{\text {ret}} \underbrace{ \left( \ln \pi_{\theta}\left(a_t \mid s_t\right) \epsilon^\lambda \right)}_{\text {REINFORCE}} 
+ (1-\beta_{\text {ret}})\underbrace{ V_t^\lambda}_{\text {stochastic}} 
+ \beta_{\text {ent}} \underbrace{ \boldsymbol{\mathrm{H}}\left[a_t \mid \hat{z}_t\right]}_{\text {entropy}})\right] 
\end{equation}
\label{eq:dreamer:actor_loss}

\noindent As in the case of the world model configuration parameters, further implementation details for hyperparameters including can be found in the Appendix.

\subsection{Interaction as Sampling in RL}
\label{sec:background:rl:sample}

One of the attributes of RL that sets it apart from supervised and unsupervised learning is that it is often used interactively. 
While supervised learning are most commonly used ``offline'' with a static gathered dataset, RL is most commonly studied in the ``online'' setting, where the agent is learning from data it has recently gathered by interacting with the environment, and uses what it learns from that learning experience to inform future interaction for gathering more learning data.
However, not all RL algorithms use interaction data the same way; different algorithms make different fundamental assumptions on the way they gather data for learning, called \textit{exploration}, and how they sample data for updating the agent. 

\subsubsection{Exploration}
\label{sec:background:rl:sample:expl}

While reinforcement learning agents generally take actions that greedily maximize future reward, agents must also sometimes move without regard to future reward or \textit{explore}. 
By mixing the greedy selection of maximum-value actions and exploring the environment, one can be more confident that the policy is not missing the most optimal path to the goal. 
This is often referred to as the exploration-exploitation trade-off~\cite{sutton2018reinforcement}.
Exploration is fundamentally necessary to the convergence of reinforcement learning, as it serves to diversify the samples seen by the learning algorithm and facilitates the agent's exposure to many trajectories in search of the optimal path to the goal. 
However, while many means of sample diversification satisfies this fundamental need, exploration methods that simply add diversity are not often not sufficient for finding the optimal solution to harder exploration problems.  
In many problem domains, including DeepMind Control Suite~\cite{tunyasuvunakool2020dmcontrol} and most Atari video games~\cite{bellemare13arcade},  the reward signal is \textit{dense}, meaning that it the agent receives regular feedback based on if it is making progress toward its optimal goal, and when there is only a sparse reward many research efforts sidestep this challenge by designing a task-specific ``shaped reward''~\cite{ng1999policy}.
For problems with dense or
``shaped'' rewards (reward functions designed to provide a dense signal for an otherwise sparse reward~\cite{ng1999policy}), simple undirected exploration methods like $epsilon$-greedy and Boltzmann sampling are sufficient for finding optimal solutions.
Reward shaping, however, undermines the generalizability of solutions as they are task-specific.
Also, many real-world decision making applications do not have any feedback signal beyond task ``success'' or ``failure'' (a \textit{sparse} reward). 
Simple undirected methods cannot be relied on to find the optimal path in these sparse ``hard exploration problems.'' 
More sophisticated exploration methods must be considered, such as those based on philosophies such as coverage, information maximization, and environment modeling. 
Section~\ref{sec:transx:characterization} provides a deeper comparative analysis RL exploration methods in the context of OTTA.

\subsubsection{Sampling for Learning and Experience Replay}
\label{sec:background:rl:sample:er}

In traditional reinforcement learning algorithms such as Q-learning, the agent updates its value function or policy based on the most recent transitions. 
However, as previously discussed, this approach can lead to high variance in the updates and instability in the learning process, particularly when combining off-policy methods like Q-learning with function approximation techniques like deep neural networks~\cite{tsitsiklis1996analysis}. 
As RL algorithms often are learning while interacting, or ``online,'' it is not trivial to reduce variance by simply learning from more samples because the samples will correlated instead of identically and independently distributed (i.i.d). 
Experience replay~\cite{lin1992self} is a key technique in deep reinforcement learning that addresses the problem of correlations in the sequence of observations encountered during the agent's interactions with the environment. 
During the training process, the agent adds experience to the buffer, and then to update the currect policy, value function, or model, the algorithm samples transitions from the entire buffer instead of just recent experience. 
This approach has several benefits, including decorrelating learning samples and increasing the data efficiency by reusing past experiences. 

Despite the effectiveness and ubiquity of experience replay in off-policy methods~\cite{lillicrap2016ddpg,wang2016acer,hessel2018rainbow,haarnoja2018soft}, new sampling methods and theoretical grounding are still being investigated~\cite{zhang2017deeper,fedus2020revisiting}. 
There are several approaches to sampling data from the replay buffer, each with its own advantages and trade-offs. 
The simplest approach is uniform sampling, where transitions are sampled uniformly at random from the buffer. 
However, this approach can be inefficient, as some transitions may be more informative than others. 
Prioritized experience replay addresses this issue by prioritizing transitions based on the TD error, giving higher priority to transitions with larger TD errors~\cite{schaul2016per}. 
In practice, prioritized experience replay represents a family of replay methods which organize and sample from the replay buffer according to different prioritization functions~\cite{gao2021expreward,oh2022modelaugmented,li2024prioritized}. 

Experience replay is not only applied to model-free off-policy methods. 
As model updating in model-based RL is fundamentally off-policy most model-based deep RL methods also take advantage of an experience replay buffer~\cite{silver2017alphagozero,hafner2019planet,hafner2019dream,schrittwieser2020muzero,yarats2021image,hansen2022temporal}. 
However, most of these works sample data naively despite the fact that, in a model-based RL method that learns both a model and a policy, the data might improve the model might not improve the policy and vice versa. 
This dissertation explores this sampling disconnect and its connection to transfer learning in RL in more detail in Chapter~\ref{chapt:transx}. 

\section{Transfer Learning and Novelty}\label{sec:background:transfer}

Reusing the prior ``knowledge'' encoded in a trained model is a common desire in machine learning problems.
Models often need to be updated as distributions shift, environments change, or new data are collected. 
Training new models from scratch is inefficient, costly, time-consuming, and can yield a lower-quality model; what's more if the model is only given access to the set of new data, the prior knowledge encoded in a learned model may be the way to learn about the new data in the context of the old data. 

Problematically, however, when a parameterized model such as a neural network is trained on data of one distribution, attempting to train it on data from a new distribution could induce \textit{catastrophic inference}~\cite{mccloskey1989catastrophic}. 
Catastrophic inference, also called catastrophic forgetting, occurs when training a learned model on a novel distribution causes the model to shift its parameters into a space that poorly models both the prior and novel distributions. 
When this happens prior models can end up transferring little, if any, of its prior knowledge, and can cause the model to actually be less efficient in learning the new distribution. 
\textit{Transfer learning}~\cite{zhu2021transfer} is a broad field of techniques and research that seek to avoid catastrophic inference and maximizing the amount of benefit a learned model's prior knowledge can have on learning new tasks.

\subsection{Fundamentals of Transfer Learning}

Transfer learning is an incredibly broad field, and applies to a wide variety of important problems. 
In \textit{model transfer learning}, sometimes referred to as knowledge distillation~\cite{hinton2015distilling,gou2021knowledge} or teacher-student frameworks~\cite{torrey2013teaching,zhan2015online}, transfers the knowledge of a domain and task encoded in a ``teacher'' model into a ``student'' model with the same task and domain. 
In this work, we focus on transfer learning concerned with adapting a model trained to solve one learning problem to a different learning problem.

In transfer learning in general, learning problems are decomposed into \textit{domains} and \textit{tasks}.
A \textit{domain} is defined as~$\mathcal{D} = \{\mathcal{X}, P(X)\}$, where $P(X)$ is the marginal distribution over the set of all input data $X$ sampled from the input space $\mathcal{X}$.
A \textit{task} is defined as~$\mathcal{T} = \{\mathcal{Y}, P(Y \vert X)\}$, where $P(Y \vert X)$ is the conditional distribution over the set of all output data $Y$ from the output space $\mathcal{Y}$ given the input data $X$~\cite{pan2009survey}. 
In the simplest, two-task setting, tasks and domains are divided into a \textit{source} task $\mathcal{T}_{source}$ and domain $\mathcal{D}_{source}$ for which a model is originally optimized, and a \textit{target} task $\mathcal{T}_{target}$ and domain $\mathcal{D}_{target}$ on which the performance of that model will be measured. 

\subsection{Novelties and Online Test Time Adaptation}

There are many types of transfer learning that focus on adapting models trained on source problems to target problems such as pretraining, domain adaptation, sim-to-real transfer, and skill transfer~\cite{taylor2009transfer}. 
However, in most of these problems either the target task is known in advance, the model is given some ``fine-tuning'' period to adapt to the new distribution, or both. 
This dissertation focuses on \textit{online test time adaptation}~\cite{liang2023comprehensive} (OTTA) to novelty, also known as ``online task transfer''~\cite{zhan2015online} and ``novelty adaptation''~\cite{pimentel2014review}. 
Given a model converged on a source task in a source domain, OTTA methods aim to leverage the knowledge of the source task to maximize performance on and minimize the number of training steps or samples required to adapt to the target task~\cite{liang2023comprehensive}. 
OTTA also assumes that, while there can be overlap between the source and task distributions, the source model has no prior experience training on the target distribution. 

\textit{Novelties}~\cite{pimentel2014review,boult2021towards} are characterized as sudden, previously unseen changes that functionally \textit{transform} the source domain, task, or environment into the target domain, task, or environment. 
Novelties can be big or small, ranging from the physics and mechanics of the world, to object relationships, properties, and interactions, to simply the presence of a new, unknown object~\cite{langley2020open}. 
Novelty research is broken down into three challenges: novelty detection (recognizing a change in the data distribution), novelty characterization (defining the change in the distribution), and novelty adaptation. 
While novelty characterization and detection~\cite{pimentel2014review} are important areas of study, as neither explicit detection nor characterization are necessary for novelty adaptation, this work focuses solely on adaptation. 

The critical difference between novelty adaptation and OTTA in general is that novelty adaptation assumes the existence of a transformation linking the source and target that can be characterized and specified. 
Novelties, therefore, act as a guarantee that the source and target tasks and domains are more connected than two arbitrary learning problems.   
As such, online test time adaptation to novelty, which in this dissertation will be referred to simply as OTTA, entails adapting a source-trained to a target domain and task, given that a novelty relates the source and the target. 


With a few exceptions~\cite{10351689}, the majority of research in OTTA and closely related fields like active test time adaptation~\cite{gui2024active} and open set domain adaptation~\cite{Jahan_2024_CVPR} do not consider sequential decision making problem settings, interactive or otherwise. 
Efforts exist to formulate a unified theory of novelty~\cite{boult2021towards,langley2020open} that applies to both interactive and non-interactive problem settings.
However, these efforts have not been able to fully rectify the challenges unique to interactive and non-interactive problem settings.


\subsection{Online Test Time Adaptation in Sequential Decision Making}

While OTTA for interactive settings is a mostly unstudied area, OTTA concepts are compatible with interactive settings. 
OTTA for interactive settings, required for deep reinforcement learning, can be framed in the context of MDPs \textit{$MDP_{source}$} and \textit{$MDP_{target}$}. 
As the environment is a decision process sampled non-i.i.d through agent interaction instead of a pre-sampled set, we redefine $\mathcal{D}$ and $\mathcal{T}$ for this work according the formulation of MDPs. 
For any $MDP_i$, the domain therefore becomes:
\begin{equation}
    \mathcal{D}_i = \{\mathcal{S},\mathcal{A}, \mathcal{P} \}
\end{equation}
\noindent making it strictly a function exclusively of the MDP.
The task, on the other hand, becomes: 
\begin{equation}
    \mathcal{T}_i = \{\mathcal{R}, \boldsymbol{\Pi^*}(\mathcal{S,A})\}
\end{equation}
\noindent where $\boldsymbol{\Pi^*}$ is the space of optimal solutions to MDP. 
That makes the task a function of both the MDP and the approach of solving the MDP.

This definition is beneficial as it covers and extends the definition space of open-world novelties presented by Boult et. al.'s definition of novelty~\cite{boult2021towards}.
For example, ``nuisance novelties,'' defined as novelties that do not affect the optimal solution are equivalent to changes in the world that only affect the domain $\mathcal{D}_{source} \neq \mathcal{D}_{target}$---without affecting the task---$\mathcal{T}_{source}=\mathcal{T}_{target}$. 
Such novelties can disrupt adaptation methods that strictly looking for changes in the environment without considering the impact of the change. 
Additionally, this model of transfer extends beyond that of Boult et. al. in that it can model the transfer problems of changes to the reward or the solution approach, both of which would change the task---$\mathcal{T}_{source} \neq \mathcal{T}_{target}$---but not the domain $\mathcal{D}_{source} = \mathcal{D}_{target}$. 
The work in this dissertation focuses exclusively on novelties that change the task and the domain,  $\mathcal{D}_{source}\neq\mathcal{D}_{target}$ and $\mathcal{T}_{source}\neq\mathcal{T}_{target}$.

There exists a strong history of prior work similar to reinforcement learning solutions for OTTA in interactive environments. 
Fundamental questions about reinforcement learning in non-stationary environments have been examined by prior works throughout the history of RL~\cite{littman1991adaptation,choi2001hidden,Nareyek2004}.
However, even the contributions of recent work in deep reinforcement learning are strictly theoretical~\cite{pmlr-v134-wei21b,pmlr-v139-mao21b,pmlr-v202-feng23e} or demonstrate limited applicability~\cite{JMLR:v17:14-037,NEURIPS2019_859b00ae,padakandla2020reinforcement,steinparz2022reactive} because of a tendency to focus on finding solutions to non-stationary environments \textit{in general}.
The work presented in this dissertation distinguishes itself by instead constraining the types of non-stationary environments considered in an effort to develop practical, applicable RL methods and solutions.
Transfer learning in RL~\cite{taylor2009transfer,Lazaric2012,zhu2021transfer} is also an area with similarly motivated work, with special interest in the ``sim-to-real'' problem of transferring policies learned in simulation to real robots and devices~\cite{chen2021cross}. 
However, like with the non-interactive OTTA prior work, most of these prior works presume (1) knowledge about (or access to) the target domain in advance and (2) a ``fine-tuning'' period in which the policy can be adapted to the target task, or both~\cite{xie2020deep}. 
The prior work most similar in problem setting to this dissertation are the works novelty-aware sequential decision agents that are not strictly adapted with reinforcement learning.
This work included adaptive mixed continuous-discrete planning, and knowledge graphs used in combination with reinforcement learning techniques to improve both detection and adaptation~\cite{klenk2020model,peng2021detecting,sarathy2021spotter,loyall2021integrated}. 

\section{Other Similar Fields of Research}

Researchers are investigating similar approaches with the techniques proposed for solving in  \textit{lifelong learning}~\cite{silver2013lifelong} and \textit{online learning}~\cite{shalev2012online,hazan2016introduction} problems. 
In lifelong learning, a learner operates under the assumption that the world is too complex or unpredictable to learn offline and must instead be modeled continually, and in online learning the assumption is that induction---the ability to make accurate predictions about future events from past trends---is not possible because train and test data are not drawn from the same distribution.
Functionally, these problems assume that the world is novel at each interaction or task compared to the prior interaction or task, often with no guarantees whether events will or will not be correlated. 
This relaxed set of assumptions is a useful and rich area of study because it can be applied to any problem, however the assumptions are overly-conservative for many transfer applications and yield poor overall results compared to the often-sufficient i.i.d, offline learning alternatives~\cite{kemker2018measuring,hayes2019memory,smith2021memory}.

%% file: chapters/3_novelty_chapt.tex
\chapter{Defining and Evaluating Agent Response to Novelty}
\label{chapt:novgrid}

\begin{figure}
    \centering
    \includegraphics[width=0.9\linewidth]{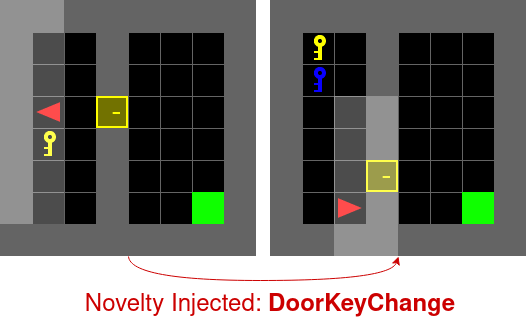}
    \includegraphics[width=0.9\linewidth]{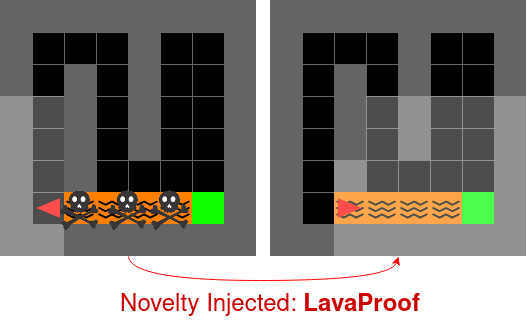}
    \caption{The NovGrid environments, where the agent (red triangle) must get to the goal (green box). The novelties are not directly observable; the agent must experience the novelty to be aware of it. 
    {\bf Top:} 
    pre-novelty only a yellow key opens a door; post-novelty only the blue key opens the door. 
    {\bf Bottom:}
    pre-novelty the lava gives a -1 reward and is a terminal state; post-novelty the lava is safe to walk on. 
    }
    \label{fig:novgrid:novgrid_splash}
\end{figure}

There exists a robust body of machine learning techniques---including but not limited to imitation learning and reinforcement learning (RL)---that can be used to learn models of agent behavior in complex sequential decision making environments. 
These techniques can be applied to find an optimal policy that solves nearly any problem that can be modeled as a Markov Decision Process (MDP), and the policies can be anything from simple look-up tables to Gaussian Processes and Deep Neural Networks~\cite{engel2005reinforcement,sutton2018reinforcement}. 
However, success in these learning methods shares the common assumption that the train-time MDP and the deployment or ``test time'' MDP are as similar as possible, if not the same.  
While this train-test similarity assumption holds in some settings, many real world applications of autonomous agents are associated with environments that cannot be guaranteed to function the same forever. 
As is the case with their human counterparts, learning agents in uncontrolled environments will encounter and need to adapt to unexpected changes \textit{as they experience them}, or ``online.'' 

Described in Section~\ref{sec:background:transfer}, adapting learned models online to unseen environment changes is called online test time adaptation (OTTA). 
OTTA is distinct from other areas of transfer learning in that it studies how a model trained on a \textit{source} task and domain can adapt to a new, unseen \textit{target} task and domain while experiencing it for the first time. 
The focus of this chapter is to develop a framework to study the properties of RL agent adaptation.
Specifically, we provide a definition of the OTTA problem setting for interactive sequential decision making, then introduce techniques and resources to study these specific OTTA challenges. 

First, the chapter presents an ontology of novelties in sequential decision making environments.
The novelty ontology distinguishes between (1) object novelties (new or changed properties of objects), (2) action novelties (changes in how the agent's actions work), and (3) whether the optimal solution for the  $MDP_{\mathrm{target}}$ is more, less, or similarly as complex to solve for a given agent as the $MDP{\mathrm{source}}$.
Second, this chapter describes the implementation of {\sc Novelty MiniGrid} (NovGrid), an extension of the MiniGrid environment~\cite{gym_minigrid} that changes the world properties and dynamics according to a generalized novelty generator based on the ontology. 
The MiniGrid environment is a grid world that facilitates reinforcement learning algorithm development with low environment integration overhead, which allows for rapid iteration and testing. 
NovGrid extends the MiniGrid environment by expanding the way the grid world and the agent interact to allow novelties to be injected into the environment. 
Specifically, this is done by 
creating an environment setup that, at a specified time unknown to the agent, ``injects'' a novelty that transforms the environment in the training process. 
NovGrid also provides a number of example novelties aligned with the dimensions of our novelty ontology and allow developers to create their own novelties. 
Third, this chapter details a set of metrics for measuring and evaluating the adaptability of agents. 


\section{Ontology of Novelties in Sequential Decision Making Problems}
\label{sec:novgrid:ontology}
In keeping with the standard formulation of RL as defined in Chapter~\ref{chapt:background}, let a stationary sequential decision making problem be modeled as a Markov Decision Process (MDP). 
For the problem setting that motivates this work, we consider non-stationary MDPs that can be approximated as two MDPs, $MDP_{\mathrm{source}}$ and $MDP_{\mathrm{target}}$, related by a transformation we call \textit{novelty}.
In considering novelty characteristics, we must consider two fundamentally different types of entities: agents and the environment.
Given this model of environments, we consider all aspects of the problem except a agent's decision-making model to be property of the environment. 
This includes agent morphology, sensors, and action preconditions and effects. 
As a result, the ontology we lay out here can be considered a specification of Boult et. al.~\cite{boult2021towards} world novelties in the context of sequential decision-making problems.

To clearly define the problem, we start with some simplifying assumptions. 
This work assumes that an agent's observation space and action space dimensionality remain consistent before and after novelty is injected. 
That is, the number of actions and the size and shape of the observations are consistent throughout each experiment. 
That said, the manifestation of these fixed sets may change; actions that initially have some specific effect or no effect pre-novelty can take on different effects post-novelty. 
Likewise, there may be observations and states that never occur pre-novelty that start to occur post-novelty. 
This is consistent with a robotics perspective on MDPs where actions and observations are governed by an underlying physics of the real world, even though most novelty experiments use grid worlds and games~\cite{kejriwal2021multi,gamagenovelty}. 
Additionally, this work assumes that the agent's mission $T$ is consistent before and after the novelty, meaning that there is no consideration of changes to the sparse extrinsic reward for reaching the goal. 

In this ontology, novelties are characterized along three dimensions. 
The first dimension is \textbf{object} vs \textbf{action} novelties. 
Objects are environment components, such as keys, doors, balls, etcetera, and object novelties involve the introduction, removal, or changes to the intrinsic properties (like mass and object-to-object interactions) of individual objects or classes of objects. 
Actions are the way agents affect the world through control. 
Action novelties involve changes in the dynamics of actions, such as the speed and force of agent motion or how an agent can interact with an object (like a key). 
This can be thought of as changes to action preconditions---the applicability criteria of actions---or action effects---the way in which the world is changed when an action is executed. 

Second, novelties can be expressed as changes to \textbf{unary} predicates or \textbf{non-unary} (or $n$-ary where $n>1$) relations. 
Unary object novelties can be thought of as added, removed, or changes to intrinsic properties of objects like mass, volume, or shape. 
Non-unary object novelties are changes in the relationship between objects, which is to say properties of objects that are necessarily defined in the context of other entities. 
Unary and non-unary action novelties involve 
(a)~the addition, removal, or change of properties of objects required for action applicability, or 
(b)~changes to the properties of objects or changes to the relationship between objects. 

Third, novelties are categorized according to how they change the distribution of solutions to a task: 
 \begin{itemize}
     \item {\bf Barrier novelty}---the optima in the solution distribution are longer after novelty than before novelty. 
     For example: pre-novelty the agent must acquire one key to pass through a door to achieve a goal, but post-novelty must acquire two keys. 
     \item {\bf Shortcut novelty}---the optima in the solution distribution are on average shorter after novelty than before novelty. 
     For example, a door that required a key pre-novelty then does not require any keys post-novelty. 
     \item {\bf Delta novelty}---the optima in the solution distribution are the same before and after novelty injection. 
     For example, a door that required one key pre-novelty requires a different key post-novelty. 
 \end{itemize}


\begin{table*}
\footnotesize
\centering
\begin{tabular}{c|c|c|c|c|}
\multicolumn{2}{c|}{} & {\bf Barrier} & {\bf Delta} & {\bf Shortcut}\\
\hline
\multirow{6}{*}{\bf Objects} & \multirow{3}{*}{\bf Unary} & {DoorLockToggle} & \multirow{3}{*}{GoalLocationChange} & DoorLockToggle\\
& & unlocked & & locked\\ 
& & $\rightarrow$locked  & & $\rightarrow$unlocked\\
\cline{2-5}
& \multirow{3}{*}{\bf Non-Unary} & {DoorNumKeys} & \multirow{3}{*}{DoorKeyChange} & \multirow{3}{*}{{ImperviousToLava}}\\
& & NumKeys=1 & &\\
& & $\rightarrow$NumKeys=2 & &\\
\hline
\multirow{6}{*}{\bf Actions} & \multirow{3}{*}{\bf Unary} & ActionRepetition & {ColorRestriction} & ActionRadius\\
& & PickCommands=1 & YellowOnly & PickDistance=1\\ 
& & $\rightarrow$PickCommands=2 & $\rightarrow$BlueOnly & $\rightarrow$PickDistance=2\\
\cline{2-5}
& \multirow{3}{*}{\bf Non-Unary} & TransitionDeterminism & \multirow{3}{*}{Burdening} & ForwardMoveSpeed\\
& & Deterministic &  & ForwardStep=1\\
& & $\rightarrow$Stochastic & & $\rightarrow$ForwardStep=2\\
\hline
\end{tabular}
  \caption{Novelty Ontology Exemplars}
  \label{tab:novgrid:exemplars}

\end{table*}


\section{Novelty Minigrid}\label{sec:novgrid:novgrid}

Novelty MiniGrid (aka NovGrid) is a testing environment we created to implement OTTA problems in the context of the above ontology. 
This provides a standard set of environments with which researchers can evaluate OTTA solutions across the spectrum of novelty types. 
NovGrid is built around an OpenAI Gym Wrapper and designed to be compatible with all MiniGrid environments. 
This means that NovGrid additionally works as a platform to evaluate OTTA performance on any of the many 3rd-party environments based on MiniGrid. 
It has three fundamental components: a novelty injection mechanism in the core wrapper class, new and modified entities designed to work with the novelty ontology, and the novelty generator with sample novelties to exemplify our ontology. 

The core novelty injection system is designed to be applicable to as many MiniGrid environments as possible. The wrapper wraps the environment, and the only argument required is the environment. Users can optionally specify 
the episode in which novelty is injected. 
Given a model in train mode, MiniGrid resets its grid at the beginning of every episode. 

Our novelty injection wrapper monitors the training cycle, and when the \texttt{novelty\_injection\_episode} is reached the wrapper class switches to using alternatives for the \texttt{reset} and \texttt{\_gen\_grid} functions. 
Specifically, after the novelty injection episode, the system now uses \texttt{\_post\_novelty\_reset} and  \texttt{\_post\_novelty\_gen\_grid}. 
This allows the wrapper to quickly and easily load in and overwrite the old environment with the new one. 

To exemplify the novelty ontology described in Section~\ref{sec:novgrid:ontology} and to provide example implementations of an OTTA scenarios, 11 exemplar novelties are built into the library that together cover all of the different categories of our ontology. 
This way all researchers using NovGrid can test their agent's adaptation sensitivity to different parts of the novelty ontology. 
The novelties delivered with NovGrid and how the respective objects would usually work in MiniGrid are:

\begin{itemize}
    \item \textbf{GoalLocationChange}: This novelty changes the location of the goal object. In MiniGrid the Goal object is usually at fixed location. 
    \item \textbf{DoorLockToggle}: This novelty makes a door that is assumed to always be locked instead always unlocked and vice versa. In MiniGrid this is usually a static property. 
    If a door that was unlocked before novelty injection is locked and requires a certain key after novelty injection, the policy learned before novelty injection will likely to fail. On the other hand, if novelty injection makes a previously locked door unlocked, an agent that does not explore after novelty injection may always still seek out a key for a door that does not need it. 
    \item \textbf{DoorKeyChange}: This novelty changes which key that opens a locked door. In MiniGrid doors are always unlocked by keys of the same color as the door. This means that if key and door colors do not match after novelty, agents will have to find another key to open the door. This may cause a previously learned policy to fail until the agent learns to start using the other key. 
    This novelty is illustrated in Figure~\ref{fig:novgrid:novgrid_splash}. 
    \item \textbf{DoorNumKeys}: This novelty changes the number of keys needed to unlock a door. The default number of keys is one; this novelty tends to make policies fail because of the extra step of getting a second key. 
    \item \textbf{ImperviousToLava}: Lava becomes non-harmful, whereas in Minigrid lava always immediately ends the episode with no reward. This may result in new routes to the goal that potentially bypass doors. 
    \item \textbf{ActionRepetition}: This novelty changes the number of sequential timesteps an action will have to be repeated for it to occur. In MiniGrid it is usually assumed that for an action to occur it only needs to be issued once. So if an agent needed to command the pick-up action twice before novelty but only once afterwards, to reach its most efficient policy it would need to learn to not command pickup twice. 
    \item \textbf{ForwardMovementSpeed} This novelty modifies the number of steps an agent takes each time the forward command is issued. In MiniGrid agents only move one gridsquare per time step. As a result, if the agent gets faster after novelty, the original policy may have a harder time controlling the agent, and will need to learn how to embrace this change that could make it reach the goal in fewer steps. 
    \item \textbf{ActionRadius}: This novelty is an example of a change to the relational preconditions of an action by changing the radius around the agent where an action works. In MiniGrid this is usually assumed to be only a distance of one or zero, depending on the object. If an agent can pick up objects after novelty without being right next to them, it will have to realize this if it is to reach the optimum solution. 
    \item \textbf{ColorRestriction}: This novelty restricts the objects one can interact with by color. In MiniGrid it is usually assumed that all objects can be interacted with. If an agent is trained with no blue interactions before novelty and then isn't allowed to interact with yellow objects after novelty, the agent will have to learn to pay attention to the color of objects. 
    \item \textbf{Burdening}: This novelty changes the effect of actions based on whether the agent has any items in the inventory. In MiniGrid it is usually assumed that the inventory has no effect on actions. An agent experiencing this novelty, for example, might move twice as fast as usual when their inventory is empty, but half as fast as usual when in possession of the item, which it will have to compensate for strategically. 
    \item \textbf{TransitionDeterminism}: This novelty changes the likelihood with which that actions selected by the agent occur. In MiniGrid it is usually assumed that all actions are deterministic. If an agent is trained with deterministic transitions before novelty and then experiences stochastic transitions after novelty, it will need to learn to take safe routes to the goal or its policy will fail more often
\end{itemize}

\noindent
To implement these novelties custom versions of different standard MiniGrid objects were designed, and these custom objects are also included with NovGrid. 

Table~\ref{tab:novgrid:exemplars} shows a mapping of the exemplar novelties built into NovGrid to dimensions of the novelty ontology.

\section{Metrics for Transfer Adaptation}\label{sec:novgrid:metrics}

Adaptability broadly refers to the ease with which a model trained for one task can be retrained for another task. 
Adaptability is measured on two major axes: efficiency and efficacy. 
Both sample efficiency (i.e., the number of interactions with the task required for convergence) and computational efficiency (i.e. the number of iterations required for convergence) are used to measure adaptation efficiency. 
The efficacy of agent adaptation is measured on performance on the task, the way the agent reacts to the novelty, and the speed with which it recovers. 
To that end, the following metrics are built into NovGrid:

\begin{itemize}
    \item {\em Resilience}: the difference between agent maximum performance pre-novelty and agent performance post-novelty {\em without} adaptation. 
    \item {\em Asymptotic adaptive performance}: converged performance post-novelty. 
    \item {\em Adaptive efficiency}: the number of environment interactions to converge post-novelty. 
    \item {\em One-shot adaptive performance}: the performance of the agent post-novelty after only one episode of interaction with the environment. 
\end{itemize}

\begin{figure}[t]
    \centering
    \includegraphics[width=\linewidth]{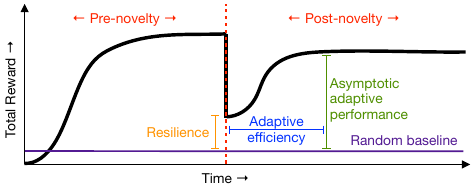}
    \caption{Evaluation metrics illustrated against a notional performance curve for an agent. }
    \label{fig:novgrid:metrics}
\end{figure}

\section{Key Takeaways}\label{sec:novgrid:conclusions}

Novelty in sequential decision making is a rich and under-investigated research area, and research into novelty adaptation of agents will enable autonomous agents to solve more complex, real world problems. 
With the work presented in this Chapter, we address gap in definitions and means of evaluating agent performance of online test time adaptation to novelty in sequential decision making. 
Our definitions and novelty ontology provides a language to discuss the effects of different novelties on agent adaptability.
The NovGrid library, exemplar scenarios that map to this ontology, and metrics proposed for adaptation evaluation provide future researchers the means to repeatably analyze and compare how different policies and strategies will adapt to different types of changes.

The presentation of the ontology and the NovGrid evaluation environment provides a starting point with which researchers can develop their own OTTA solutions. 
Future researchers investigating sequential decision making agents must consider the ways in which, if deployed, the agents' environments may change.
NovGrid, the ontology of novelties, and the proposed metrics for measuring OTTA performance provide a starting point for researchers to test the OTTA performance of existing methods and develop novel solutions.
As non-stationarity is an undeniable reality all real world agents will face, this work provides a means of characterizing adaptive response, but also a template for how online test time adaptation can be measured and investigated in other domains.

Still, this work is limited to the analysis of discrete novelties of known difficulty affecting individual agents. 
Further investigation is needed from future researchers to extend the definitions and evaluation criteria presented in this chapter to continuous change in an environment, quantifying the similarity between two MDPs, and multiple agents.
While the work proposed here provides a strong foundation of simplifying non-stationary problems definitions, many real world problems experience gradual change that builds up without adaptation.
Whether user preferences in recommender system models or sensor drift on a robot, while these phenomena can be modeled as a sequence of discrete changes, model such novelties as continuous changes will lead to solutions that consider the causal factors and progression of this change. 
Relatedly, while the ontology's characterizations of novelty based on notions such as \textit{solution complexity} are helpful and intuitive, the ontology can be strengthened by a quantitative characterization of the difference between source and target MDPs.
Without an underlying quantitative measure of change, there is no way for an agent to, for example, improve its adaptation by knowing or predicting ontological characteristics of the novelty.
Lastly, one of the most prevalent sources of novelty in autonomous agent scenarios insufficiently modeled by this work is external agent behavior change, either as the environment change or due to environment change.
Looking again to recommender systems, user preferences may change as a group or individually. 
Often times behavior and preference change is the result of changing environment factors, but users also may simply start preferring different content independent of an environment. 
The ontology and NovGrid provide a foundation that future researchers can extend to overcome these limitations and measure novelty adaptation more accurately for more scenarios.

This chapter adds to the broader thesis by providing a foundation from which we investigate the ways we might improve how RL agents adapt. 
In line with the broader thesis of this dissertation, in the following Chapters we use these novelty definitions, especially the \textit{shortcut}, \textit{delta}, and \textit{barrier}, and the novelties implemented in NovGrid to develop and assess the quality of adaptation solutions for RL agents that explore novel phenomena and reuse source domain knowledge appropriately. 

%% file: chapters/4_transx_chapt.tex
\chapter{Characteristics of Effective Exploration for Adaptation in Reinforcement Learning}
\label{chapt:transx}

As described in Chapter~\ref{chapt:background}, reinforcement learning algorithms trade off exploration and exploitation.
When there is a novel change in the environment, the adaptation efficiency of the RL agent depends on the data collected by interacting with the novel environment and how those data are used. 
In this way, exploration strategies in RL serve a dual purpose in OTTA: they are fundamentally designed to sample the state-action space for more efficient learning in stationary environments, and they also have the potential to facilitate adaptation to environmental changes. 
Similarly, the process of selecting which samples for learning can impact an agent's ability to learn and adapt to novel situations. 
This chapter and Chapter~\ref{chapt:dops} delve into the critical role of exploration and sample selection in enabling efficient online test time adaptation to novelty in reinforcement learning (RL).

In theory, exploration designed for stationary RL can enable agents to adapt to environment novelties with no fundamental changes~\cite{schmidhuber1991curious,Schmidhuber91apossibility,chentanez2004intrinsically}.
In spite of this, exploration algorithms designed to improve the exploration-exploitation trade-off of solving single, stationary MDPs have not been comprehensively analyzed for their impact on efficient online test-time adaptation. 

In this chapter, we answer the question: 
\textbf{which characteristics of traditional exploration algorithms are important for efficient transfer in RL?}
To reach our answer, we conducted experiments with eleven popular RL exploration algorithms on five novelties in discrete and continuous domains.
The algorithms were selected to represent a diverse space of exploration characteristics.
We systematically examine the within- and between-class relationships of the algorithms across all characteristics. 

Our results indicate, foremost, that exploration methods that explicitly emphasize {\em diverse training experiences} and use {\em stochasticity} to avoid overfitting benefit policy transfer the most.
This is true across all types of novelties and for both discrete and continuous domains.
When novelty makes a task easier---called a {\em shortcut novelty}---, exploration methods that rely heavily on stochasticity lose some effectiveness, but the benefits of diversity are more pronounced.
When novelty makes the task harder---called a {\em barrier novelty}---we find that the difference in performance between all exploration methods was severely diminished.
Finally, our continuous control experiments showed even more pronounced benefit of stochasticity and that exploration methods that are time independent or explore based on the entire training process---i.e., {\em temporally global} methods---outperformed methods that explore based on short-term change. 

In this chapter, 
Section~\ref{sec:transx:characterization} defines the exploration characteristics that have observable effects on transfer and maps eleven RL algorithms chosen for the experiment to these characteristics.
Section~\ref{sec:transx:exp} then details our experimental methodology. 
Section~\ref{sec:transx:disc} details our results and discusses implications. 
Finally, Section~\ref{sec:transx:conclusion} revisits the key takeaways from this chapter, explains how this chapter supports the thesis of this dissertation, and the implications of future work.

\section{Related Work}\label{sec:transx:relatedwork}
There is a large body of work characterizing and surveying the impact of exploration on transfer in RL. 
These works consider transfer in RL where exploration is a single variable~\cite{taylor2009transfer,Lazaric2012,da2019survey,zhao2020sim,zhu2023transfer} and exploration as one of several use cases~\cite{ladosz2022exploration,yang2021exploration}.
There is also a body of work that examines the relationship between active learning and adaptation to novelty and open-worlds~\cite{langley2020open,boult2021towards}.
Our work contributes by characterizing exploration methods across multiple dimensions and analyze their transfer performance specifically for RL and sequential decision-making.

Of the techniques investigating exploration methods for transfer in RL, they are tailored to a specific algorithm~\cite{zhan2015online,barreto2017successor}, do not translate to deep RL~\cite{konidaris2012transfer}, or do not compare themselves to stationary MDP exploration methods. 
Our work contributes by providing new analytical frameworks for further developing exploration methods depending on the transfer problem.
Most similar to our work is \cite{burda2018large}, which empirically investigates the implications of different exploration algorithms that share a curiosity objective as their exploration principle. 
Our work distinguishes itself by including a broader group of exploration principles than just intrinsic reward and does so for the purposes of online test time adaptation in RL instead of the typical single-task formulation.

\section{Characterizing Exploration Methods}
\label{sec:transx:characterization}

\begin{figure}[t]
    \begin{center}
        \includegraphics[width=\textwidth]{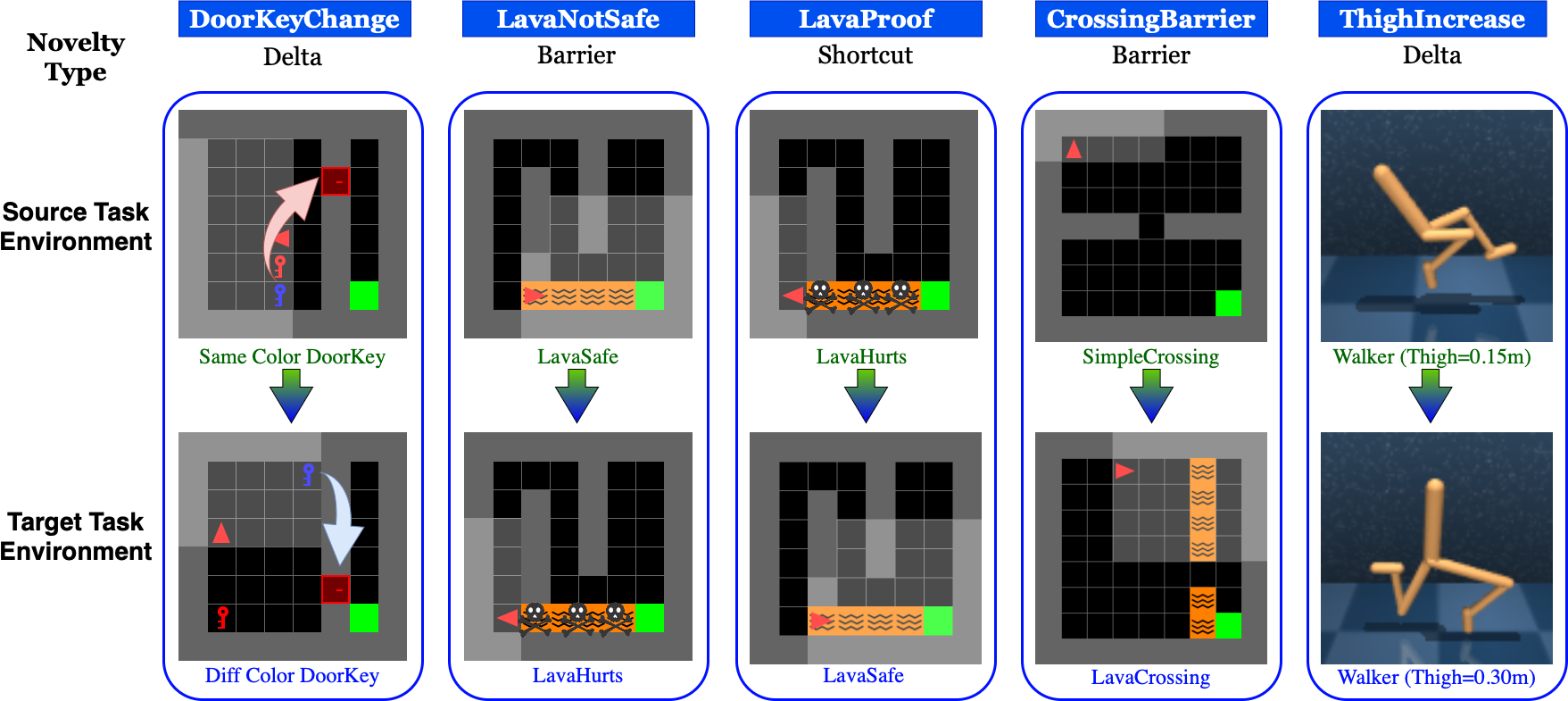}
    \end{center}
    \caption{Environments and novelties used to evaluate the exploration algorithms and their characteristics, including discrete and continuous control environments.} 
    \label{fig:transx:envs}
\end{figure}

There are many ways one might categorize exploration methods.
From the perspective of OTTA, we divide exploration methods into two high-level categories: \textit{exploration principle} and \textit{temporal locality}, which both have subcategories.
These are consistent with the existing taxonomy of \cite{ladosz2022exploration}.


%
\textbf{Exploration principle}
    characterizes an agent's behavior beyond greedy maximization of reward. 
We identified three subcategories of exploration principles. 
(1) Adding \textit{stochasticity} into the learning process.
There are many ways to use stochasticity in exploration, whether by injecting random noise into the input or an intermediate weight layer, using a stochastic task policy, or simply selecting random actions. 
(2) \textit{Explicit diversity} over the different random variables in the process.
Explicit diversity methods encourage models to experience all parts of the domain and task equally, ensuring that a greedy process does not lead the agent into stale transitions.
(3) Having a \textit{separate objective} in addition to greedy pursuit of reward. 
Methods with a separate objective complement the flaws of greedy reward maximization with a non-greedy goal, alternating or combining the objectives. 

\textbf{Temporal locality}
characterizes an exploration algorithm's relationship to time. 
Most exploration algorithms are designed to adapt to the needs of an agent at different points in the learning process.
We identified three temporal locality subcategories. 
(1) Algorithms with short-term or temporally \textit{local} characteristics. These methods implement adaptive behavior as a function of how agent and environment properties evolve time step to time step or episode to episode. 
(2) Algorithms with long-term temporally \textit{global} characteristics. These methods 
influence exploration based on trends in agent and environment properties recorded or aggregated across the entire learning problem or by comparing these global properties with the current agent, environment, or learning state.
(3) \textit{Time-independent} exploration methods. 
Similar to characterizations~\cite{sutton2018reinforcement} of exploration methods as ``directed'' or ``undirected,'' time-independent methods counteract greedy behavior by altering the learning process as a whole or within the agent architecture itself. 
Time-independent methods are critical to evaluation of exploration in transfer applications because online test time adaptation induces a temporal shift, both globally and locally. 

We summarize the exploration principle and temporal locality categories, along with exemplar algorithms, in Table~\ref{tab:transx:characteristics} and Appendix~\ref{app:transx:chars}.

\renewcommand{\arraystretch}{1.2}
\begin{table}[tb!]
\centering
\footnotesize
\begin{tabular}{p{0.15\textwidth}|p{0.2\textwidth}|p{0.45\textwidth}}
\textbf{Categories} & \textbf{Characteristics} & \textbf{Example Algorithms} \\
\hline\hline
\multirow{3}{2cm}{\bf Exploration Principle} & Stochasticity & NoisyNets, DIAYN\\
\cline{2-3}
 & Explicit Diversity & RND, REVD, RISE, RE3, RIDE,
NGU, DIAYN\\
\cline{2-3}
 & Separate Objective & RND, RIDE, ICM, NGU, GIRL\\
\hline\hline
\multirow{3}{2cm}{\bf Temporal Locality} & Global & RND, ICM, RE3, NGU, GIRL\\
\cline{2-3}
 & Local & EVD, RIS, RIDE, NGU\\
\cline{2-3}
 & Time Independent & NoisyNets, DIAYN\\
\hline
\end{tabular}
\caption{This table lays out our decomposition of exploration algorithms into 
two
major categories---exploration principle
and temporal locality---with three core characteristics in each.
The algorithms listed here are evaluated as described in Section~\ref{sec:transx:exp.algos}.
Algorithms are described in detail in the Appendix.}
\label{tab:transx:characteristics}
\end{table}

\section{Experiments}\label{sec:transx:exp}

We selected 11 reinforcement learning algorithms based on their exemplary usage of stochasticity, explicit diversity, separate exploration objectives, and orientation to global or local temporal locality.
We trained and tested each algorithm in discrete and continuous domains and in the presence of shortcut, delta, or barrier novelties.

\subsection{Exploration Algorithms}
\label{sec:transx:exp.algos}

\begin{figure}[t]
    \begin{center}
        \includegraphics[width=0.75\textwidth]{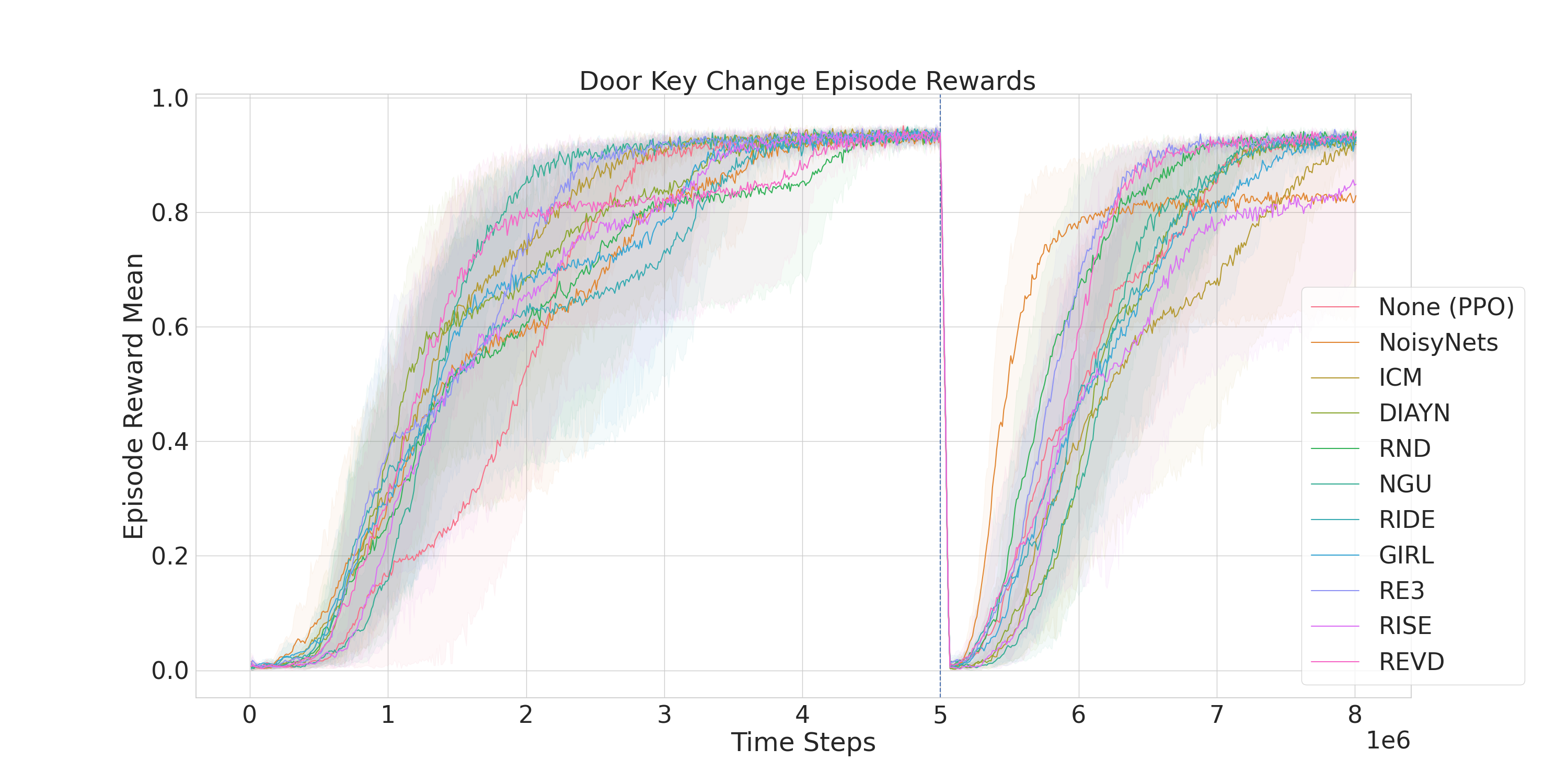}
    \end{center}
    \caption{Full learning and adaptation process of eleven RL exploration algorithms on the \texttt{DoorKeyChange} novelty problem from NovGrid~\cite{balloch2022novgrid}. 
    The agents first learn a task assuming a stationary MDP. 
    The rate of learning at this stage is {\em convergence efficiency}.
    At time step 5,000,000 novelty is injected into the environment, transferring from $MDP_\mathrm{source}$ to $MDP_\mathrm{target}$,  often causing a performance drop-off. 
    The algorithms then recover their performance as they learn the new world transition dynamics. 
    The rate of learning at this stage is {\em adaptive efficiency}.
    The maximum episode reward is the {\em final adaptive performance}, which may not always be as high as pre-novelty performance.
    } 
    \label{fig:transx:full_group_doorkey}
\end{figure}

For our assessment, we focus on model-free, on-policy deep policy gradient methods that apply to a variety of reinforcement learning tasks. 
Specifically, we use proximal policy optimization (PPO)~\cite{schulman2017proximal}, a high-performing actor-critic policy gradient method, as the algorithmic backbone of all the exploration methods we test. 
On-policy actor-critic methods such as PPO are more versatile than off-policy methods, which only apply to a subset of RL problem formulations. 
For example, methods like Deep Q-Networks~\cite{mnih2015human} only apply to problems with discrete action spaces and methods like Soft-Actor Critic~\cite{haarnoja2018soft} and Deep Deterministic Policy Gradients~\cite{lillicrap2019continuous} only work in continuous control environments. 
Additionally, off-policy methods are very sensitive to the management of an experience replay buffer for successful learning~\cite{mnih2015human}, which becomes significantly more complex when adapting online
because hyperparameters such as how often the experience replay buffer should be reset become potential confounding variables.
In an effort to control as many independent variables as possible and focus our investigation on exploration, we only consider the PPO algorithm for this initial investigation.

We select 11 popular exploration algorithms that represent a broad sampling of exploration principle and temporal locality categories, while being compatible with PPO and our environments.
Those algorithms are Random Network Distillation (RND)~\cite{burda2018exploration}, Intrinsic Curiosity Module (ICM)~\cite{pathak2017icm}, Never Give Up (NGU)~\cite{Badia2020Never}, Rewarding Impact-Driven Exploration (RIDE)~\cite{raileanu2019ride}, Renyi State Entropy Maximization (RISE)~\cite{yuan2022renyi}, Rewarding Episodic Visitation Discrepancy (REVD)~\cite{yuan2022rewarding}, enerative Intrinsic Reward Learning (GIRL)~\cite{yu20girl}, Parameter Space Noise for Exploration (NoisyNets)~\cite{plappert2018parameter}, and ``online'' Diversity Is All You Need (DIAYN)~\cite{eysenbach2018diversity}. 

Table~\ref{tab:transx:characteristics} shows how the algorithms relate to exploration characteristics; descriptions of the algorithms can be found in Appendix~\ref{app:transx:algos}.
Our implementation of these algorithms is based on the Stable-Baselines3~\cite{raffin2021stable} and RLeXplore libraries,\footnote{https://github.com/RLE-Foundation/RLeXplore} 
which we modify and expand for the purposes of our investigation.

\subsection{Learning Environments and Transfer Tasks}\label{sec:transx:exp.envs}

To experiment with online transfer, agents are trained to convergence in one environment (the source task), and then a novelty is introduced to create the target task. 
The agent must recover its performance during online execution in the target environment.
We run our experiments with two transfer learning libraries, NovGrid~\cite{balloch2022novgrid} and Real World Reinforcement Learning suite~\cite{Dulac-Arnold2021rwrl}. 

NovGrid, as described in Chapter~\ref{chapt:novgrid}, is a specialization of the MiniGrid~\cite{MinigridMiniworld23} environment designed to promote experimentation in novelty adaptation in RL.
Specifically, NovGrid sets up learning scenarios then injects a novelty---changing the transition dynamics---at a time that is unknown to the agent.
We use three novelty environments within NovGrid---\texttt{DoorKey}, \texttt{LavaMaze}, and \texttt{CrossingBarrier} environment---which are used with the injection of specific novelties.
\texttt{DoorKeyChange} is a delta novelty in which a DoorKey environment is changed so that the key that opens the door is changed.
\texttt{LavaProof} is a shortcut novelty where the lava in LavaMaze is changed from being a zero-reward terminal state into a safe, passable, non-terminal state.
\texttt{LavaNotSafe} is a barier novelty that is functionally the reverse of \texttt{LavaProof}, changing the lava in \texttt{LavaMaze} from non-terminal into a terminal state. Lastly, in \texttt{CrossingBarrier}, the impassable but safe walls are exchanged for standard, terminal-state lava.  
%
We allowed the algorithms to run until the majority of runs on all algorithms converged 
before the novelty was injected. 
We tuned the hyperparameters of the algorithms on the novelty-free \texttt{DoorKey} environment for use with the NovGrid environments, maximizing convergence in the source environment so as to help ensure convergence on the source task. 
The details of the hyperparameter tuning is in Appendix~\ref{app:transx:impl:hyperparams}.  

The Real World Reinforcement Learning suite~\cite{Dulac-Arnold2021rwrl} provided a continuous control environment for evaluating adaptation performance. 
We tuned the hyperparameters of our algorithms on the Cartpole-Swingup environment by changing the pole length, which maintains the same approximate difficulty of the target task.
We evaluate OTTA performance on the more complex Walker2D environment with the ThighIncrease novelty, where the length of the thigh link is increased from 0.15 meters to 0.3 meters.
See Figure~\ref{fig:transx:envs} for illustrations of the environments and novelties.

\subsection{Measuring Online Test Time Adaptation Performance}\label{sec:transx:exp.metrics}

To assess the exploration methods, we measure learning efficiency and performance motivated by the desire to minimize the number of environment interactions required to learn good policies in the target task, as described in Chapter~\ref{chapt:novgrid}.
The primary metrics are: 


\textbf{Adaptive efficiency}: The number of environment steps necessary for the agent to reach 95\% of maximum performance on the target task.


\textbf{Transfer Area Under the Curve (Tr-AUC)}:
Inspired by the performance ratio of \cite{taylor2009transfer}, Tr-AUC is a novelty-agnostic measure of the overall transfer performance as a function of both the source and target task:
\begin{equation}
\text{Tr-AUC} = \frac{1}{2} \left(\max(r_{\text{S}})+\frac{1}{K} \sum_{i \in K} r_{i,\text{T}} \right)
\label{eq:trauc}
\end{equation}
%
where $\max(r_{\text{S}})$ refers to the final performance on the source task and the summation over $r_\text{T}$ gives accumulated adaptive performance until the final adaptive performance point on the target task.  
Tr-AUC balances efficient adaptation with prior task performance by penalizing methods that performed well on the target task due to underperforming on the source task or vice versa. 
For all metrics, we calculate the mean and standard deviation of a bootstrapped sampling of the runs of each method, and calculate the interquartile mean (IQM) and the bootstrapped 95\% confidence interval per~\cite{agarwal_deep_2021}. 

One of the key assumptions that we make in the motivation of this work is that in real-world online test time adaptation scenarios, the policy is assumed to have converged to maximize the performance on the source task before novelty is injected and the policy must be adapted to the target task. 
However, in practice one of the deficiencies of deep reinforcement learning is the highly stochastic nature of convergence, especially in sparse reward tasks like those of NovGrid. 
For an analysis best aligned with our motivations, we measure our results with respect to the full set of experiments that converged on the source task unless otherwise specified.

\section{Results and Discussion}\label{sec:transx:disc}

\begin{figure}
  \centering
  \begin{subfigure}[b]{0.47\textwidth}
    \includegraphics[width=\textwidth]{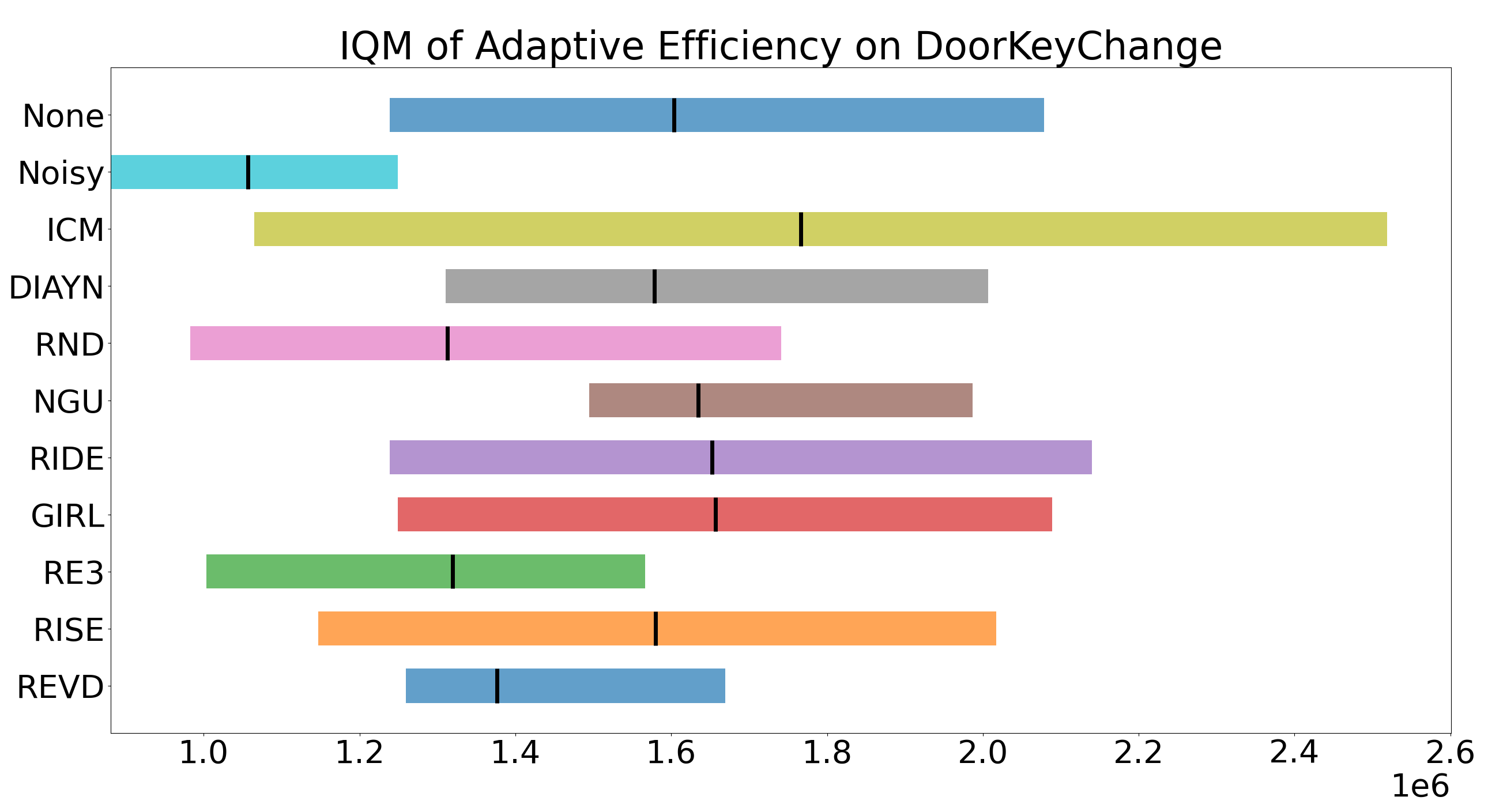}
  \end{subfigure}
  \begin{subfigure}[b]{0.45\textwidth}
    \includegraphics[width=\textwidth]{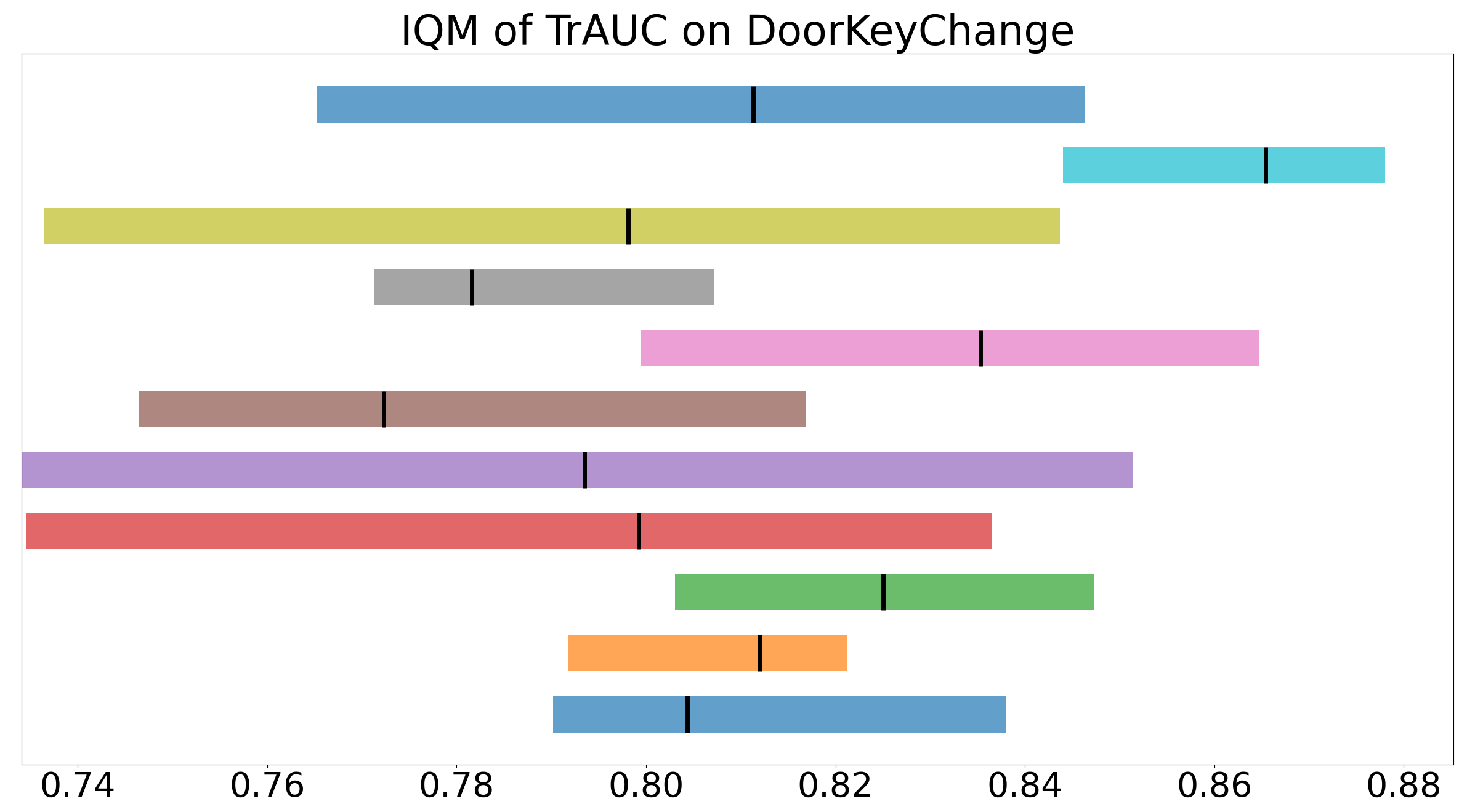}
  \end{subfigure}
  
  \caption{The Adaptive Efficiency and Tr-AUC inter-quartile mean plots for DoorKeyChange. These plots show NoisyNets performing well by both metrics. It should be noted that the Adaptive Efficiency graphs are only showing runs that converged on both tasks and the Tr-AUC graphs are filtering for runs that converged on the first task. 
  }
\end{figure}\label{fig:transx:iqm}

We compared the relationship between source task convergence efficiency with adaptive efficiency for different algorithms in our environments, and validated our analysis of these comparisons with results on the Tr-AUC metrics, exemplified in Figure~\ref{fig:transx:iqm}. 
A complete list of our results across all algorithms, metrics, environments, and novelties can be found in Appendix~\ref{app:transx:chars}.
We discuss and analyze our results in the context of the specific experimental research questions we laid out in Section~\ref{sec:transx:exp}. 

\paragraph{The exploration principle characteristics have a large impact on the effectiveness of online test time adaptation.}

Exploration methods with {\em stochasticity} and {\em explicit diversity} characteristics are slower to converge on the source task, but adapt most efficiently to the target task. 
Representing the exploration principles of explicit diversity and stochasticity, respectively, RE3 and NoisyNets are the two algorithms that consistently performed well.
While not as consistent in performance as RE3 and NoisyNets, other explicit diversity and stochasticity methods REVD, RND, and DIAYN also adapt efficiently in most tasks, as can be seen in Tables~\ref{tab:transx:adaptive_efficiency} and~\ref{tab:transx:trauc}.

Further reflecting the importance of exploration principle, ICM, NGU, and other \textit{separate objective} performed consistently below average on the tasks and novelties. 
This can be attributed to inductive bias caused by the task-dependence of separate objective exploration methods.
The ICM exploration method adds an inductive bias to the typical prediction error-based curiosity metric by focusing only on state change predictions \textit{that result from agent action}.
This is a productive approach in conventional single-task RL because it is robust to arbitrary changes in the environment, like the ``Noisy TV'' problem~\cite{burda2018exploration}. 
However, in online test time adaptation this would mean the exploration algorithm might avoid the novelty as it was not caused by agent action. 
NGU and several other separate objective algorithms use a similar action-focused inductive bias in their embedding spaces and, as a result, also see their performance suffer.

\begin{figure}[t]
    \centering
    \includegraphics[width=0.75\textwidth]{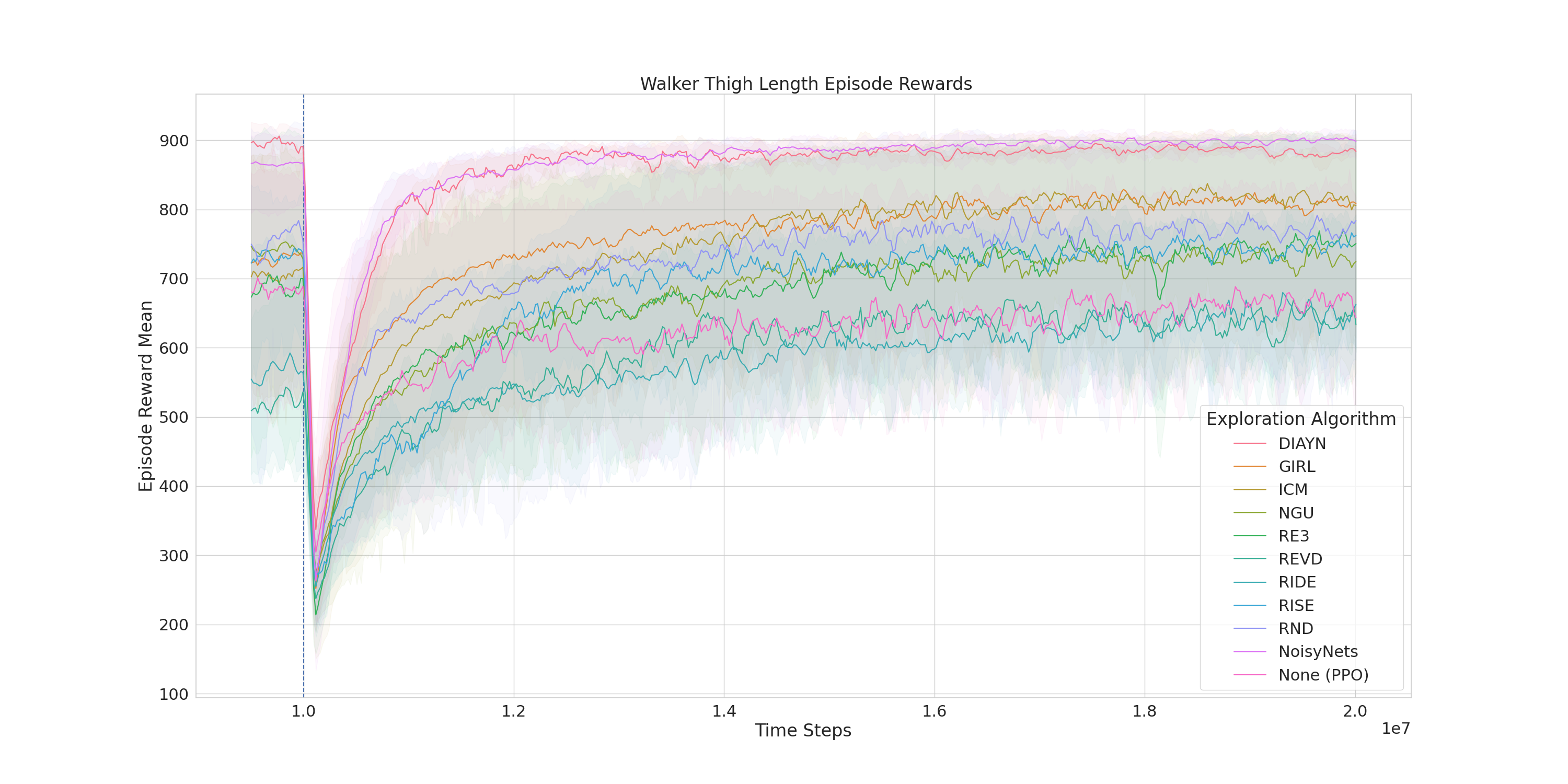}
    \caption{The reward plot from dm\_control Walker-Walk ThighIncrease delta novelty transfer task. 
        The vertical line at 1E7 steps indicates where novelty was injected. 
        The shaded areas represent the variance over all seeds.    
    NoisyNets and DIAYN are the highest performing and most efficient adapting methods. In contrast to the DoorKeyChange discrete delta novelty, there appears to be some correlation between performance before and after the novelty. The shaded areas represent the variance over all seeds.}
    \label{fig:transx:walker}
\end{figure}

\begin{figure}[t]
    \centering
        \includegraphics[width=0.75\textwidth]{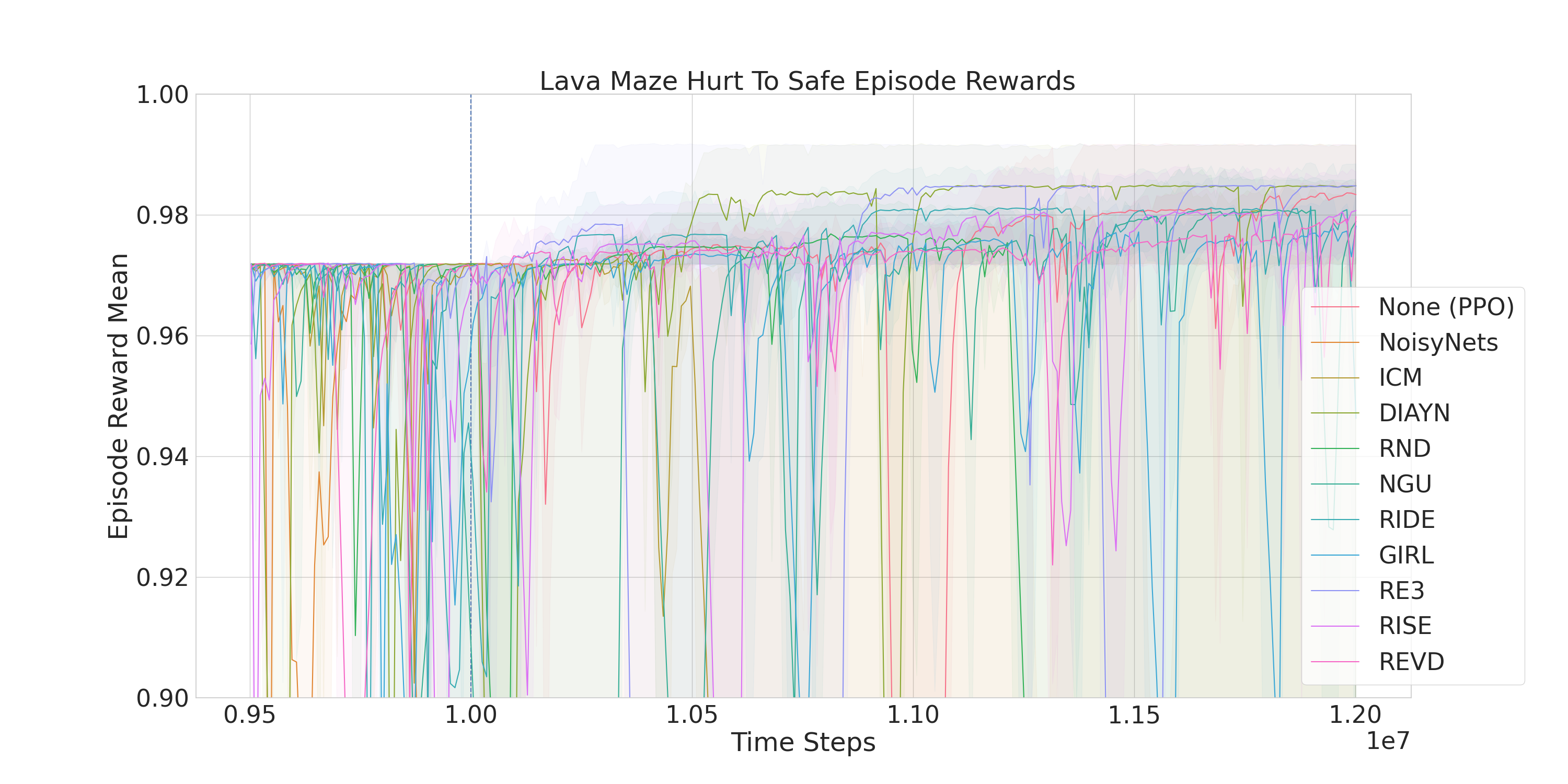}
        \caption{Results from the \texttt{LavaSafe} shortcut novelty. 
        The vertical line at 1E7 steps indicates where novelty was injected. 
        The shaded areas represent the variance over all seeds.
        Some of the exploration algorithms are able to find the shortcut, rising above the pre-novelty performance, while others never discover the shortcut. 
        }
    \label{fig:transx:shortcut}
\end{figure}

\paragraph{In continuous action environments, exploration methods with stochastic principles dominate, and the difference between of temporal locality characteristics is more important than in discrete action environments.}

Stochastic methods dominate in the continuous action domain. 
DIAYN and NoisyNets recover significantly faster than all other methods. 
Diversity in exploration for transfer is less important to efficiency than in discrete experiments, but still performs on par with the non-stochastic exploration methods.
This is most likely a result of high transferability of the continuous control ThighIncrease novelty compared to the discrete environments novelties; because of the nature of a continuous action space, noise in both the random conditioning space of DIAYN and the noisy weights of NoisyNets exposes those policies to ``nearby'' actions corresponding to the new optimal policy.
The explicit diversity principles underlying methods like RE3, REVD, and RND are more impactful in discrete action space environments as the optimal actions in $MDP_{target}$ are not similar to the optimal actions in $MDP_{source}$. 
That said, diversity methods are not significantly worse in adaptive performance than separate objective methods, thus remain useful. 

Temporal locality showed greater impact on performance in our continuous environment compared to our discrete environment. 
As shown in Figure~\ref{fig:transx:walker}, we find that the time-independent strategies---NoisyNets and DIAYN---dominate; the temporally global strategies such as RND, ICM, and NGU perform well; and the temporally local strategies struggle the most both pre- and post-novelty. 
We attribute this result to optimal continuous control policies often only needing small, smooth action differences in time to learn a policy, which favors exploration methods with global and time-independent temporal locality.

\begin{table}[t]
    \centering
    \footnotesize
    \begin{tabular}{|c|c|c|c|c|c|}
        \hline
         & \multicolumn{5}{|c|}{Adaptive Efficiency $\downarrow$} \\ \cline{2-6} 
        Exploration & DoorKeyChange & LavaNotSafe & LavaProof & CrossingBarrier & ThighIncrease \\ 
        Algorithm & ($10^{6}$) & ($10^{6}$) & ($10^{4}$) & ($10^{5}$) & ($10^{6}$) \\ 
        \hline 
        \hline
        None (PPO) & 1.5 $\pm$ 0.477 & 2.56 $\pm$ 2.09 & \textbf{2.05 $\pm$ 0.0} & 6.48 $\pm$ 3.15 & 3.4 $\pm$ 2.51 \\
        \hline
        NoisyNets & \textbf{0.965 $\pm$ 0.204} & 0.963 $\pm$ 0.534 & 7.58 $\pm$ 9.15 & 5.88 $\pm$ 3.72 & 1.69 $\pm$ 0.538 \\
        ICM & 1.57 $\pm$ 0.589 & 7.58 $\pm$ 1.26 & \textbf{2.05 $\pm$ 0.0} & 7.69 $\pm$ 4.69 & 4.18 $\pm$ 1.39 \\
        DIAYN & 1.52 $\pm$ 0.422 & 3.65 $\pm$ 2.47 & 5.8 $\pm$ 5.31 & 5.43 $\pm$ 3.71 & \textbf{1.66 $\pm$ 0.389} \\
        RND & 1.23 $\pm$ 0.385 & 4.64 $\pm$ 3.63 & \textbf{2.05 $\pm$ 0.0} & 5.25 $\pm$ 2.39 & 2.81 $\pm$ 1.46 \\
        NGU & 1.58 $\pm$ 0.317 & 2.39 $\pm$ 1.38 & 6.4 $\pm$ 11.5 & 4.41 $\pm$ 4.02 & 3.71 $\pm$ 1.74 \\
        RIDE & 1.53 $\pm$ 0.527 & 4.51 $\pm$ 2.34 & 2.56 $\pm$ 1.35 & 5.32 $\pm$ 3.71 & 5.18 $\pm$ 2.73 \\
        GIRL & 1.57 $\pm$ 0.541 & 5.49 $\pm$ 3.1 & \textbf{2.05 $\pm$ 0.0} & 6.31 $\pm$ 4.57 & 3.08 $\pm$ 1.98 \\
        RE3 & 1.21 $\pm$ 0.312 & \textbf{0.896 $\pm$ 0.21} & \textbf{2.05 $\pm$ 0.0} & \textbf{4.14 $\pm$ 2.07} & 4.32 $\pm$ 1.81 \\
        RISE & 1.41 $\pm$ 0.374 & 1.37 $\pm$ 0.478 & \textbf{2.05 $\pm$ 0.0} & 4.67 $\pm$ 3.26 & 3.6 $\pm$ 0.597 \\
        REVD & 1.27 $\pm$ 0.319 & 2.43 $\pm$ 1.34 & 2.87 $\pm$ 1.64 & 5.43 $\pm$ 3.24 & 3.92 $\pm$ 0.202 \\
        \hline
    \end{tabular}
    \caption{This table shows the mean and variance of the adaptive efficiency on the post-novelty tasks. It is computed by calculating the number of steps from the start of the novel task until convergence on the second task. Thus, lower numbers are better. Only runs that converged on both tasks are taken into account for this metric.}
    \label{tab:transx:adaptive_efficiency}
\end{table}

\begin{table}[ht]
    \centering
    \footnotesize
    \begin{tabular}{|c|c|c|c|c|c|}
        \hline
         & \multicolumn{5}{|c|}{Transfer Area Under Curve $\uparrow$} \\ \cline{2-6} 
        Exploration & DoorKeyChange & LavaNotSafe & LavaProof & CrossingBarrier & ThighIncrease \\ 
        Algorithm & ($10^{-1}$) & ($10^{-1}$) & ($10^{-1}$) & ($10^{-1}$) & ($10^{2}$) \\ \hline 
        \hline
        None (PPO) & 7.72 $\pm$ 0.792 & 7.43 $\pm$ 1.29 & 9.66 $\pm$ 0.0835 & 8.89 $\pm$ 0.297 & 6.5 $\pm$ 1.63 \\
        \hline
        NoisyNets & \textbf{8.13 $\pm$ 1.23} & \textbf{8.37 $\pm$ 0.885} & 7.69 $\pm$ 3.36 & 8.94 $\pm$ 0.388 & 8.62 $\pm$ 0.39 \\
        ICM & 7.28 $\pm$ 1.07 & 5.43 $\pm$ 0.667 & 9.22 $\pm$ 1.16 & 8.74 $\pm$ 0.537 & 7.25 $\pm$ 1.56 \\
        DIAYN & 7.54 $\pm$ 0.624 & 6.25 $\pm$ 1.22 & \textbf{9.7 $\pm$ 0.0773} & 9.01 $\pm$ 0.493 & \textbf{8.72 $\pm$ 0.203} \\
        RND & 8.09 $\pm$ 0.542 & 6.25 $\pm$ 1.53 & 9.37 $\pm$ 0.66 & 9.0 $\pm$ 0.399 & 7.47 $\pm$ 1.73 \\
        NGU & 7.56 $\pm$ 0.508 & 6.86 $\pm$ 1.38 & 9.48 $\pm$ 0.38 & 9.09 $\pm$ 0.444 & 7.07 $\pm$ 1.68 \\
        RIDE & 7.67 $\pm$ 0.727 & 7.63 $\pm$ 0.895 & 9.5 $\pm$ 0.605 & 9.02 $\pm$ 0.373 & 5.76 $\pm$ 1.67 \\
        GIRL & 7.59 $\pm$ 0.855 & 6.01 $\pm$ 1.08 & 9.55 $\pm$ 0.295 & 8.86 $\pm$ 0.51 & 7.45 $\pm$ 1.74 \\
        RE3 & 8.12 $\pm$ 0.387 & 6.82 $\pm$ 1.48 & 9.37 $\pm$ 0.524 & \textbf{9.1 $\pm$ 0.266} & 6.77 $\pm$ 1.99 \\
        RISE & 7.35 $\pm$ 1.08 & 7.07 $\pm$ 1.77 & 9.42 $\pm$ 0.402 & 9.09 $\pm$ 0.343 & 7.05 $\pm$ 1.03 \\
        REVD & 7.99 $\pm$ 0.402 & 7.3 $\pm$ 1.68 & 9.69 $\pm$ 0.056 & 8.92 $\pm$ 0.384 & 5.58 $\pm$ 1.33 \\
        \hline
    \end{tabular}
    \caption{The mean and variance of the transfer area under the curve metric, which is computed by adding final reward on the first task with the area under the reward curve in the second task. Higher numbers indicate better adaptation. This only includes runs that converged on the first task.}
    \label{tab:transx:trauc}
\end{table}


\paragraph{Compared to delta novelties, shortcut novelties increase the importance of diversity principles, while barrier novelties demonstrate the limitations of exploration to improve transfer in general.}
On delta novelties DoorKeyChange and ThighIncrease, stochastic methods have very good general performance.
However, in the LavaProof shortcut novelty, the stochastic method NoisyNets fails to adapt and find the shortcut novelty, whereas the DIAYN stochastic method excelled. 
DIAYN differentiates itself from NoisyNets by combining elements of stochasticity with explicit diversity. 
As can be seen in Figure~\ref{fig:transx:shortcut}, globally temporal methods NGU, GIRL, and ICM also fail to consistently identify the shortcut over the safe lava in spite of learning how to safely navigate around it in the source task. 
One possible explanation is that shortcut novelties have no performance drop that forces models to explore more.
In that case, it illustrates a scenario in which implementation of an exploration principle would be very important, such as to require a principle of explicit diversity.

On the other extreme, the barrier novelty results from CrossingBarrier and LavaNotSafe showed methods with exploration principles of stochastic and explicit diversity generally continued to be most effective, but there is larger variance between methods within and across the categories.
This difference in variance is especially obvious in the CrossingBarrier task, where adaptive efficiency and Tr-AUC variances are as high as $91.1\%$ of the mean. 
These findings suggest the limits of exploration to improve transfer.
For the barrier novelties, there is the target task solution is significantly longer than the target task solution compared to barrier novelties, meaning that often less prior knowledge can be transferred. 
Thus, at the extreme, online test time adaptation for a barrier novelty is akin to learning two single-tasks with no prior knowledge, as compared to online test time adaptation for a delta or shortcut novelty. 
It illuminates online test time adaptation's implicit assumption that some knowledge learned in the source task can be transferred to the target task, and suggests that most general purpose exploration methods, such as those studied in this work, are unlikely to benefit policy adaptation to difficult barrier novelties in general. 



\section{Key Takeaways}\label{sec:transx:conclusion}

The understanding of online test time adaptation for reinforcement learning developed in Chapter~\ref{chapt:novgrid} leads us to see the importance of exploration for adaptation, but most prior exploration methods are not designed for adaptation.
In an effort to determine which characteristics of traditional exploration algorithms are important for efficient adaptation in RL, we evaluated several deep reinforcement learning exploration algorithms on a number of OTTA problems. 
Our results and analysis reveal three key findings:
(1)~Exploration principles of \textit{explicit diversity}, represented by a method such as RE3, and \textit{stochasticity}, such as NoisyNets, are the most consistently positive exploration characteristics across our novelty and environment types. 
(2)~Time-independent and stochasticity-based exploration methods
are best suited to online test time adaptation in the continuous control tasks, whereas temporal locality characteristics are less important in discrete control tasks. 
(3)~The relative importance of exploration characteristics like explicit diversity varies with novelty type. 
The findings that stochasticity and explicit diversity in exploration lead to more efficient adaptation demonstrate that carefully designed exploration strategies can improve an agent's ability to adapt online to environmental changes. 
The fact that these characteristics outperform separate objective and temporally local exploration approaches suggests that task-agnostic exploration principles (such as stochasticity and diversity) are more beneficial for adaptation than task-specific exploration strategies. 
Taken together, these results highlight the benefits to agent adaptation efficiency of understanding characteristics of the environment, the agent's exploration approach, and potential novelties, as well as the importance of the relationship between these characteristics.

In addition to the results themselves, the scale of this research effort contributes a sweeping baseline of exploration approaches for future OTTA research. 
That said, the limitations of this work provide opportunities for future investigation. 
Specifically, the characterization of exploration algorithms can be used to produce approaches tailor-made for OTTA problems, both for individual novelties or many novelties.
While all of the results presented in this chapter cover known stationary \textit{tabula rasa} RL exploration approaches, the algororithms each have differing but overlapping combinations characteristics. 
The fact that slightly different combinations of characteristics impacts results in specific novelties suggests that designing algorithms by combining the best exploration characteristics for a given environment and novelty could maximize an agent's adaptive efficiency.

Taking that idea of combining characteristics one step further and considering the notion of \textit{objective mismatch}~\cite{lambert2020mismatch} (which we explore in detail in Chapter~\ref{chapt:dops}), an interesting direction of future research would be examining how an agent might effectively combine multiple different exploration algorithms.
By having a modular exploration approach that employed different exploration characteristics based on the nature of the novelty, agent would be able to adapt efficiently to many different novelties with only limited understanding of the novelty.

The key insight provided by the work in this chapter is that design of reinforcement learning OTTA agents must take into consideration the nature of the problem setting and the how this interacts with the agent's capacity to explore. 
The data sampled through exploration is critical for adaptation, and, as the results in this Chapter show, exploration methods ideal 
with respect to pre-novelty policy convergence 
are not necessarily best suited to adaptation.
By considering exploration as an attribute of the agent that depends on the environment and potential for novelty, agents can adapt more capably either through a single ideal exploration method or exploration specifically selected according to a novelty. 

The insights provided by this chapter's results further supports the broader thesis of this dissertation.
Specifically, the results in this chapter demonstrate that efficient adaptation requires exploration strategies that prioritize reducing task-overfitting and the distribution shift between $MDP_{\mathrm{source}}$ and $MDP_{\mathrm{target}}$ learning data.
We built on these findings to do the work described in Chapter~\ref{alg:dops:dops} where we complement our investigation of exploration---how an agent acquires learning data---with how that agent then decides which data is worth learning on.


%% file: chapters/5_dops_chapt.tex
\chapter{Dual Objective Priority Sampling in Model-based Reinforcement Learning}
\label{chapt:dops}

Model-based reinforcement learning (MBRL)
is theoretically well suited for the problem of online test-time adaptation (OTTA) as natural environment changes may have a large effect on the optimal policy but only cause a small change to the environment. 
Given that world models may require less adaptation than policy models, an agent should theoretically be able to more efficiently adapt its policy by also modeling the dynamics of the world.
However, even state-of-the-art MBRL approaches such as the Dreamer~\cite{hafner2019dream,hafner2020dv2,hafner2023dv3} family of algorithms struggle to adapt efficiently in part due to the way learning data are sourced and sampled in MBRL.
Because learning data are sourced from environment interactions by agent behavior, the distribution of learning data is biased toward states frequently visited by the optimal policy. 
As a result, the longer training continues the more likely it is that world model will overfit to on-policy environment dynamics. 
On the other hand, the world model samples data from a buffer for world model learning without regard to whether these data are useful to the policy learning objective. 
Prior work~\cite{lambert2020mismatch} describes the tension between the optimization objectives of the world model and the policy as \textit{objective mismatch}. 
To alleviate the problems caused by mismatched objectives, the policy and world model training processes should be independent enough to prioritize differing objectives while aligned enough to avoid large distribution mismatches between policy and world models. 

Compared to stationary MDP problem settings, OTTA settings further complicate the balance between independence and alignment of objectives because adapting to non-stationarity also causes different \textit{distribution shifts} for policy and world models. 
When adapting source MDP models to the target MDP, a small shift in the transition distribution will often correspond to a very differently distributed optimal policy.
As learning new policies will impact the state distribution visited by the agent, efficient policy and world model adaptation requires the world model to accurately represent the environment beyond just the states frequently visited by the source policy. 

As an example, imagine a commuter driving to work in a new city. 
In the typical case of a city where all dynamics are stationary, or changes are short lived and stochastic enough to be solved through robustness, finding a route that minimizes commute time will converge to a single route. 
Once converged, the commuter's mental map of the city will be limited to the optimal route to work, because that is the only part of the city they experience. 
If suddenly known roads are blocked or a new shortcut opens up, the commuter should adapt their route.
However, if the commuter's mental model is strongly biased to only the original route, then the commuter will see no benefit from having a mental model of the city.
In fact, trying to update the mental model at the same time as finding a new route to work is more mentally taxing than if a mental model was not used to find new routes. 

If a world model predicts incorrect future states due to overfitting then the world model may negatively impact the policy. 
Problems with overfitting can be improved by increasing exploration or further biasing sampling toward world model coverage; however, any additional learning or sampling bias towards the world model threatens to further slow the adaptation of the policy.

To resolve the tension between these disjoint learning objectives, we introduce dual-objective priority sampling (DOPS), a novel sampling method for MBRL that enables more efficient learning and adaptation to OTTA problems.
DOPS increases learning and adaptation efficiency by enabling the policy and the world model to learn from training data that best suits a specific objective without undermining the need for shared learning distributions. 
Through theoretical analysis of the Dreamer architecture in the context of the OTTA learning problem, we identify causes of distribution shift within and between the different component models of Dreamer.
Our proposed sampling approach addresses the core problem of aligning distributions of learning data with model learning objectives.
Then we propose a low-overhead algorithm that combines all the disparate sampling solutions and discuss its computational complexity implications.
To verify the effectiveness of this method empirically, we evaluate learning and adaptation performance on OTTA problems implemented in Novelty Minigrid, a novelty injection modification of the DMControl-based~\cite{tunyasuvunakool2020dmcontrol} Real-World Reinforcement Learning environment. 

In summary, we make the following key contributions:
\begin{enumerate}
\item We analyze the different learning categories present in interleaved model-based reinforcement learning that use an actor-critic for optimizing agent behavior, extend the objective-mismatch hypothesis to a consideration of distinctions between models with different learning signals, how models of each learning adversely affected by OTTA problems, and how prioritized sampling can compensate.
\item We formulate and analyze the dual objective priority sampling (DOPS) algorithm for addressing specific challenges of adaptation in interleaved model-based RL methods.
\item We demonstrate that DOPS improves adaptation performance over Dreamer and state-of-the-art Curious Replay in adaptation-focused environments Novelty MiniGrid and the MuJoCo-based Real World-RL Suite.
\end{enumerate}

\section{Preliminaries}
\label{sec:dops:prelim}

\subsection{Sampling Training data in Dreamer MBRL Models}

Recall from Chapter~\ref{chapt:background} that Dreamer~\cite{hafner2019dream} refers to a Dyna-style~\cite{sutton1991dyna} model-based reinforcement learning algorithm and architecture, and forms the basis of a family of MBRL techniques~\cite{hafner2019dream,sekar2020planning,hafner2020dv2,mendonca2021lexa,hafner2022director,hafner2023dv3,kauvar2023curious}.
Dreamer's architecture is broadly divided into two end-to-end updated modules: the RSSM-based \textit{world model} (defined in Equation~\ref{eq:dreamer:wm}) and a latent actor-critic \textit{behavior model} (defined in Equation~\ref{eq:dreamer:actorcritic}).

In the original Dreamer learning algorithm, given a replay buffer of sequences from prior agent interactions with the environment, a batch of sequences is sampled uniformly from the buffer for updating both the world model and the behavior model. 
At each training step, the Dreamer algorithm first updates the world model, which includes representation learning and prediction learning.
After the world model update, Dreamer uses the embeddings of the same data to update the behavior model, which includes policy learning and critic learning. 

Sampling old transitions from a replay buffer can be problematic for model-free variants of on-policy agents because old transitions become ``stale,'' meaning that they may not reflect the current value and policy distributions. 
In the Dreamer algorithm, however, this is not a problem. 
World model learning is not affected by the old data because the representation and prediction of the dynamics are independent of the policy distribution. 

Behavior learning is also unaffected by the problem of stale data distributions.
Instead of learning directly from the sampled interaction data, Dreamer's latent actor-critic learns from ``imagined'' rollouts in latent space. 
The embedded states of the real interaction data are used to initialize the behavior learning rollouts, and then the rollouts are imagined by the actor selecting on-policy actions, transition model predicting the next embedded state, the reward model predicting the reward of that state, and then repeating this process for finite horizon of steps. 
The resulting latent state-action distribution is on-policy and therefore well suited for behavior learning.

\subsection{Objective Mismatch in Model-based Reinforcement Learning}

As discussed at length in Chapter~\ref{chapt:background}, reinforcement learning can be broadly separated into model-\textit{based} and model-\textit{free} reinforcement learning. 
While it has advantages in learning efficiency, model-based reinforcement learning suffers from objective mismatch~\cite{lambert2020objective}. 
This occurs because the objectives maximized by the policy learning and world model learning processes are neither fully aligned nor fully separable.
The policy's objective is to learn the distribution over sequential actions that maximizes future expected task reward, while the world model objective is to minimize error in its representation and prediction of environment transition dynamics independent of the task. 
Yet these two tasks are at odds, where a more accurate, task-agnostic world model will yield a worse policy than a less accurate, task-focused world model~\cite{lambert2020objective}. 
This is in large part attributable to the fact that, unlike other machine learning problems like supervised learning, data for the world model is sampled according to some entirely off-policy or offline data collection, or, as in the case of the Dreamer family of algorithms, both policy and world model learning learning data is collected and sampled non-i.i.d (independently identically distributed) according to the exploration-exploitation strategy of the policy. 

Thus, the objective mismatch in model-based reinforcement learning stems from several key factors:
\begin{enumerate}
\item Compounding Errors: The world model $f_\theta$ is typically trained to minimize one-step prediction errors. 
However, policy optimization often requires multi-step predictions, leading to error accumulation.
\item Reward Sparsity: In many tasks, rewards are sparse, making it challenging for the model to capture reward-relevant dynamics. 
Eysenbach et al. \cite{NEURIPS2022_935151cc} demonstrated that in sparse reward settings, models trained to minimize prediction error often fail to capture critical task-relevant information.
\item Non-uniform State Visitation: The optimal policy $\pi^*$ induces a state visitation distribution that often differs significantly from the distribution in the training data $\mathcal{D}$. 
This discrepancy, formalized by Levine et al. \cite{levine2020offline} as:
\begin{equation}
\mathbb{D}_{\text{KL}}(p_{\pi^*}(s) || p_{\mathcal{D}}(s)) \geq \epsilon
\end{equation}
where $p_{\pi^*}(s)$ and $p_{\mathcal{D}}(s)$ are the state distributions under the optimal policy and training data respectively, can lead to poor model performance in critical regions of the state space.
\end{enumerate}
\noindent State visitation non-uniformity is especially relevant in this dissertation as efficient adaptation from prior knowledge implies a need for accurate the world model prediction in parts state space potentially far from the current optimal policy. 
As in the adapting commuter example, the bias of the commuter's experience ultimately results in an adaptation process that is made more complex by the need to update both the policy and world models.

Although objective mismatch can lead to various issues, its impact on sample efficiency is particularly concerning, as it undermines one of the primary motivations for using model-based methods: sample efficiency. 
For example, sample \textit{inefficiency} can be caused by suboptimal exploration strategies. 
As discussed in Pathak et al.~\cite{pathak2017icm} and examined as one of the methods in Chapter~\ref{chapt:transx}, naive model error-based ``curiosity'' exploration methods can become sample inefficient by focusing on stochastic or irrelevant aspects of the environment.

While the model-policy objective mismatch is often acceptable for single-task reinforcement learning problems as overfitting the world model to the on-policy data distribution is tolerable when the goal, addressing this mismatch is critical when the environment changes as model-based agents are dependent on a model to be accurate on-policy.
Saemundsson et al.~\cite{Smundsson2018MetaRL} showed that models trained to minimize prediction error often struggle to transfer to new tasks, even when the underlying dynamics remain unchanged. 
Dreamer's underlying learning formulation learns a latent stochastic policy that is also used for sampling from the environment. 
However, as observed in prior work~\cite{dorka2023dynamic} and in this dissertation (Chapter~\ref{chapt:knowledge}), these decisions can worsen overfitting and, as a result, test time adaptation efficiency. 
Similarly to the ideas presented in this Chapter, there has been some effort to combat these overfitting problems with exploration~\cite{sekar2020planning,mendonca2021lexa,kauvar2023curious,ma2022maxent} or sampling~\cite{kessler2023effectiveness}. 
However, none of these works directly address the underlying issue of objective mismatch. 
We are not the first to consider using multiple replay priorities or buffers to solve objective mismatch problems. 
Laroche et. al.~\cite{larocheNEURIPS2021}, for example, uses independent policies and buffers for on-policy exploration and off-policy exploitation to enable a more explicit mixture of on-policy and off-policy updates.

\section{Dual Objective Priority Sampling}
\label{method}

Data prioritization methods in model-based RL---including both exploration methods~\cite{kauvar2023curious} and direct sampling methods~\cite{kessler2023effectiveness}---
use a non-uniform distribution of sample importance to provide the agent with data that improves learning. 
In light of objective mismatch, we must then ask the question: these prioritization methods optimize the learning data distribution \textit{with respect to which objective}? 

To avoid interference between the model and policy objectives in online test-time adaptation, we formed the following adaptive sampling requirements: 
\begin{enumerate}
    \item The data used for world model learning and behavior learning should be sampled according to the specific objective of each,  
    \item The sampling learning data for adaptation should balance model-specific priorities with the negative impacts of distribution shift, and 
    \item The data prioritization methods should consider the need of the replay buffer to supply both the world model and behavior policy with data.
\end{enumerate}


In this section, we propose the dual objective priority sampling (DOPS) method that we designed to reflect these adaptive sampling requirements.
We first identify how the specifics of Dreamer's architecture and learning algorithm present unique challenges in OTTA problems, and analyze these challenges with insights inspired by Curious Replay (CR)~\cite{kauvar2023curious} and Actor-Prioritized Experience Replay (LA3P)~\cite{saglam2023la3p}. 
We develop a sampling method to prioritize data according to the separate constraints and objectives of the model, actor, and critic. 

\subsection{Adaptive Sampling for Dreamer}

The Dreamer architecture can be decomposed into learning categories according to the distinct gradient distribution during learning. 
Related to the objective mismatch hypothesis of Lambert et. al.~\cite{lambert2020mismatch}, we believe that the objective mismatch in Dreamer extends to more than the model vs. the policy.  
In Dreamer, different components of the architecture are trained based on varying means of estimating the quality of a model, which leads to distinct groups that respond differently to adaptation. 
The Dreamer model components fall into one of four learning categories: 
\begin{enumerate}
    \item \textit{Prediction learning}, which includes the observation, reward, and discount prediction models, is characterized by the fact that the models are only experience gradients originating from estimating quality as error in regression of real data values. 
    \item \textit{Representation learning}, which includes the observation encoder, the trajectory model, and the dynamics prediction model.  
    The representation learning models are all configured to project an input into a learned compact embedding space, and the quality of this projection dictating the gradients are the divergences of these embeddings, either from different priors or based on the change of the embeddings through time. 
    \item \textit{Policy learning} of the latent actor policy, where because Dreamer uses an actor-critic learning approach,  gradients are a function of the critic-weighted ``policy gradient.'' 
    This means that the gradients reflect an approximation of the rate of change to the policy parameters that maximizes the expected reward. 
    \item \textit{Critic learning} of the latent critic value function used to approximate the discounted future reward of state action pairs. 
    The estimate of critic quality that determines these gradients is the difference between the change in reward and the change in discounted critic estimate through time. 
    As discussed in Chapter~\ref{chapt:background} this is called the TD error.  
\end{enumerate}

The Dreamer algorithm uses ``interleaved'' training, meaning that, unlike some model-based approaches~\cite{ha2018recurrent,micheli2023iris}, the behavior and world models are both for each iteration of a training loop on a single, shared batch of data drawn from the replay buffer. 
This leads to gradients from some losses impacting the multiple learning categories both directly and indirectly.
However, in spite of gradient mixing, the distinctions in gradient source are enough to differentiate the behavior of the gradient in models of different learning categories.

For efficient adaptation to sudden novel change, uniform sampling can be suboptimal for both the world model and behavior learning processes. 
Sampling of pre-novelty data may reinforce the errors in modeling and decision making only noticeable with respect to recent, post-novelty data relevant to the novelty.
Moreover, as the data is sourced from agent interaction that are highly biased towards solving the $MDP_{\mathrm{source}}$, uniformly sampled data will be neither diverse nor distributed with respect to the new $MDP_{\mathrm{target}}$.

In addition, adaptation changes the learning problem for models of all learning categories.
In \textit{tabula rasa} learning, small adjustments to initial random parameters in the direction of increased quality mean the likelihood of parameter adjustments leading to a model is low.
Recall from Chapter~\ref{chapt:background} that in the formulation of novelty-based OTTA a key difference from other forms of non-stationarity is the constrain that the $MDP_{\mathrm{source}}$ is related to the $MDP_{\mathrm{target}}$ by a \textit{knowable} transformation.
We refer to novelties as ``knowable'' to convey the fact that they can be defined as such a transformation. 
However, a novelty being knowable doesn't not mean that the transformation from $MDP_{\mathrm{source}}$ to $MDP_{\mathrm{target}}$ is trivial to model with gradient descent. 
As a result, uniform distributed data constitutes a large distribution shift with respect to the pre-novelty model leading to catastrophic forgetting. 

\subsection{Sampling for the World Model}

As discussed previously, any small change to $MDP_{\mathrm{source}}$ in OTTA can cause a very large distribution shift for the actor or critic.
However, the impact of data shift on representation and prediction learning when changing from $MDP_\mathrm{source}$ to $MDP_\mathrm{target}$ is related to the attributes that characterize the change. 
These attributes are complementary to novelty characterization theories from~\cite{boult2021towards} as well as those discussed in Chapter~\ref{chapt:novgrid}:
\begin{enumerate}
    \item \textbf{Differences in the dynamical process} of the domain, or the way the state changes from time step to time step. 
    This includes differences in control, which changes the way agent decisions affect the state. 
    For example, an agent trained to drive on roads re-tasked to operate off-road in dirt and mud. 
    The more extreme the difference in the dynamics, the greater the distribution shift experienced by the representation learning models.
    \item \textbf{Differences in the distribution of percept features} of the domain, such as the appearance of new entities, known entities occurring in unknown contexts, or a change in the appearance of known entities. 
    For example, an agent trained to drive only during the day re-tasked to operate at night will be surrounded by a different distribution of vehicles and perceive those vehicles differently. 
    The more extreme the difference in the distribution of percept features, the greater the distribution shift experienced by the prediction learning models.
    \item \textbf{Differences in task}, including changes in the distribution of rewards, initial conditions, and terminal conditions of the MDP. 
    An example of this is adapting an agent trained to map a glass-walled maze to find the shortest path to the center of a maze with pits instead of walls. 
    Changes in the distribution of terminal conditions and rewards like this does have a slight impact on transition, discount, and reward prediction models.
    However, as long as in both the $MDP_\mathrm{source}$ and $MDP_\mathrm{target}$ the task is roughly the same length, and the total return magnitude, and the majority of experiences involve non-terminal states, the prediction and representation learning models will not experience a major distribution shift regardless of how extreme the change. 
    This is because, while the task is dependent on the domain, the domain is not dependent on the task.
\end{enumerate} 

Curious Replay (CR)~\cite{kauvar2023curious} provides insights to address the issues that affect primarily prediction learning models, but also to some extent representation learning models. 
The first suggestion of CR is to prioritize samples based on the number of times an experience has been used for training. 
This count-based~\cite{bellemare2016unifying} sampling emphasizes the importance of learning with newly-collected data, but if the agent ever stops collecting data count-based prioritization will, in the training limit, converge to uniform. 
As training data is collected over the course of an agent's lifetime in online RL and resampled frequently during interleaved training, this prioritizes recent data without adding additional learning bias with respect to a uniform sampler. 
This is beneficial in general to adaptive agents because of its emphasis on recent experience.

The second method employed in CR sampling is adversarial prioritization. 
Similar to traditional Prioritized Experience Replay (PER)~\cite{schaul2016per}, this simply prioritizes samples according to their most recent world model learning loss (Equation~\ref{eq:dreamer:wm_loss}. 
As discussed in Chapter~\ref{chapt:transx}, such ``curiosity''-based incentives can make agents susceptible to the ``Noisy-TV problem,''~\cite{pathak2017curiousity}, where an agent will continually just seek out and train on states that contain unpredictably changing observations, like someone scrolling through programs on a television. 
CR, however, is less vulnerable to this problem because the count-based priority balances out this effect, and because this incentive is provided during sampling not agent interaction. 
CR samples data according to the balanced sum of these two priorities. So for a given sample $s_i$, the CR priority is~\cite{kauvar2023curious}:
\begin{equation}
    pr_{\mathrm{CR}}(s_i) = c \beta^{\nu_i} + (|\mathcal{L}_i| + \epsilon)^{\alpha}
\end{equation}
\label{eq:dops:cr}
Although OTTA is beyond the scope of the original CR work, because the combination of count-based and adversarially-based priorities provides a means by which to use old replay data without harming the adaptation process, we reason that it is well-suited to OTTA problems.
By emphasizing novelty in a way that balances old and new replay data and slowly phases out the older data, CR sampling can alleviate the distribution shift experienced by the world model. 
Work in offline-to-online reinforcement learning (discussed in greater detail in Chapter~\ref{chapt:cbwm} and the Appendix) shows that mixing the tapered mixing of RL loss gradients with (BC) loss gradients can lessen distribution shift when fine tuning BC policies with online RL.
So too with CR sampling, blending data from the old and new distributions provides training batches that ease the distribution shift that can lead to catastrophic gradients in representation and prediction learning. 

\subsection{Sampling Data for the Actor and Critic}

When adapting, the actor and critic models experience distribution shift regardless of the novelty characteristics. 
Dreamer's latent actor-critic can lessen this somewhat because the sudden shift in observations and dynamics can manifest as a more smooth shift in embedding space. 
However, 
because the actor learns to select actions with respect to the critic as a surrogate for future reward, the critic is key to distribution shifts in both the actor and critic models.

There are two main ways that behavioral distribution shift manifests through the critic. 
Firstly, the change from $MDP_\mathrm{source}$ to $MDP_\mathrm{target}$ in OTTA problems causes a distribution shift in the critic because imaginary rollouts are incorrectly valued with respect to the new reward and termination distributions. 
This will result in unusually large TD-errors that, instead of being smoothed out by TD-lambda, are actually amplified by TD-lambda as error will accumulate over longer trajectories. 
As data continues to be sampled that has a preference toward areas where the model is unfamiliar, because a converged Dreamer critic will both be biased towards the old policy, the critic will experience a number of poor value estimates with high gradients which, being disproportional to global critic accuracy, can lead to gradient overshooting and therefore catastrophic forgetting. 

We can correct for this shift and maximize adaptive efficiency by following the recipe of Prioritized Experience Replay (PER)~\cite{schaul2016per}, which uses importance sampling to prioritize high-TD error states for behavior learning. 
For Dreamer to work the latent actor-critic must remain on policy. 
So instead, we sample the initial states by prioritizing the total TD error of \textit{trajectory sequences}. 
By ensuring that the entire trajectory used to initialize the imaginary rollout has a high TD-error, we maximize the likelihood of learning on high TD-error imaginary trajectories.

There is additional bias in loss calculations when  using TD-based prioritization.
Schaul et. al.~\cite{schaul2016per} recognized this and suggest that exchanging an MSE loss typical in actor-critic for a Huber loss can help reduce errors from this bias~\cite{schaul2016per}:
\begin{equation}
    L_{\text{Huber}}(\delta_{TD}(\tau_i)) = \begin{cases} 
    \frac{(\delta_{TD}(\tau_i))^2}{2} & \text{if } |\delta_{TD}(\tau_i)| \leq 1 \\
    |\delta_{TD}(\tau_i)| - \frac{1}{2} & \text{otherwise}
    \end{cases}
\label{eq:dops:huber}
\end{equation}
Fujimoto et. al.~\cite{fujimoto2020equivalence} builds on this insight and suggested that a loss that is more to TD-error bias is a Huber loss modified to only calculate $L_2$ gradients over uniform data. 
Limiting the data exposed to $L_2$ can be done without modifying the loss when blending prioritized and uniform sampling by thresholding the priority at a $1$~\cite{saglam2023la3p}:
\begin{equation}
    pr_{\mathrm{PER}} = \left(\frac{\max(|\delta_{TD}(\tau_i)|^{\alpha}, 1)}{\sum_j \max(|\delta_{TD}(\tau_j)|^{\alpha}, 1)} \right)
\label{eq:dops:per}
\end{equation}
For our purposes, we find that this is additionally beneficial for solving OTTA problems. 
By making the TD-error loss linear in the limit instead of quadratic, we reduce the risk of gradient overshoot when adapting the critic.

The second cause of distribution shift in behavior learning is caused by the propagation of large TD-errors to the actor.
These high-TD errors can result from critic distribution shift or other causes like the prioritized selection of high-TD samples for learning. 
Saglam et. al.~\cite{saglam2023la3p} finds that if there exist transitions for which TD-error increasing corresponds to Q-value estimation error increasing for future states, the computed policy gradient will diverge from the true policy gradient for at least the current step (~\cite{saglam2023la3p} Theorem 1).
We already know that this occurs in the TD-lambda estimates in OTTA, and Saglam et. al.~\cite{saglam2023la3p} demonstrate in their work that this also occurs naturally when using PER~\cite{schaul2016per} for off-policy actor critics. 
Dreamer's latent actor-critic is on-policy.
However, we argue here that due to the nature of Dreamer's sampling process the conclusions of Saglam et. al.~\cite{saglam2023la3p} still apply by biasing behavior learning through the initialization based on a replay buffer. 
Specifically, because each step in a continuous trajectory is considered an initial starting point, even with uniform sampling the initial behavior states are \textit{non-i.i.d.}. 
While reinforcement learning is no-regret in theory, in practice  diversity is helpful for both learning and robustness of policy learning~\cite{sutton2018reinforcement} (see Chapter~\ref{chapt:transx} for more discussion on the values of diversity in adaptation of on-policy RL). 
The priority-based sampling of both the world model and the critic further worsens this problem as transitions will, in general, have higher than average TD-error and be less distributed. 

To address the distribution shift in the actor caused by high TD-error, we propose to follow the findings of Saglam et. al.~\cite{saglam2023la3p} and emphasize \textit{low-TD error} transitions for actor learning. 
We therefore follow draw samples for the actor prioritizing the inverse of the TD-error priority. 
\begin{equation}
    pr_{\mathrm{iPER}}= \frac{\sum_j \max(|\delta_{TD}(\tau_j)|^{\alpha}, 1)}{\max(|\delta_{TD}(\tau_i)|^{\alpha}, 1)}
\label{eq:dops:iper}
\end{equation}
By sampling for low TD-error transition sequences to initialize the imagined sequences for actor training, we dramatically reduce the risk distribution shift in the actor from any sources of high TD-error, whether from sampling or adapting to $MDP_{\mathrm{target}}$.

\subsection{Shared Transitions with Multiple Priorities}

Returning to our adaptive sampling requirements, 

\begin{enumerate}
    \item We have data prioritization methods that consider the mismatched needs objectives of representation learning, prediction learning, critic learning, and policy learning. 
    \item We have sampling approaches specific to adaptation that balance learning priorities with the negative impacts of distribution shift manifest of OTTA.  
\end{enumerate}

Three different prioritization methods, when implemented individually as separate SumTrees~\cite{schaul2016per} (as is typical) of $N$ transitions, each will take $O(log N)$ time to prioritize separations and triple the memory cost. 
If the actor and critic use a shared SumTree as suggested by Saglam et. al.~\cite{saglam2023la3p} will reduce the memory overhead but increase the sampling time to $O(N)$ in the worst case due to the computational requirements of calculating an inverse TD priority from a standard TD priority. 

In addition, we know that training all of these learning models on completely different data distributions violates the assumptions of ``interleaved'' model-based reinforcement learning theory~\cite{hafner2019planet} and actor-critic theory~\cite{konda1999actorcritic}.
This occurs because, fundamentally, the in-distribution performance of neural network-approximated functions does not reflect the out-of-distribution neural net performance. 
In Dreamer's learning configuration, there is complete codependency of these different elements.
The world model depends on the actor to provide new experiences that become progressively more optimal, the actor requires the world model and critic to correctly predict latent state transitions and approximate the reward respectively, and the critic depends on the world model and the actor to move the agent to more and more reward. 

We propose to merge these priorities by subsampling each batch of learning data according to the objectives and constrains of the world model, actor, and critic. 
First, a blend of CR and uniform samples is sampled by the world model. 
By blending some uniform samples in with CR samples, we can ensure that the world model training data still emphasizes novel transitions, but never so aggressively as to cause the critic to never predict correct values during adaptation. 
For the latent behavior learning, imagined latent trajectories are computed for each sample in the world model's learning data, and the TD-errors are calculated for each trajectory based on the initial error of the world model sample. 
Then, a percentage of those trajectories are masked to create two equal-sized batches of that distributed specific to the objectives of the critic and policy learning processes. 
Given the fraction of actor-critic data overlap $W\in [0,1]$, the critic and policy learning batches each overlap with $\frac{1}{2-W}$ of the data used in world model learning.
This method is significantly faster with much less overhead: per step the buffer is only prioritized once by CR and the batch data is only sorted according to TD-error. 
This gives this algorithm an efficiency of $O(log N + log B)$, where $B$ is the batch size.

This algorithm for Dual Objective Prioritized Sampling (DOPS) is described in the context of the Dreamer algorithm in Algorithm~\ref{alg:dops:dops},
where we have highlighted the steps that distinguish DOPS from traditional Dreamer in \textcolor{blue}{blue}.
While we analyze, present, and implement the DOPS algorithm in the context of Dreamer DOPS can in theory benefit any interleaved MBRL algorithm that uses an actor-critic to model behavior.
That said, Dreamer is well-suited to the masking strategy in DOPS because---due to the compact nature of the latent actor-critic training samples and the nature of rolling out H-step imaginary trajectories for each real initial state---there are a lot of actor-critic samples generated from each sample batch. 
As a result, there is more flexibility to reduce the actor-critic batch size.

\begingroup
\singlespacing 

\begin{algorithm}[ht]
\caption{Dreamer with \textcolor{blue}{Dual Objective Priority Sampling}}
\KwIn{\textcolor{blue}{Curious Replay-prioritized replay buffer $\mathcal{D}$.}}
\KwIn{An interactive environment ``env''.}
\KwIn{Neural network parameters $\theta$, $\phi$, $\psi$, including model components: representation model $p_\theta(s_t \mid s_{t-1}, a_{t-1}, o_t)$, transition prediction model $q_\theta(r_t \mid s_{t-1}, a_t)$, reward model $q_\theta(r_t \mid s_t)$, policy model $\pi_\phi(a_t \mid s_t)$, and value model $v_\psi(s_t)$.}
\KwData{Given hyperparameters: collect interval $C$, batch size $\parallel B\parallel$, sequence length $L$, imagination horizon $H$, and learning rate $\alpha$.}

\While{not converged}{
    \For{update step $c = 1$ \KwTo $C$}{
        Sample batch of $B$ transitions from $\mathcal{D}$ weighting selecting transition $i$ according to normalized CR score $pr_{\mathrm{CR}}(s_i)$\;
        Compute model states $s_t \sim p_\theta(s_t \mid s_{t-1}, a_{t-1}, o_t)$\;
        Update $\theta$ using representation learning with Equation~\ref{eq:dreamer:wm_loss} \;
        \textcolor{blue}{Update the latent actor-critic with \textbf{Algorithm~\ref{alg:dops:behavior}: Subsampled Behavior Learning}}\;
        \For{each transition $i$ in batch}{
            Update visit count $\nu_i \leftarrow \nu_i + 1$\;
            Calculate priority $pr_i$ using Equation~\ref{eq:dops:cr}\;
            Update priority $pr_i$ and $\delta_{TD}$ for samples in $B$\;
        }
    }
    $o_1 \leftarrow \text{env.reset()}$\;
    \For{time step $t = 1$ \KwTo $T$}{
        Compute $s_t \sim p_\theta(s_t \mid s_{t-1}, a_{t-1}, o_t)$ from history\;
        Compute $a_t \sim q_\phi(a_t \mid s_t)$ with the actor model\;
        Add exploration noise to action\;
        $r_t, o_{t+1} \leftarrow \text{env.step}(a_t)$\;
    }
    Add experience to dataset $\mathcal{D} \leftarrow \mathcal{D} \cup \{(o_t, a_t, r_t)_{t=1}^T\}$, with each new transition added with priority $p_i \leftarrow p_\text{MAX}$ and visit count $\nu_i \leftarrow 0$\;
}
\label{alg:dops:dops}
\end{algorithm}
\endgroup

\begingroup
\begin{algorithm}[ht]
\caption{\textcolor{blue}{DOPS Subsampled} Behavior Learning}
\KwIn{Current policy parameters $\phi$, critic parameters $\psi$, batch of latent states $(s_t) \in B$, }
\KwData{Given hyperparameters: imagination horizon $H$, overlap fraction $W$, and learning rate $\alpha$.}
Imagine trajectories $\mathbf{\hat{\tau}}_t = \{(s_\tau, a_\tau)\}_{\tau=t}^{t+H} \in \hat{B}$ from each $s_t$ by iteratively computing $a_\tau \sim q_\phi(a_\tau \mid s_\tau)$ with the actor model then $s_\tau \sim p_\theta(s_{\tau+1} \mid s_{\tau}, a_{\tau})$ \;
Predict values and rewards for the imagined trajectories: \\
\For{time step $\tau = t$ \KwTo $t+H$}
{
    Predict rewards $\mathbb{E}(q_\theta(r_\tau \mid s_\tau))$ and values $v_\psi(s_\tau)$\;
    \If{$\tau == H+t$}
    {
    Compute $V_\tau^\lambda$ as   $V_\tau^\lambda \gets r_\tau + \gamma_\tau v_\xi(s_H)$ 
    }
    \Else
    {
    Compute $V_\tau^\lambda$ as $V_\tau^\lambda \gets r_\tau + \gamma_\tau \big((1 - \lambda) v_\xi(s_{\tau+1}) + \lambda V^\lambda_{\tau+1}\big)$\;
    }
}
\textcolor{blue}{Compute the of the imagined trajectories $\delta_{TD}$ with Equation~\ref{eq:background:td}}\;
\textcolor{blue}{Compute the TD-priority $pr_{\mathrm{PER}}$ of the trajectories in $\hat{B}$ with Equation~\ref{eq:dops:per}} \;
\textcolor{blue}{Subsample actor transitions $B_\pi$ as the \textbf{min-k} TD-priority samples:} \\
\textcolor{blue}{$\hat{B}_\pi = \underset{\mathbf{\hat{\tau}_t} \subseteq \hat{B}: |\hat{B_\pi}| = k*\mid \hat{B}\mid }{\mathrm{argmin}} \left(pr_{\mathrm{PER}}(\mathbf{\hat{\tau}}) \right)$, where $k=\frac{1}{2-W}$ \;}
Update policy parameters $\phi \leftarrow \phi + \alpha \nabla_\phi \sum_{\tau=t}^{t+H} \lambda(s_\tau)$\; 
\textcolor{blue}{Subsample critic transitions $B_\pi$ as the \textbf{max-k} TD-priority samples:} \\
\textcolor{blue}{$\hat{B}_\pi = \underset{\mathbf{\hat{\tau}_t} \subseteq \mathbf{\hat{B}}: |\hat{B_\pi}| = k*\mid \hat{B}\mid }{\mathrm{argmax}} \left(pr_{\mathrm{PER}}(\mathbf{\hat{\tau}}) \right)$, where $k=\frac{1}{2-W}$ \;
}  
Update critic parameters with Equation~\ref{eq:dops:huber} $\psi \leftarrow \psi - \alpha \nabla_\psi L_{\mathrm{Huber}}$;
\textcolor{blue}{Update the TD-error for $s_t \in B$ : $\delta_{TD} \leftarrow \max(|\delta_{TD}(s_t)|^{\alpha}, 1)$}\;
\label{alg:dops:behavior}
\end{algorithm}
\endgroup
\doublespacing

\section{Experiments}

As in Chapter~\ref{chapt:transx}, our experiments utilize two transfer learning frameworks: NovGrid, described in Chapter~\ref{chapt:novgrid}, and the Real World Reinforcement Learning (RWRL) suite~\cite{Dulac-Arnold2021rwrl} with NovGrid novelty injection. 
To evaluate online test time adaptation  capabilities, agents are initially trained to convergence in a source environment before introducing a novelty to create the target task.
At a certain number of environment interactions after the agent has converged, the novelty occurs---changing the environment from the $MDP_\mathrm{source}$ to $MDP_\mathrm{target}$ and thereby altering transition dynamics and the optimal policy.
This process is illustrated in  Figures~\ref{fig:novgrid:novgrid_splash} and~\ref{fig:novgrid:metrics}
The agent's ability to recover performance during online execution in the target environment is then assessed.

Extending the set of environments used in the work described in Chapter~\ref{chapt:transx}, this work is evaluated on a number of novelties. 
In NovGrid we tested the following novelties: 
\begin{enumerate}
    \item DoorKeyChange: A delta novelty where the key that opens the door in the DoorKey environment is altered.
    \item CrossingBarrierChange: Replaces safe, impassable walls with standard, terminal-state lava in the original Minigrid Crossing environment.
\end{enumerate}

In addition, in RWRL we tested the following novelties: 
\begin{enumerate}
    \item ThighLengthChange: a novelty in the Walker2D environment where the thigh link length is increased from 0.175 to 0.3 meters, or reduced from 0.3 to 0.175 meters. 
    \item TorsoDensityChange: a novelty in the Quadruped environment where the density of the torso element is doubled from 1000 to 2000 grams, and halved from 2000 to 1000 grams.
\end{enumerate}
For each of these environments, each algorithm was trained in five runs, each with a different random seed.  

The algorithms we used were implemented based on the original author's implementations of the DreamerV3 version of Curious Replay, which itself is based on the original author's implementation of DreamerV3 in Jax. 
We compared to Curious Replay and Dreamer as baselines. 
While other baselines such as Plan2Explore~\cite{sekar2020planning} are well suited to comparison in theory, in practice the results presented in the original Curious Replay work~\cite{kauvar2023curious} demonstrate that Curious Replay universally outperforms Plan2Explore.
Algorithms were allowed to run until convergence in the source MDP task, which required no more than 2 million environment steps in the RWRL experiments and no more than 5 million in the NovGrid experiments. 
After transfer to the target MDP task, all algorithm performance was assessed on 100k adaptation steps in RWRL and 200k steps in NovGrid.

\begin{figure}[ht]
  \centering
\includegraphics[width=0.85\linewidth]{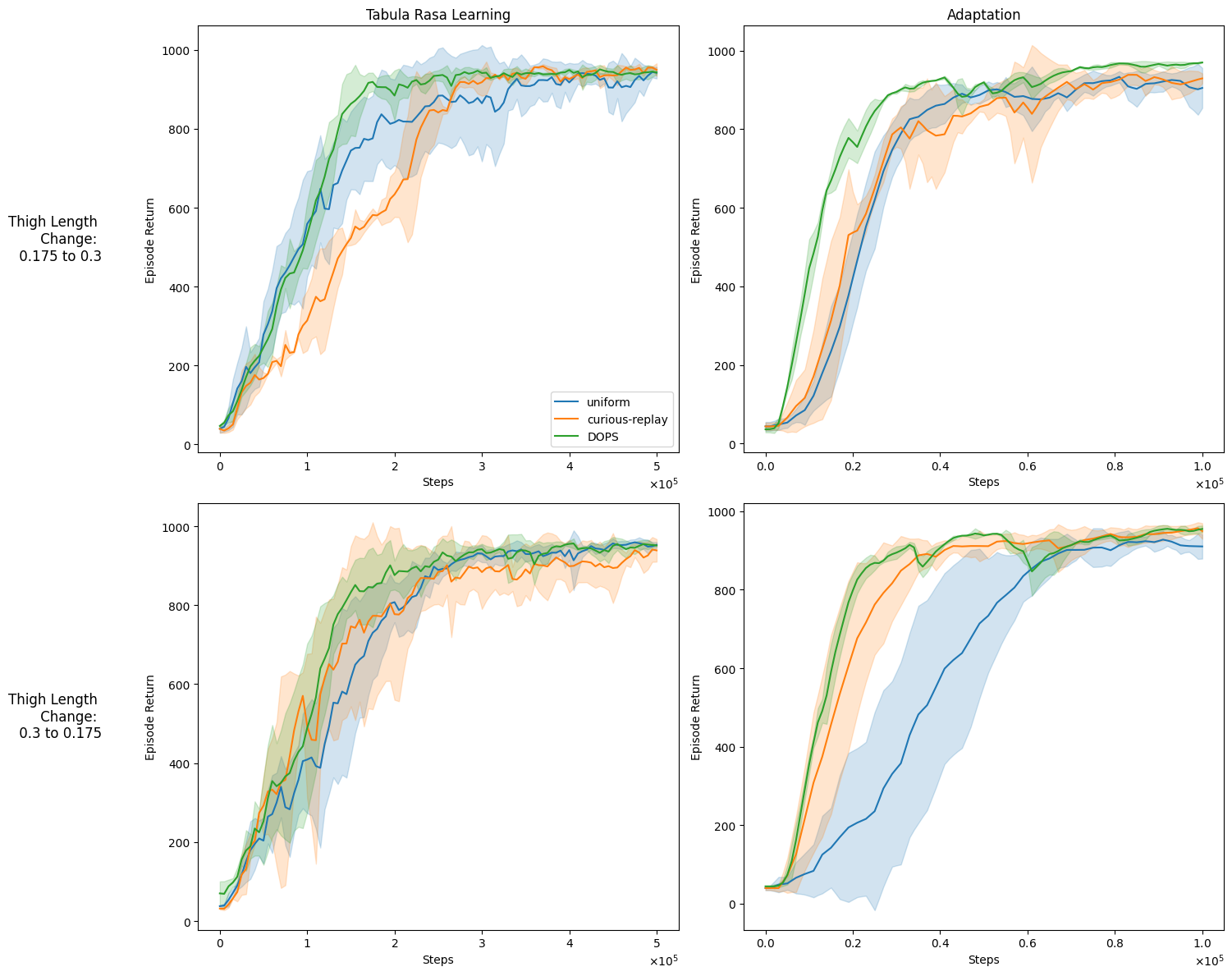}
  \caption{This graphic shows the learning curves of DOPS and the baselines solving Walker2d from the RWRL environment with the ThighLengthChange novelties.
  Each row is a different novelty scenario, and for each novelty the left plot represents the \textit{tabula rasa} learning while the right plot represents the adaptation process.   
  In the first row, the length of the Walker2d thigh link is 0.175 meters, and then adaptation of the agent's policy to a thigh length of 0.3 meters. The second row shows the reverse: learning an optimal policy for a thigh length of 0.3 meters and then adapting to 0.175 meters. 
  From five trials with different random seeds for each method the line plot represents the mean of the learning process smoothed with an EMA window of 5 steps, and the shaded region represents a 95\% bootstrapped confidence interval. }
  \label{fig:dops:full}
\end{figure}

For DOPS we sample 20\% of the world model data uniformly and 80\% with CR.
Based on our search over hyperparameter values for the masking overlap function $W$ we confirm the finding from Saglam et. al.~\cite{saglam2023la3p} that an overlap value of $W=0.5$ works well.
All other hyperparameters for RWRL were kept from the DMControl experiments in Curious Replay and DreamerV3, and the hyperparameters for NovGrid were taken from the default hyperparameters of DreamerV3 from Atari with the default replay hyperparameters from Curious Replay from their interaction assay.  
Further detailed hyperparameter information is available in the Appendix. 

\subsection{Results}

Figure~\ref{fig:dops:full} demonstrates DOPS's performance on both \textit{tabula rasa} learning and adaptation in the Walker2d environment from RWRL. 
The figure shows two scenarios: increasing the thigh length from 0.175 to 0.3, and decreasing it from 0.3 to 0.175. In both cases, we compare DOPS against uniform sampling (the Dreamer baseline) and Dreamer with Curious Replay sampling. 

In \textit{tabula rasa} learning (left figures), DOPS demonstrates notably improved sample efficiency compared to both baselines, particularly in early learning. While Curious Replay eventually achieves similar final performance, it requires approximately 2x more environment steps to reach equivalent reward levels. 
This efficiency gain can be attributed to DOPS's dual-objective sampling strategy---by explicitly separating world model and policy learning objectives, the agent can more effectively leverage both prediction error signals and behavioral learning signals during the initial learning phase.

The adaptation results (right column) reveal even more striking differences. 
When adapting to both larger and smaller thigh lengths, DOPS maintains the strong performance characteristics of Curious Replay while demonstrating improved stability, as evidenced by the tighter confidence intervals. 
This suggests that DOPS's approach of blending uniform samples with prioritized sampling helps prevent the catastrophic forgetting that can occur with pure priority-based methods. 
Particularly notable is the case of adapting to the shorter thigh length (bottom right), where DOPS achieves roughly 1.5x faster adaptation than uniform sampling while matching Curious Replay's efficiency. 
This scenario represents a more challenging adaptation problem as it requires the agent to learn more precise control with less mechanical advantage.

\begin{toexclude}

\begin{figure}[ht]
  \centering
  \begin{subfigure}[b]{0.5\textwidth}
    \includegraphics[width=\linewidth]{example-image-a}
    \caption{a} 
    \label{fig:dops:rwrl:a}
  \end{subfigure}%
  ~
  \begin{subfigure}[b]{0.5\textwidth}
    \includegraphics[width=\linewidth]{example-image-a}
    \caption{b} 
    \label{fig:dops:rwrl:b}
  \end{subfigure}
  \caption{Adaptation curves for DOPS and the baselines for the OTTA problem settings ThighLengthChange and TorsoDensityChange from the RWRL environment. 
  The shaded region represents a 95\% bootstrapped confidence interval over five trials.  }
  \label{fig:dops:rwrl}
\end{figure}

We also plot the adaptation curves for the results on the RWRL OTTA problem settings in Figure~\ref{fig:dops:rwrl}. 
From this figure, we can see that in adaptation as well DOPS adapts faster than Curious Replay and Dreamer. 

\begin{figure}[ht]
  \centering
  \begin{subfigure}[b]{0.5\textwidth}
    \includegraphics[width=\linewidth]{example-image-a}
    \caption{a} 
    \label{fig:dops:novgrid:a}
  \end{subfigure}%
  ~
  \begin{subfigure}[b]{0.5\textwidth}
    \includegraphics[width=\linewidth]{example-image-a}
    \caption{b} 
    \label{fig:dops:novgrid:b}
  \end{subfigure}
  ~
  \begin{subfigure}[b]{0.5\textwidth}
    \includegraphics[width=\linewidth]{example-image-a}
    \caption{c} 
    \label{fig:dops:novgrid:c}
  \end{subfigure}
  ~
  \begin{subfigure}[b]{0.5\textwidth}
    \includegraphics[width=\linewidth]{example-image-a}
    \caption{d} 
    \label{fig:dops:novgrid:d}
  \end{subfigure}
  \caption{Adaptation curves for DOPS and the baselines for the OTTA problem settings DoorKeyChange and CrossingBarrierChange from the NovGrid environment. 
  As in the RWRL experiments, the shaded region represents a 95\% bootstrapped confidence interval over five trials.  }
  \label{fig:dops:novgrid}
\end{figure}

Across all tasks, of DOPS converges more efficiently than Dreamer and Curious Replay in both \textit{tabula rasa} learning and  adaptation.
The results documented in Table~\ref{tab:dops:all} demonstrate that this trend is consistent across all environments, albeit more pronounced in some than others.
\Mark{I'm a bit lost by table 5.1. Transfer area under the curve between DOPS and non-DOPS sampling on top of the exploration algorithms? What happened to Dreamer and Curious Replay?}
At the algorithm level, we can take the following information\Mark{there is no following information}.

\begin{table}[ht]
    \centering
    \footnotesize
    \begin{tabular}{|c|c|c|c|c|c|}
        \hline
         & \multicolumn{5}{|c|}{Transfer Area Under Curve $\uparrow$} \\ \cline{2-6} 
        Exploration & DoorKeyChange & LavaNotSafe & LavaProof & CrossingBarrier & ThighIncrease \\ 
        Algorithm & ($10^{-1}$) & ($10^{-1}$) & ($10^{-1}$) & ($10^{-1}$) & ($10^{2}$) \\ \hline 
        \hline
        None (PPO) & 7.72 $\pm$ 0.792 & 7.43 $\pm$ 1.29 & 9.66 $\pm$ 0.0835 & 8.89 $\pm$ 0.297 & 6.5 $\pm$ 1.63 \\
        \hline
        NoisyNets & \textbf{8.13 $\pm$ 1.23} & \textbf{8.37 $\pm$ 0.885} & 7.69 $\pm$ 3.36 & 8.94 $\pm$ 0.388 & 8.62 $\pm$ 0.39 \\
        ICM & 7.28 $\pm$ 1.07 & 5.43 $\pm$ 0.667 & 9.22 $\pm$ 1.16 & 8.74 $\pm$ 0.537 & 7.25 $\pm$ 1.56 \\
        DIAYN & 7.54 $\pm$ 0.624 & 6.25 $\pm$ 1.22 & \textbf{9.7 $\pm$ 0.0773} & 9.01 $\pm$ 0.493 & \textbf{8.72 $\pm$ 0.203} \\
        RND & 8.09 $\pm$ 0.542 & 6.25 $\pm$ 1.53 & 9.37 $\pm$ 0.66 & 9.0 $\pm$ 0.399 & 7.47 $\pm$ 1.73 \\
        NGU & 7.56 $\pm$ 0.508 & 6.86 $\pm$ 1.38 & 9.48 $\pm$ 0.38 & 9.09 $\pm$ 0.444 & 7.07 $\pm$ 1.68 \\
        RIDE & 7.67 $\pm$ 0.727 & 7.63 $\pm$ 0.895 & 9.5 $\pm$ 0.605 & 9.02 $\pm$ 0.373 & 5.76 $\pm$ 1.67 \\
        GIRL & 7.59 $\pm$ 0.855 & 6.01 $\pm$ 1.08 & 9.55 $\pm$ 0.295 & 8.86 $\pm$ 0.51 & 7.45 $\pm$ 1.74 \\
        RE3 & 8.12 $\pm$ 0.387 & 6.82 $\pm$ 1.48 & 9.37 $\pm$ 0.524 & \textbf{9.1 $\pm$ 0.266} & 6.77 $\pm$ 1.99 \\
        RISE & 7.35 $\pm$ 1.08 & 7.07 $\pm$ 1.77 & 9.42 $\pm$ 0.402 & 9.09 $\pm$ 0.343 & 7.05 $\pm$ 1.03 \\
        REVD & 7.99 $\pm$ 0.402 & 7.3 $\pm$ 1.68 & 9.69 $\pm$ 0.056 & 8.92 $\pm$ 0.384 & 5.58 $\pm$ 1.33 \\
        \hline
    \end{tabular}
    \caption{Table of learning and adaptation results}
    \label{tab:dops:all}
\end{table}

\end{toexclude}

\section{Key Takeaways}

In this work, we develop and test DOPS, our sampling algorithm designed to improve the adaptive efficiency of RL agents by considering both the interactions between \textit{tabula rasa} training and adaptation and interactions between the different parts of the Dreamer architecture and learning algorithm.
Through our theoretical analysis we extend the objective-mismatch hypothesis to a consideration of distinctions between models with different learning signals, how models of each learning adversely affected by OTTA problems, and how prioritized sampling can compensate.
Through our tests on 
RWRL compared to Curious Replay and Dreamer, we demonstrate that DOPS improves the sample efficiency of \textit{tabula rasa} learning and adaptation.

The empirical results support our core thesis that efficient online test-time adaptation requires careful management of exploration and sampling. 
DOPS achieves this through two key mechanisms. 
First, by separating the sampling objectives for world model and behavior learning, DOPS enables more targeted exploration that validates past assumptions while building task-agnostic representations. 
This is evidenced by the improved sample efficiency in \textit{tabula rasa} learning, where DOPS consistently outperforms both uniform sampling and Curious Replay baselines. 
Second, the blended sampling approach, which combines uniform samples with prioritized transitions, helps regulate which parts of the model are updated during adaptation. 
This selective updating process is particularly apparent in the adaptation curves, where DOPS demonstrates faster recovery while maintaining narrower confidence intervals than competing methods.
Just as Chapter~\ref{chapt:transx} demonstrates the importance of exploration characteristics in adaptation of model-free RL, this work shows that sampling strategies in model-based RL must be tailored to the distinct learning objectives they serve. 

Our findings have important implications for model-based reinforcement learning. 
The success of DOPS suggests that the conventional approach of using identical sampling distributions for world model and policy learning may be fundamentally limiting. 
Specifically, DOPS' strong learning performance indicates the importance of considering how different neural modules in a complex architecture benefit from different data, especially in adaptation.
Instead of conceptualizing end-to-end architectures as monolithic, researchers should consider how gradients from different objectives impact different parts of an architecture, and train with data that balances the needs of the overall architecture with specific model parts.
In the same way that we designed DOPS by first examining the distinctions between the learning, designers of all neural architectures with multiple objectives or ``heads''---not just deep RL---should consider how to properly handle parts of an architecture that are differently affected by these objectives. 


The results also highlight some limitations and areas for future work. 
While DOPS consistently improves adaptation efficiency, the gains are more pronounced in some scenarios than others. 
Further investigation is needed to understand how the relationship between world model and policy learning objectives varies across different types of environmental changes. 
Additionally, while our implementation focuses on the Dreamer architecture, the principles underlying DOPS could potentially be extended to other model-based RL frameworks such as the TDMPC family of algorithms~\cite{hansen2022temporal}. 
A comparison of DOPS applied to a more comprehensive group of model-based RL techniques will help us to understand how objective mismatch affects each technique differently.

This chapter adds to the broader thesis by demonstrating that efficient online test-time adaptation requires more consideration than just prioritizing new data. 
Learning phenomena such as the distribution shift from objective mismatch exemplify the reason why prioritized exploration and sampling, while important to efficient adaptation, demands that prioritization methods are designed in the context of the learning process. 
DOPS provides a concrete mechanism for achieving this balance in model-based RL, complementing the exploration insights from Chapter~\ref{chapt:transx} and setting the stage for the investigation of structured knowledge representations in subsequent chapters.

%% file: chapters/6_worldcloner_chapt.tex
\chapter{Neuro-Symbolic Model-based Reinforcement Learning for Efficient Adaptation}
\label{chapt:knowledge}

Looking beyond the need to explore in novel scenarios, a critical aspect of adaptation to novelty is reusing what we already know about the world. 
Specifically, by separating knowledge that is impacted by a novelty from knowledge that is not affected, agents can update model components that have changed without needing to update all components of a learned model. 
For example, when people adapt to new technologies such as smart phones that change the way we communicate, it is important that this adaptation does not also impact physical skills like walking or carrying items. 
Moreover, it is beneficial when we have skills, like text messaging and sending emails; having prior similar knowledge  makes the adaptation to the novel scenario earier.

World-model based reinforcement learning offers possible reuse between the model and the behavior policy, but existing state-of-the-art approaches such as Dreamer cannot always update rapidly in the face of sudden change. 
To address this limitation, we developed WorldCloner, an efficient \textit{world model} reinforcement learning system with a neural policy consisting of two online test time adaptation improvements to the standard deep RL execution loop: 
(1)~A  symbolic world model for learning a model of the transition function---how features of the environment change and can be changed over time---that can be updated with a single post-novelty observation, allowing faster adaptation than neural world models. 
(2)~An {\em imagination-based adaptation} method that improves the efficiency of deployment-time policy adaptation using the updated world model to simulate environment transitions in the post-novelty world. 
By employing a symbolic world model parameterized by bounded intervals in feature space, the world model can adapt to novelty with a single example.  
Augmenting the policy adaptation process with synthetic data from a world model that adapts faster than neural models reduces the number of real environment interactions required to update the policy.


We evaluated the sample efficiency of WorldCloner in the NovGrid environment~\cite{balloch2022novgrid} with multiple novelty types. 
We show that post-novelty adaptation with WorldCloner requires fewer policy updates and environment interactions
than model-free and neural world model reinforcement learning techniques. 
To summarize, our contributions are as follows:
\begin{itemize}
    \item We present WorldCloner, a neuro-symbolic world model for novelty detection and adaptation. 
    \item We define a new symbolic representation with an efficient learning algorithm and a way to use this representation to help world models adapt to novelty. 
    \item We show that WorldCloner adapts to novelties more efficiently than state-of-the-art reinforcement learners. 
\end{itemize}

\begin{figure}
    \centering
    \includegraphics[width=1\linewidth]{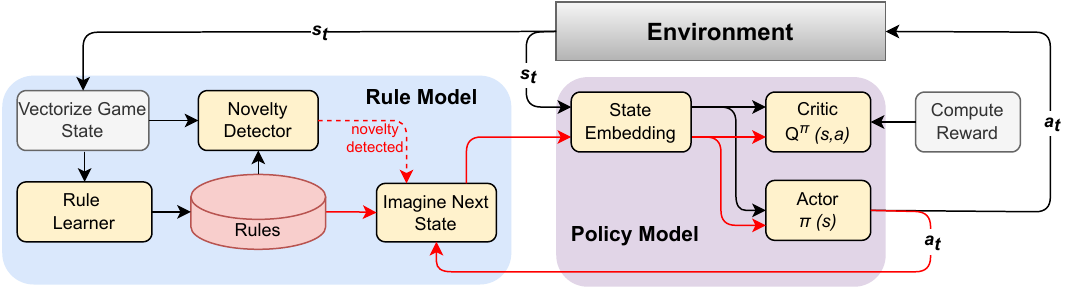}
    \caption{The WorldCloner architecture. 
    The purple module 
    and black arcs represent the conventional RL execution loop with loss back-propagating backward through black arcs in the purple module. 
    The blue module 
    contains rule model learning and novelty detection. 
    The red arcs represent information flow in a post-novelty environment, using learned rules to simulate the new environment. 
    Post-novelty, loss is back-propagated backward along the red arcs and black arcs within the policy model.}
    \label{fig:knowledge:architecture}
\end{figure}

\section{Approach}\label{sec:knowledge:approach}

WorldCloner is an end-to-end trainable neuro-symbolic world model comprised of two components: (1) a neural policy and (2) a symbolic world model. 
The symbolic world model (which we also refer to as the ``rule model'') consists of {\em rules} that, in aggregate, approximate the environment's latent transition function. 
The rule model serves two core functions. 
First, the rule model learns to predict state transitions pre-novelty. 
Rule violations  thus indicate the introduction of novelty and the need to update the rule model and the policy. 
Second, once in a post-novelty environment, WorldCloner uses the rule model to simulate the environment, enabling rollouts for retraining the neural policy model so as to require fewer interactions with the real environment. 
Shown in Figure~\ref{fig:knowledge:architecture}, this interaction between the world model and the policy allows WorldCloner to trust its policy pre-novelty, then depend more heavily on its world model post-novelty.  
Our world model algorithm is designed so that the rule model is independent of the neural policy implementation, making our approach compatible with any policy framework that uses the same data inputs as the rule model. 
For our implementation of WorldCloner, we use Proximal Policy Optimization (PPO)~\cite{schulman2017proximal} on an Advantage Actor-Critic (A2C) neural architecture. 

\subsection{Interval-Based Symbolic World Model}

In WorldCloner the symbolic world model--modeling the transition function $P$--is represented as a set of $K$ rules $\{\rho_k\}$ of the form $\langle c_s, c_a, e\rangle$. 
In this representation, $c_s$ is a state precondition, $c_a$ is the action precondition (similar to a do-calculus precondition \texttt{do(a)}), and $e$ is an effect. 
A rule $\rho$ is determined to apply if the input state $s$ and action $a$ match that rule's preconditions. 
The state preconditions contain a set of values corresponding to a subset of state features $\phi_1...\phi_m$. 
When both the state and action preconditions of a rule $\rho$
are satisfied, then $\rho$ is applicable and can be executed if chosen. 
Effects $e$ are the difference between the input state and the predicted state: $e=s^{'}-s$.  
This formulation has similarities to logical calculus frameworks such as ADL and PDDL~\cite{McDermott2000The1A} by encoding preconditions and effects. 
Our approach is designed to be learned, rather than engineered, similar to ``game rule'' learning~\cite{guzdial:ijcai2017}. 

WorldCloner uniquely formulates preconditions as a set of \textit{axis-aligned bounding intervals} (AABIs), also known as hyperrectangles or
$n$-orthotopes~\cite{coxeter1973regular} in feature space that cover the training data. 
AABIs are simple, $d$-dimensional convex geometries that, given a set of sample points to group $x_1...x_n$, define the minimum interval along each dimension that encloses the entire set. 
Regardless of the size of the interval, AABIs can be defined by two $d$-dimensional points---a minimum and maximum bound---which makes them very efficient to query for both training and inference. 
They can accommodate a mixture of continuous and categorical (non-continuous) variables, both of which are common in symbolic methods, where categorical AABI values are simply the exact set of matching values. 
For example, see the bottom of Figure~\ref{fig:knowledge:rule-creation}, which shows the AABIs for a rule for unlocking a door in the NovGrid grid world. In this case, the intervals limit the action to a single agent location $(3,5)$.
Figure~\ref{fig:knowledge:rule-relaxation} shows another example where the action is applicable to an interval of locations.

\begin{figure}
\centering
\scriptsize
    \begin{tikzpicture}[roundnode/.style={circle, draw=black!60, fill=gray!5, very thick, minimum size=7mm},
    squarenode/.style={rectangle, draw=black!60, fill=gray!20, very thick, minimum size=5mm}]
    \node[roundnode, align=center] (state1) {\underline{\bf State $T$}\\AgentLocation=(3,5)\\AgentFacing=East\\Inventory=\{YellowKey\}\\DoorState=Locked\\DoorLocation=(3,6)};
    \node[squarenode, align=center] (action) [right=of state1] {UnlockDoor};
    \node[roundnode, align=center] (state2) [right=of action]{\underline{\bf State $T+1$}\\AgentLocation=(3,5)\\AgentFacing=East\\\blueuwave{Inventory=None}\\\blueuwave{DoorState=Closed}\\DoorLocation=(3,6)};
    \node (txt) [below=of action,yshift=-1cm] {\textcolor{blue}{This state transition yields the creation of the following rule:}};
    \draw[->] (state1.east) -- (action.west);
    \draw[->] (action.east) -- (state2.west);
    \draw[->, dotted] (action.south) -- (txt.north);
    \end{tikzpicture}
    
    \colorbox{rulecolor}{
    \begin{tabular}{|l|l|}
    \hline
    \multicolumn{2}{|c|}{\bf Rule}\\
    \hline
    \underline{\bf Preconditions:} & \underline{\bf Effect:} \\
    AgentLocation: min=(3,5), max=(3,5) & DoorState: set=\{locked\} $\rightarrow$ DoorState: set=\{closed\}\\
    AgentFacing: set=\{East\} & Inventory: set=\{YellowKey\} $\rightarrow$ Inventory: set=None \\
    Inventory: set=\{YellowKey\} &\\
    DoorState: set=\{Locked\} & \\
    DoorLocation: min=(3,6), max=(3,6) &\\
    \cline{1-1}
    \underline{\bf Action Precondition:} &\\ 
    UnlockDoor & \\
    \hline
    \end{tabular}
    }
\caption{Top shows example environmental states passed to the rule learner (changes underlined). Bottom shows the learned world model rule describing the key opening a door. }
\label{fig:knowledge:rule-creation}
\end{figure}

The AABIs of the preconditions do not need to be intersecting; each unique rule can have multiple disjoint intervals. Our rule update algorithm (see next section) minimizes the number of different intervals. 
Multiple rules per action will exist when actions have different effects depending on the current state. 
For example, in a grid world, the \texttt{forward} action changes the agent's $\phi$
positional feature in the state along the direction the agent is facing when there is no obstruction, 
but \texttt{forward} will have no effect on the state if there is a wall directly in front of the agent. 
This same functionality enables us to account for probabilistic transitions; multiple rules will have the same action and state precondition but different \textit{effect distributions}. 
Using the example of opening a locked door with a key, there is the possibility that a lock is ``sticky'' and an agent may require several tries. The effect of the rule that predicts the opening of the lock would then be a distribution over the 
rules with identical preconditions but different effects. 


\begin{algorithm}[t]
\scriptsize
\SetAlgoLined
\KwIn{Prior State $s_{t-1}$, Action $a_{t-1}$, NextState $s$, WorldModel rule set $P$}
\KwOut{Applied Rule $\rho$}
StateChange $\delta s = s - s_{t-1}$\;
$RuleHit = \texttt{False}$\;
\For{Rule $\rho_k = \langle c_{s,k}, c_{a,k}, e_k\rangle$ in $P$}{
    \If{$a_{t-1} == c_{a,k}$}{
        \eIf{$\delta s == e_k$}{
            \eIf{$CollisionCheck(s_{t-1}, c_{s,k})$}{
                $RuleHit \leftarrow \texttt{True}$ \;
                return $\rho$ \;
            }
            {
                $RuleRelaxation(\rho, s_{t-1})$ \;
                $RuleHit \leftarrow \texttt{True}$ \;

            }
        }
        {
            \eIf{$CollisionCheck(s_{t-1}, c_{s,k})$}{
                $RuleCollisionResolution(\rho, s_{t-1})$ \;
            }
            {
                $RuleCreation(c_{s} = s_{t-1}, c_{a} = a_{t-1}, e = \delta s)$ \;
                $RuleHit \leftarrow \texttt{True}$ \;

            }
        }
    }
}
\If{$RuleHit == \texttt{False}$}{
    $RuleCreation(c_{s} = s_{t-1}, c_{a} = a_{t-1}, e = \delta s)$ \;
}
\caption{The rule model update algorithm exemplifies how rules can be inductively updated with a single change.}
\label{alg:knowledge:update}
\end{algorithm}

\subsection{Rule Learning}

The rule learning process 
constructs a compact, collision-free set of rules that provide maximum coverage of the state-action space while minimizing the complexity of the symbolic world model. 
Moreover, it is an online updating process; once a rule is learned, it can be updated without knowledge of past observations. 

The rule update process begins with the rule model initialized as an empty set. 
After an action is taken in the environment, the rule learner receives the prior state of the environment from which to derive a precondition, 
the action taken, and a new state of the environment from which to derive an effect. 
Comparing the prior state, action, and new state with the state preconditions, action preconditions, and effects (respectively) of any existing rules in the model, 
the update algorithm enters one of four cases: 
\begin{enumerate}
    \item {\em No Change}: The prior state falls inside the state precondition AABI of an existing rule with a matching action and effect. 
    \item \textit{Rule Creation}: There is no rule where the action precondition is satisfied or the state difference matches the effect. 
    A new ``point'' rule is created that exactly describes the prior state. 
    \item \textit{Rule Relaxation}: A rule exists where the action precondition is satisfied and state difference matches the effect, but the prior state is not covered by the existing rule's state precondition AABI. 
    The rule is ``relaxed'' by expanding the AABI. 
    \item \textit{Rule Collision Resolution}: A rule exists where the action precondition and state precondition AABI are satisfied but the effect is different. 
    The AABI of the existing rule is split with a minimum cut (min-cut) operation. 
\end{enumerate}

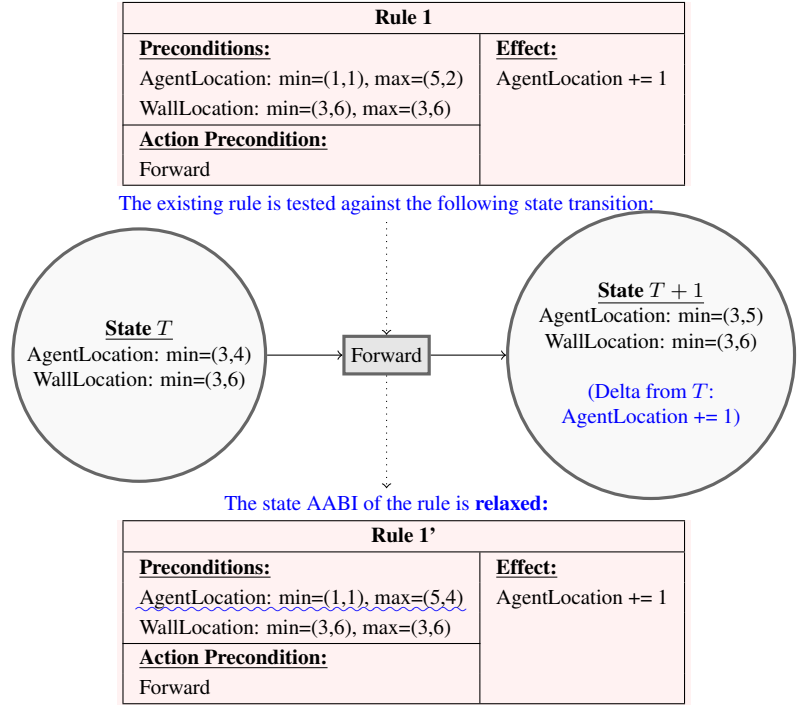
\begin{figure}
\centering
\scriptsize
    \colorbox{rulecolor}{
    \begin{tabular}{|l|l|}
    \hline
    \multicolumn{2}{|c|}{\bf Rule 1}\\
    \hline
    \underline{\bf Preconditions:} & \underline{\bf Effect:} \\
    AgentLocation: min=(1,1), max=(5,2) & AgentLocation += 1\\
    WallLocation: min=(3,6), max=(3,6) & \\
    \cline{1-1}
    \underline{\bf Action Precondition:} &\\ 
    Forward & \\
    \hline
    \end{tabular}
    }
    
    \begin{tikzpicture}[roundnode/.style={circle, draw=black!60, fill=gray!5, very thick, minimum size=7mm},
    squarenode/.style={rectangle, draw=black!60, fill=gray!20, very thick, minimum size=5mm}]
    \node[roundnode, align=center] (state1) {\underline{\bf State $T$}\\AgentLocation: min=(3,4)\\WallLocation: min=(3,6)};
    \node[squarenode, align=center] (action) [right=of state1] {Forward};
    \node[roundnode, align=center] (state2) [right=of action]{\underline{\bf State $T+1$}\\AgentLocation: min=(3,5)\\WallLocation: min=(3,6)\\\\\textcolor{blue}{(Delta from $T$:}\\\textcolor{blue}{AgentLocation += 1)}};
    \node (txt1) [above=of action,yshift=0.5cm] {\textcolor{blue}{The existing rule is tested against the following state transition:}};
    \node (txt2) [below=of action,yshift=-0.5cm] {\textcolor{blue}{The state AABI of the rule is \textbf{relaxed:}}};
    \draw[->] (state1.east) -- (action.west);
    \draw[->] (action.east) -- (state2.west);
    \draw[->, dotted] (txt1.south) -- (action.north);
    \draw[->, dotted] (action.south) -- (txt2.north);
    \end{tikzpicture}
    
    \colorbox{rulecolor}{
    \begin{tabular}{|l|l|}
    \hline
    \multicolumn{2}{|c|}{\bf Rule 1'}\\
    \hline
    \underline{\bf Preconditions:} & \underline{\bf Effect:} \\
    \blueuwave{AgentLocation: min=(1,1), max=(5,4)} & AgentLocation += 1\\
    WallLocation: min=(3,6), max=(3,6) & \\
    \cline{1-1}
    \underline{\bf Action Precondition:} &\\ 
    Forward & \\
    \hline
    \end{tabular}
    }
\caption{Rule Relaxation example, where the blue underlined precondition AABI corresponding to the agent location has been expanded in the modified Rule 1' 
to include agent location from state S.}
\label{fig:knowledge:rule-relaxation}
\end{figure}

When the agent takes an action, the rule learner observes the prior state at time $T$, the action executed, and the next state at time $T+1$. 
If rules exist where the state transition satisfies the action precondition, the state precondition, or the effect, we first try to modify existing rules using Rule Relaxation and Rule Collision Resolution. 
Rule Creation is only necessary if collision resolution and relaxation do not apply. 

Given the agent's performed action, we initially only consider rules where the action precondition is satisfied. 
We then identify whether the state prior to the action is contained in any of these existing rules' AABI. 
This is achieved using the geometric hyperplane separation theorem (also called the separating axis theorem)~\cite{hastie2009elements, ball1997elementary}. 
Geometrically, we can assert that for all features $\phi_d \in \Phi$, if there exists a feature for which the point is less than the min or more than the max of an interval, then a separating hyperplane exists. 
The hyperplane separation theorem states simply that if a hyperplane exists in feature space between the point and the geometric shape, there is no collision and the point is outside the AABI. 
Specifically, for prior state $s_{t-1}$ and an AABI $I = [I_{min},I_{max}]$:

\begin{equation}
    s_{t-1} \not\in I \iff \exists \phi~s.t.~[ s_{t-1,\phi} > I_{max,\phi}  \cup  s_{t-1,\phi} < I_{min,\phi} ]. 
\end{equation}
\label{eq:wc:hyperplane}

\subsubsection{Rule Creation}
Rule creation occurs when the combination of prior state and action in an observed state-action-state transition does not fall within the AABI of any existing rule. 
A new rule is added to the rule model where the AABI for continuous state features are assigned \texttt{min} and \texttt{max} values equal to their current value, and categorical state feature values are singleton members of their features. 
Similarly, the rule's action precondition is set to the action in the state transition, and the effect is equal to the difference between $s_{t-1}$ and $s$. 
An example of this process is illustrated in \ref{fig:knowledge:rule-creation}. 


\subsubsection{Rule Relaxation} 
Rule relaxation expands an AABI of an existing rule to cover a newly encountered action precondition and effect. 
For the prior state $s_{t-1}$ and AABI $I$ represented by points $I_{min}$ and $I_{max}$, the relaxed minima and maxima are
$I^*_{min} = \texttt{min}(s_{t-1},I_{min})$ and $I^*_{max} = \texttt{max}(s_{t-1},I_{max})$. 

Figure~\ref{fig:knowledge:rule-relaxation} illustrates an example of rule relaxation.
In Figure~\ref{fig:knowledge:rule-relaxation}, a previously learned rule models the change in \texttt{AgentPosition} caused by a \texttt{forward} action. 
This rule matches the action precondition and effect but not the state precondition, possibly because it simply had not been observed yet in that part of the environment. 
As a result, the \texttt{AgentPosition} AABI is expanded to include the state precondition associated with the observed transition.

After expanding the AABI, the new AABI $I^*$ is checked for collisions with the AABIs of other rules, again using the hyperplane separation theorem described in Equation~\ref{eq:wc:hyperplane}. 
For comparing intervals, however, we check for a hyperplane between $I*$ and the AABI $I^k$ of another rule $\rho_k$ by comparing the maxima to the minima of the intervals. 
Given maximum and minimum points $[I^*_{min},I^*_{max}]$ and $[I^k_{min},I^k_{max}]$, 
if there exists a feature $\phi$ for which $I^*_{min} > I^k_{max}$, or vice versa, then there is no collision. 
If a collision does exist, instead of trying to compromise between the two rules, we execute Rule Collision Resolution.

\subsubsection{Rule Collision Resolution} finds the min-cut partition of the existing rule's state precondition AABI. 
Consider the AABI as a graph, where the graph nodes are the bounding hyperplanes (``faces'' of the hyperrectangle) and the edges are lines that connect opposing hyperplanes weighted by their length. 
This min-cut of the AABI is simply the division along the largest feature axis that intersects with the prior state. 
The new, divided AABIs are added to the existing rule's preconditions, the original AABI is removed, and the prior state is assessed for accommodation with Rule Creation or Rule Relaxation because it may not be included in either split. 

Figure~\ref{fig:knowledge:rule-collision} illustrates a situation where the \texttt{forward} rule is observed to not correctly predict the outcome of a state transition \texttt{forward} because the agent hits a wall. 
Because the rule describing this state transition has a precondition subsumed by the existing \texttt{forward} rule but with a different effect, there is a rule collision. 
The collision is resolved by splitting the rule. 
A new point rule can be created because the prior state is now outside the AABIs of the split rules. 

Post-novelty, splitting rules in Rule Collision Resolution can result in one of the split rules having a precondition with a feature with empty interval or empty categorical set. 
If this occurs, the ``empty split'' is discarded. 


\subsection{Novelty Detection}

Once pre-novelty neural policies converge and the symbolic world model is created as described above, learning is turned off for both; this saves compute and allows the policy to focus on exploitation. 
As the agent performs tasks in the environment, it looks for state-action-state transitions that are inconsistent with the rules in the world model and triggers adaptation when one of two cases occur. 
\begin{enumerate}
    \item \textbf{A single rule is violated $n$ consecutive times}. 
    A violation is defined as the observed subsequent state of the environment not matching a rule's expected state effect even though the state and action preconditions both match. 
    Violations occurring consecutively is a heuristic for novelty injection because it indicates that a previously correct rule might be poorly modeling local behavior.
    \item \textbf{A single observed state causes a failed prediction in more than $n$ consecutive visits.}. 
    Consecutive visits to the same state that result in only failed predictions is a heuristic for novelty injection because it indicates that a state expected to be covered by the model is in fact not. 
\end{enumerate}

$n$ is a hyperparameter tuned based on the desired sensitivity of novelty detection. 
Based on testing multiple values for $n$, and in an effort to not miss any novelties, we found setting $n=2$ in our experiments was a good compromise between false positives detections and missed novelty detections.

Once a novelty is detected, 
the neural policy and the symbolic world model begin to update online again. 
The post-novelty rule update process is exactly the same as pre-novelty rule learning (Algorithm~\ref{alg:knowledge:update}). 
The rule model can thus be updated with as little as a single iteration of the rule learning algorithm, with guaranteed improved next-state prediction. 


\begin{algorithm}[t]
\scriptsize
\SetAlgoLined
\KwIn{Pre-Novelty Policy $\pi$, World Model $P$, Mix Ratio $\eta$}
$PostNov \leftarrow False$\;
$s_t \leftarrow$ initial observation\;
\While{$true$}
{
Select action $a_{t}=\pi(s_{t})$ \;

Predict next state $\hat{s}_{t+1} = P(s_{t}, a_{t})$ \;

Execute $a_{t}$ in $Env$ and observe 
next state and reward 
$s_{t+1}, r_{t}$\;

$PostNov \leftarrow [PostNov$ {\bf OR} $DetectNovelty(P,\hat{s}_{t+1}, s_{t+1})$] \;

\If{$PostNov$}
{
Add $\langle s_t, a_t, s_{t+1}, r_t\rangle$ to $UpdateBuffer$\;

$P \leftarrow RuleModelUpdate(s_t, a_t, s_{t+1})$\;

Add $ImagineRollouts(P)$ per $\eta$ to $UpdateBuffer$\;

Periodically update $\pi$ with $UpdateBuffer$ \; 
}
$s_t \leftarrow s_{t+1}$
}

\caption{Imagination-Based Adaptation}
\label{alg:knowledge:imagination-based_adaptation}
\end{algorithm}


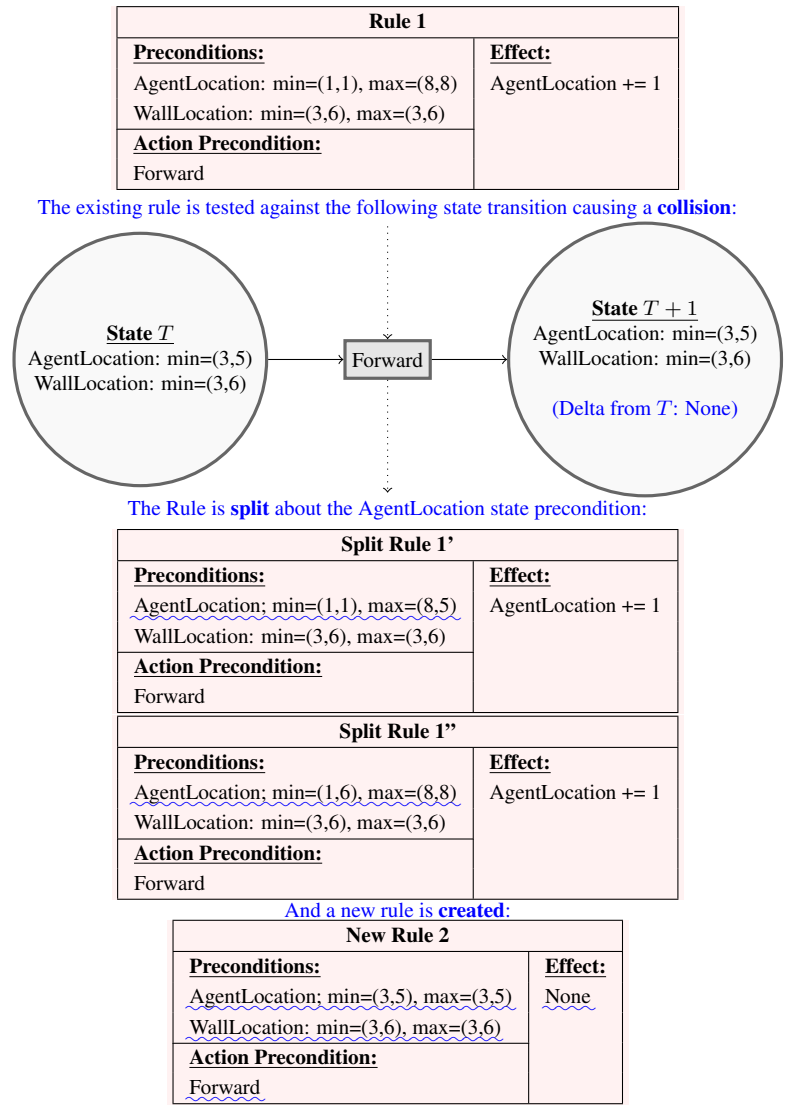
\begin{figure}
\centering
\scriptsize
    \colorbox{rulecolor}{
    \begin{tabular}{|l|l|}
    \hline
    \multicolumn{2}{|c|}{\bf Rule 1}\\
    \hline
    \underline{\bf Preconditions:} & \underline{\bf Effect:} \\
    AgentLocation: min=(1,1), max=(8,8) & AgentLocation += 1\\
    WallLocation: min=(3,6), max=(3,6) & \\
    \cline{1-1}
    \underline{\bf Action Precondition:} &\\ 
    Forward & \\
    \hline
    \end{tabular}
    }
    
    \begin{tikzpicture}[roundnode/.style={circle, draw=black!60, fill=gray!5, very thick, minimum size=7mm},
    squarenode/.style={rectangle, draw=black!60, fill=gray!20, very thick, minimum size=5mm}]
    \node[roundnode, align=center] (state1) {\underline{\bf State $T$}\\AgentLocation: min=(3,5)\\WallLocation: min=(3,6)};
    \node[squarenode, align=center] (action) [right=of state1] {Forward};
    \node[roundnode, align=center] (state2) [right=of action]{\underline{\bf State $T+1$}\\AgentLocation: min=(3,5)\\WallLocation: min=(3,6)\\\\\textcolor{blue}{(Delta from $T$: None)}};
    \node (txt1) [above=of action,yshift=0.5cm] {\textcolor{blue}{The existing rule is tested against the following state transition causing a \textbf{collision}:}};
    \node (txt2) [below=of action,yshift=-0.5cm] {\textcolor{blue}{The Rule is \textbf{split} about the AgentLocation state precondition:}};
    \draw[->] (state1.east) -- (action.west);
    \draw[->] (action.east) -- (state2.west);
    \draw[->, dotted] (txt1.south) -- (action.north);
    \draw[->, dotted] (action.south) -- (txt2.north);
    \end{tikzpicture}
    
    \colorbox{rulecolor}{
    \begin{tabular}{|l|l|}
    \hline
    \multicolumn{2}{|c|}{\bf Split Rule 1'}\\
    \hline
    \underline{\bf Preconditions:} & \underline{\bf Effect:} \\
    \blueuwave{AgentLocation; min=(1,1), max=(8,5)} & AgentLocation += 1\\
    WallLocation: min=(3,6), max=(3,6) & \\
    \cline{1-1}
    \underline{\bf Action Precondition:} &\\ 
    Forward & \\
    \hline
    \end{tabular}
    }
    
    \colorbox{rulecolor}{
    \begin{tabular}{|l|l|}
    \hline
    \multicolumn{2}{|c|}{\bf Split Rule 1''}\\
    \hline
    \underline{\bf Preconditions:} & \underline{\bf Effect:} \\
    \blueuwave{AgentLocation; min=(1,6), max=(8,8)} & AgentLocation += 1\\
    WallLocation: min=(3,6), max=(3,6) & \\
    \cline{1-1}
    \underline{\bf Action Precondition:} &\\ 
    Forward & \\
    \hline
    \end{tabular}
    }
    
    \textcolor{blue}{And a new rule is \textbf{created}:}
     
    \colorbox{rulecolor}{
    \begin{tabular}{|l|l|}
    \hline
    \multicolumn{2}{|c|}{\bf New Rule 2}\\
    \hline
    \underline{\bf Preconditions:} & \underline{\bf Effect:} \\
    \blueuwave{AgentLocation; min=(3,5), max=(3,5)} & \blueuwave{None}\\
    \blueuwave{WallLocation: min=(3,6), max=(3,6)} & \\
    \cline{1-1}
    \underline{\bf Action Precondition:} &\\ 
    \blueuwave{Forward} & \\
    \hline
    \end{tabular}
    }
\caption{\textit{Rule Collision}, and the resulting rule \textit{split} and \textit{creation}. The blue-underlined preconditions in the newly split Rule 1' and Rule 1'' indicate the feature dimension along which the original Rule 1 is split. The newly created Rule 2 accounts for the state transition that caused the collision with the original Rule 1.}
\label{fig:knowledge:rule-collision}
\end{figure}

\subsection{Imagination-Based Policy Adaptation}

Post-novelty, the newly updated rules within the symbolic world model reflect the agent's belief about the post-novelty state transition function. 
The agent now uses that rule model to ``imagine'' how sequences of actions will play out---which we refer to as ``imagination-based simulation''---and update its policy without interacting or executing actions in the true environment. 
Specifically, as we show in Algorithm~\ref{alg:knowledge:imagination-based_adaptation},
we use the rule model to simulate state-action-state transitions that then populate the agent's update buffer--the data on which the policy will be trained. 
The policy training algorithm generates a loss over samples drawn from the update buffer and back-propagates loss through the policy model. Figure~\ref{fig:knowledge:architecture} shows how the standard Actor-Critic neural architecture and imagination-based simulation work together to feed real and imagined state observations. 

The agent follows its policy in the imagined environment and repeatedly experiences the first rule change's consequences, receiving a reduced (or increased) expected reward, pushing the policy away from (or toward) the impacted actions. 
As the symbolic world model detects new discrepancies and represents the post-novelty environment more accurately, the policy may be able to ``imagine'' experiencing and accommodating novelty with minimal exposure to the novelty in the environment. 

Using the example of the re-keyed door lock (Figure~\ref{fig:novgrid:novgrid_splash}, top), the agent has executed its pre-novelty policy of navigating to the yellow key and then to the yellow door only to discover that the door no longer opens. 
Upon arriving at the state $\langle$\textit{AgentInFrontOfDoor,  AgentCarryingYellowKey, DoorLocked}$\rangle$ 
and performing the action \texttt{unlock}, the agent expects $DoorUnlocked$ to become true. 
However, since the door has been re-keyed the \texttt{unlock} action results in no change. 
When this occurs, there is a rule collision, resolved by creating a rule with a $\{\emptyset\}$ effect delta, and the old rule is ``split''. In this case, the split results in states with empty AABIs, in which case the rule is deleted. 
The rule collision is detected as a novelty, and the agent begins updating its policy. 
Specifically, the agent no longer receives the utility of walking through the door and the policy updates to reflect reduced utility of being directly in front of the door. 
Over time, the utility of the state will decrease enough that the agent will prefer other states, increasing exploration and eventually coming across the blue key. 
This switch from exploitation to exploration is accelerated by the agent's ability to repeatedly imagine arriving before the door and choosing alternative actions because there is no valid rule in which the door opens. 

\sloppypar{Eventually through exploration the agent will find itself before the door with the blue key: $\langle$\textit{AgentInFrontOfDoor, \blueuwave{AgentCarryingBlueKey}, DoorLocked}$\rangle$ 
Trying to open the door with the blue key will this time result in the effect of $DoorUnlocked$ changing to $DoorClosed$. 
This, in conjunction with the \texttt{unlock} action precondition does not match any existing rule, and a new rule is created, at which point imagination will facilitate faster policy learning. 
Once again having access to the utility and reward of states beyond the door, the agent will converge to a new policy involving the blue key. 
}

To account for world model error, our policy is also trained on real environment interactions. 
A mixing ratio parameter controls the ratio of imagined and real environment rollouts in the policy update buffer;
For every $t$ steps in the real environment, the agent runs $\frac{t}{\eta}$ steps in the imagined environment. 
This mixing effectively adds noise to the policy update and helps drive the policy back into ``explore'' mode, where actions will be selected more randomly by the policy. 

\begin{table}[t!]
\footnotesize
\centering
\begin{tabular}{c|c|c|c|c|}
{} & {\bf Adaptive Efficiency } & {\bf Pre-novelty} & {\bf Asymptotic } &  {\bf Update Efficiency}\\
& {\bf @0.95 (steps) $\downarrow$} & {\bf Performance} $\uparrow$ & {\bf Performance} $\uparrow$ &   {\bf (policy updates) $\downarrow$}\\
\hline
\multicolumn{5}{c}{\texttt{DoorKeyChange} novelty}\\
\hline PPO & 2.25E6 & 0.973 & 0.971 & 2.25E6 \\
\hline
 DreamerV2 & 5.3E5 & 0.971 & 0.973 &  3.82E8 \\
\hline
 \textbf{Ours} & 9.8E5 & 0.972 & 0.970 & \textbf{1.63E6} \\
\hline
\multicolumn{5}{c}{\texttt{LavaProof} novelty}\\
\hline
PPO & 1.39E5 & 0.972 & 0.991 & 1.39E5 \\
\hline
 DreamerV2 & Failed to adapt & 0.965 & Failed to adapt & Failed to adapt\\
\hline
 \textbf{Ours} & 8.3E4 & 0.972 & 0.991 & \textbf{1.38E5} \\
\hline
\multicolumn{5}{c}{\texttt{LavaHurts} novelty}\\
\hline
PPO & 2.08E6 & 0.992 & 0.971 & 2.08E6 \\
\hline
 DreamerV2 & 1.05E6 & 0.992 & 0.968 & 7.56E8 \\
\hline
 \textbf{Ours} & 1.07E6 & 0.992 & 0.972 & \textbf{1.78E6} \\
\hline
\end{tabular}
  \caption{Novelty metric results
  averaged over three runs. 
  DreamerV2 did not adapt to the novelty on \texttt{LavaProof}. 
}
\label{tab:knowledge:wc_results}
\end{table}

\section{Results}\label{sec:knowledge:results}

The experiments are performed in the NovGrid~\cite{balloch2022novgrid} environment, which extends the MiniGrid~\cite{gym_minigrid} environment with novelty injection and enables controlled, replicable experiments with stock novelties. 
We use two 8x8 Minigrid environments as the base environments: 

\begin{enumerate}
    \item \texttt{DoorKey} a standard environment where an agent must pick up a key, unlock a door, and navigate to the goal behind that door, and
    \item \texttt{LavaShortcutMaze}, a custom environment where an agent must navigate a maze that has a pool of lava lining the side of the maze nearest to the goal. 
\end{enumerate}
In all cases, we used the default sparse MiniGrid reward, which gives $1-t/(h*w*10)$ reward when agents navigate onto the terminal goal location, and no reward shaping. 
In the reward function, $t$ is the number of environment steps taken and $(h*w*10)$ is the default max number of steps for a Minigrid environment, where $h$ and $w$ are the height and width of the grid.

Performance of our method and the baselines was evaluated on three novelty types~\cite{balloch2022novgrid}:
\texttt{LavaProof} is a \textit{shortcut} novelty that makes lava in the \texttt{LavaShortcutMaze} environment to be harmless to the agent (where pre-novelty it destroyed the agent), offering a shorter path to the goal. 
\texttt{DoorKeyChange} is a \textit{delta} novelty that changes which of two keys unlock a door in the \texttt{DoorKey} environment, 
not changing the difficulty of reaching the goal but requiring different state-action sequences of similar length and complexity. 
\texttt{LavaHurts} is a \textit{barrier} novelty and the inverse of \texttt{LavaProof}, 
changing the effect of lava in the \texttt{LavaShortcutMaze} to destroy the agent (where pre-novelty, lava was harmless), thereby eliminating the shorter lava path to the goal. 
The \texttt{DoorKeyChange} and \texttt{LavaProof} novelties are conceptually shown in Figure~\ref{fig:novgrid:novgrid_splash}, and results are summarized in Table~\ref{tab:knowledge:wc_results} and Figures~\ref{fig:knowledge:env_steps_doorkey},~\ref{fig:knowledge:env_steps_lava_proof}, and~\ref{fig:knowledge:env_steps_lava_hurts}.    

\begin{figure}
\centering
    \includegraphics[width=0.99\linewidth]{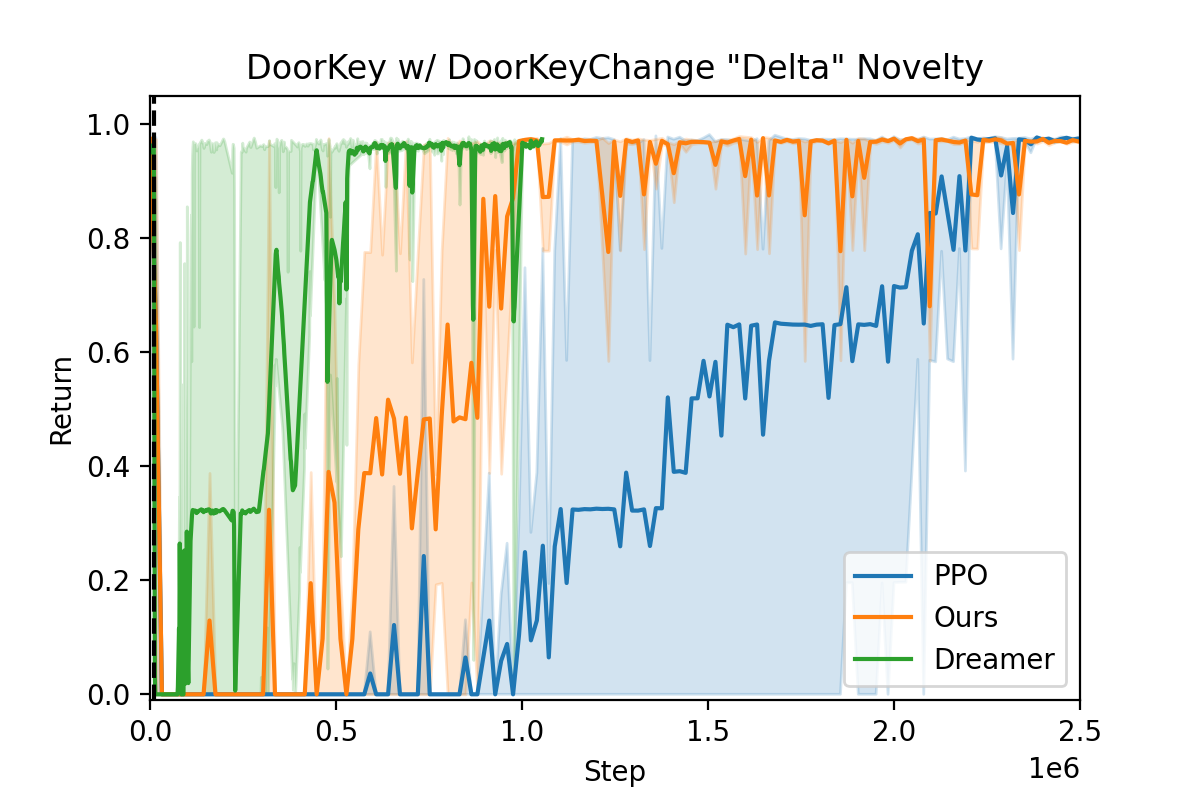} \\
    \caption{
    This plot shows the adaptive performance of agents post-novelty in the \texttt{DoorKeyChange} novelty. 
    The plot charts 10,000 pre-novelty environment steps followed by the number of environment steps required for agent convergence. Novelty injection is signified by the vertical dotted black line. 
    In the adaptation response to the \texttt{DoorKeyChange} ``delta'' novelty in the \texttt{DoorKey} environment, Dreamer adapted before WorldCloner, and both adapted before PPO. 
    }
    \label{fig:knowledge:env_steps_doorkey}
\end{figure}

\begin{figure}
\centering
    \includegraphics[width=0.99\linewidth]{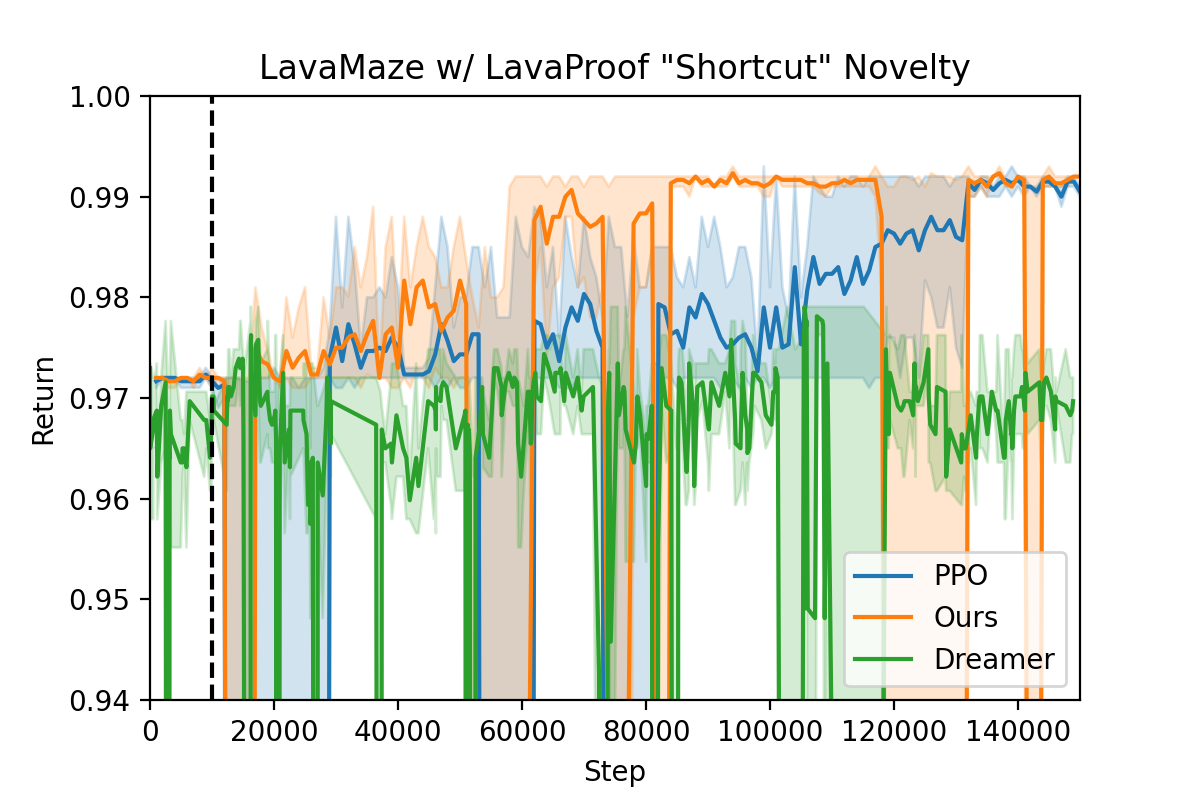} \\
    \caption{
    This plot shows the adaptive performance of agents post-novelty in the \texttt{LavaProof} novelty. 
    The plot charts 10,000 pre-novelty environment steps followed by the number of environment steps required for agent convergence. Novelty injection is signified by the vertical dotted black line. 
    In the adaptation response to the \texttt{LavaProof} ``shortcut'' novelty in the \texttt{LavaShortcutMaze} environment WorldCloner adapted faster than before PPO. 
    Interestingly, Dreamer never finds the shortcut. 
    }
    \label{fig:knowledge:env_steps_lava_proof}
\end{figure}

\begin{figure}
\centering
    \includegraphics[width=0.99\linewidth]{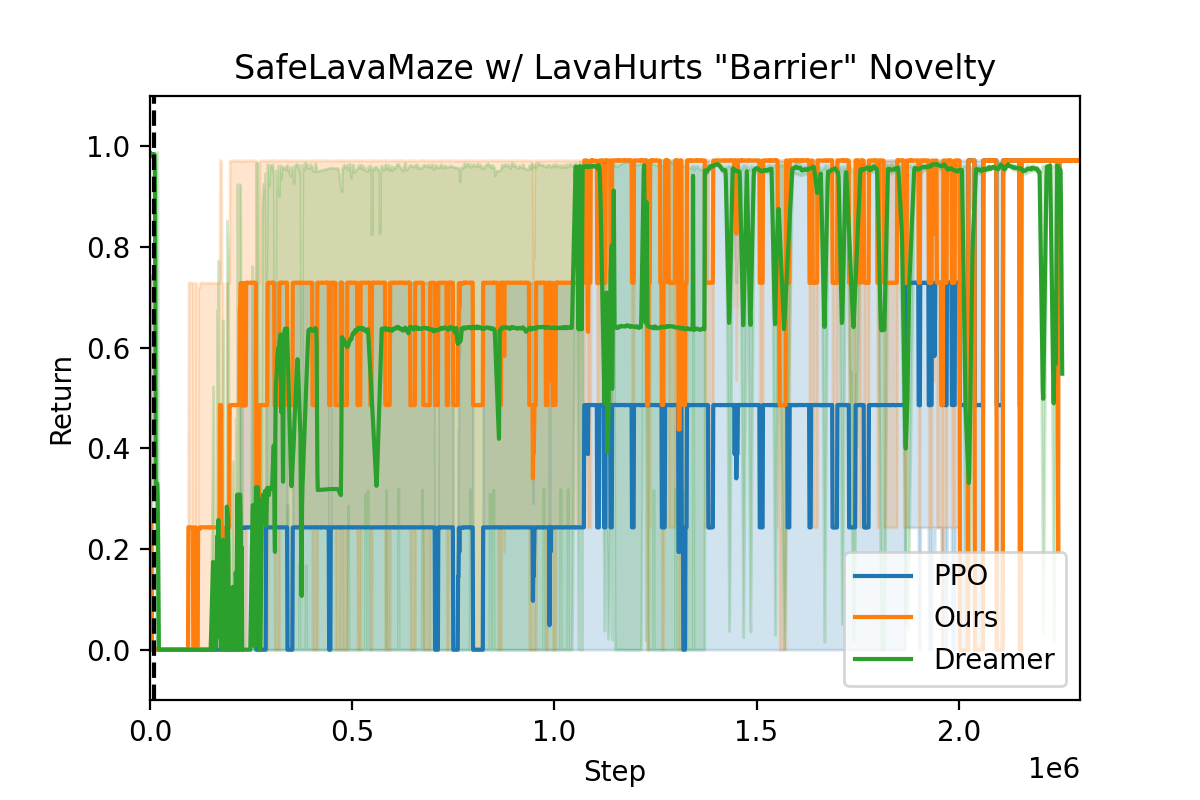} \\
    \caption{
    This plot shows the adaptive performance of agents post-novelty in the \texttt{LavaHurts} novelty. 
    The plot charts 10,000 pre-novelty environment steps followed by the number of environment steps required for agent convergence. Novelty injection is signified by the vertical dotted black line. 
    In the adaptation response to the \texttt{LavaHurts} ``barrier'' novelty in the \texttt{LavaShortcutMaze} environment, where lava only becomes harmful post-novelty, Dreamer and WorldCloner fully adapt at the same time, both faster than PPO which fails to reach maximum performance during adaptation. 
    }
    \label{fig:knowledge:env_steps_lava_hurts}
\end{figure}

To evaluate performance in these test environments, we adopt two metrics for novelty adaptation from \cite{balloch2022novgrid} that builds on \cite{alspector2021representation}. 
\begin{enumerate} 
\item {\em Asymptotic adaptive performance} is the final performance of the agent post-novelty relative to a random baseline, where higher is better. 
This is used to observe whether a method fully adapted to the post-novelty environment. 
\item {\em Adaptive efficiency} is the number of time steps in the real, post-novelty environment required to converge to asymptotic adaptive performance, where fewer steps is better. 
In our work, this is achieved when the 10-step moving average method reaches $95\%$ of asymptotic adaptive performance. 
We add a third measure, {\em update efficiency}, which is the number of policy updates required post-novelty to reach asymptotic adaptive performance, where--again--fewer steps is better. 
\end{enumerate}

We compare WorldCloner post-novelty performance with two baselines. 
The first is a standard reinforcement learning agent using Proximal Policy Optimization (PPO)~\cite{schulman2017proximal}, the same model-free reinforcement learning approach used by the neural policy in our WorldCloner method. 
The second baseline is DreamerV2~\cite{hafner2020dv2}, a state-of-the-art world modeling agent that learns an end-to-end neural world model. 
The agents were not given any knowledge about the novelties at training time. 
All methods were allowed to train for as many as 10 million time steps to ensure convergence, and the results for each method were averaged over three runs. 
Novelty was injected at episode 50k, well after all agents had converged. 
Since the baseline agents lack novelty detection capabilities, we keep their learning on during evaluations so that agents can react immediately after novelty is injected.

The policy architectures for all agents use a convolutional neural net feature extractor and two fully connected output networks, one to estimate the value and one to serve as the policy functions of the agent. 
All hyperparameters and architectures for DreamerV2 were consistent with the original publication~\cite{hafner2020dv2}, 
and all PPO hyperparameters are consistent with MiniGrid-suggested hyperparameters~\cite{willems_2020}. 
For WorldCloner we use a mixing ratio of 60\% real rollouts to 40\% imagined rollouts. 
This ratio was determined empirically by looking at the trade-off between asymptotic adaptive performance and adaptive efficiency. 
At higher amounts of imagination, the agent did not recover full post-novelty performance. 
The WorldCloner world model uses symbolic features from MiniGrid for rule learning, including the object type, color, and position of the agent and objects, the agent orientation, the agent's inventory, and whether doors are locked, unlocked, or open.

We document the results of these evaluations in Table~\ref{tab:knowledge:wc_results}, with the adaptation process of our method further illustrated in Figures~\ref{fig:knowledge:env_steps_doorkey},~\ref{fig:knowledge:env_steps_lava_proof}, and~\ref{fig:knowledge:env_steps_lava_hurts}.  
The table shows that pre-novelty, as expected, all three methods converge in all three novelty scenarios to effectively the same performance. 
This means that all methods were able to find solution sequences of equal length to the goal for all environments. 

For the \texttt{DoorKeyChange} ``delta'' novelty, DreamerV2 slightly outperforms WorldCloner in adaptive efficiency, while both dramatically outperform PPO. 
From this result we can observe that imagination is strongly beneficial to post-novelty adaptation. 
However, adaptive efficiency doesn't tell the entire story. 
DreamerV2 updates its policy on imagination only, which means that for each update to the world model, DreamerV2 must update its policy using many imagination iterations. 
As a result, update efficiency of DreamerV2 demands nearly two orders of magnitude more policy updates than WorldCloner. 

For the \texttt{LavaProof} ``shortcut'' novelty, WorldCloner substantially outperforms PPO in adaptive efficiency while DreamerV2 never finds the shortcut novelty. 
Dreamer was trained for 2.5 million post-novelty environment steps in multiple runs with the same result. 
We attribute DreamerV2's failure to the unique 
way in which its policy learner
depends on the accuracy of its world model. 
Unfortunately, the world model continues to predict negative consequence of the lava. 
Since the longer, pre-novelty solution still reaches the goal, Dreamer's world model remains fixed, imagination never varies, and the policy never updates. 
As a result, the world model can overfit easily if the policy produces the same sequence of actions every time. 
PPO and WorldCloner do not encounter this issue because small variations in the policy occur that cause policy updates, which allows for more sensitivity to the shortcut novelty. 
However, these methods do not completely resolve the adaptation problem.
As can be seen in  Figure~\ref{fig:knowledge:env_steps_lava_proof}, both methods take more than 25,000 environment steps to react to the shortcut novelty injection, unlike in the barrier and delta novelties where adaptation takes longer, . 
This demonstrates a unique challenge in shortcut novelties: when the the pre-novelty optimal path still exists, converged, non-exploring agents will struggle to detect novelty. 
As noted by \cite{balloch2022role}, future research is needed on novelty-aware exploration techniques. 

\begin{figure}
	\centering
\includegraphics[width=0.99\linewidth]{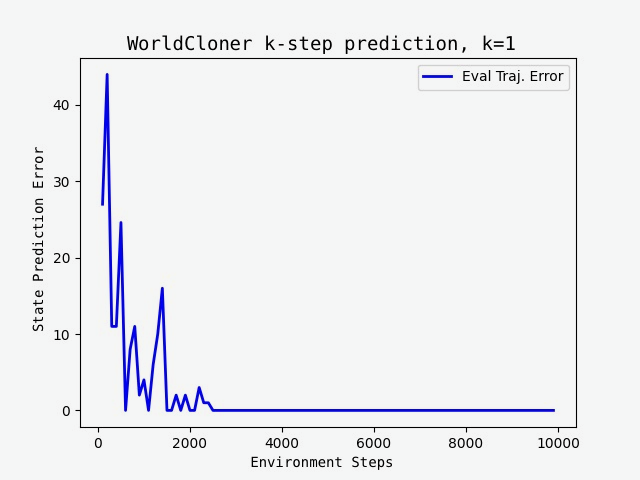}
\caption{This shows the WorldCloner 1-step prediction error vs environment steps during rule learner training in the Empty MiniGrid Environment. 
}
\label{fig:cloner_steps}
\end{figure}

Finally, the \texttt{LavaHurts} ``barrier'' novelty is similar to the 
\texttt{DoorKeyChange} novelty, where WorldCloner and DreamerV2 have very similar adaptive efficiency both outperforming PPO by a wide margin, but in all cases taking a long time to converge. 
What's more, in the \texttt{LavaHurts} and \texttt{DoorKeyChange} novelty, the high variance in both DreamerV2 and WorldCloner shows that in some of the trials both of these methods found the new solution very early in the adaption process, but could not consistently converge to it. 
Transfer learning tells us that both of these results are expected, as the simple pre-novelty solution does not prepare the policies for the more complex post-novelty environment. 

In general, our model is able to make much more efficient use of our updates than DreamerV2. 
Because DreamerV2 only trains its policy in imagination, it runs hundreds of imagination policy updates for every world model update it executes. 
As shown in the last column of Table~\ref{tab:knowledge:wc_results}, blending the environment interaction data with the imagination data, WorldCloner requires significantly fewer policy updates. 

We also evaluate the efficiency and accuracy of the rule learner (results shown in Figure~\ref{fig:cloner_steps}). 
After every 100 training steps we would validate the model accuracy by running a random policy for 1000 steps and measure the average 1-step predictive error of the world model. 
The pre-novelty rule learner requires only about 2000 training steps before it converges to near-perfect prediction accuracy. 
Grid worlds are simple, deterministic environments, so this is an unsurprising result. 
The key is that it converges rapidly and is sample-efficient compared to neural world model learning techniques~\cite{hafner2019dream}. 


These results show that WorldCloner with only 40\% imagination improves adaptation efficiency across all novelties over a neural policy with no world model. 
Furthermore, WorldCloner is competitive with DreamerV2 across these novelties and adapts with shortcuts that DreamerV2 fails to detect. 
Most interestingly, we can see from the last column of Table~\ref{tab:knowledge:wc_results} that WorldCloner achieves these results with fewer policy updates than PPO and DreamerV2. 
This, however, will vary with the amount of imagination injected into the policy learner by WorldCloner, and should be the subject of further study. 

\section{Key Takeaways}\label{sec:knowledge:conclusions}

As autonomous agents are deployed in open-world decision-making situations, techniques designed deliberately to handle novelty will be required. 
This can include novelties from learning to unlock a door with a new key to discovering a new shortcut to reduce travel time or avoiding new hazards on previous safe solutions. 

To this end, we showed that reinforcement learning agent policies can be adapted more efficiently to novelties using symbolic world models that (1)~can be updated rapidly and (2)~simulate rollouts that then can be added to the policy learning process, thereby reducing the number of direct interactions with the post-novelty environment. 
Specifically, our results show that WorldCloner is comparable in adaptation efficiency to state-of-the-art neural world modeling techniques while requiring only a fraction of policy updates. 
This suggests that, unlike neural world models, symbolic world models are good for distinguishing which knowledge must change and what can be preserved.
As a result, symbolic representations are an effective complement to neural representations for adapting reinforcement learning agents to novelty.

Although this work breaks new ground for architectures specific to online test-time adaptation, the limitations of this work leave room for improvement by future researchers. 
Critically, the WorldCloner AABI symbolic modeling approaches and rule-learning algorithms proposed in this chapter lack performance guarantees and are only applied to discrete deterministic environments with structured environment features. 
Either by replacing the AABI-based world model with a symbolic model learning approach with known covergence guarantees or by providing a deeper theoretical analysis of the AABI-base approach we can better understand the efficiency-performance trade-off in future neuro-symbolic world models. 
On the applied side, rule learning for complex continuous dynamics can be challenging without prior knowledge. 
For example, if modeling the dynamics of a complex system such as a parallel robot, simply learning rules predicting precondition-action-effect intervals will likely result in a very large number of rules, especially without any ability to use prior knowledge of the dynamics.
By extending the rule learner to utilize more complex induction and modeling techniques such as dynamic movement primitives, WorldCloner can improve the adaptive efficiency of agents in a wider variety of environments and OTTA scenarios.

The key insight the work in this Chapter provides is that architecture design of data-driven OTTA agents should be contingent on the specific problems with adaptation. 
In the specific case of WorldCloner, the decision to employ a symbolic world model only benefits the agent because the forgetting caused by distribution shift is so disruptive that even WorldCloner's imperfect world model was beneficial.
Highlighted by PPO's poor adaptation to the delta and barrier novelties and Dreamer's failure to adapt to the shortcut novelty, we see that naively expecting a neural network architecture to learn and transfer the appropriate prior knowledge in adaptation can negatively affect adaptive efficiency. 
When considering whether parts of a model architecture ought to be symbolic, neural, or something else, we need to think about how that particular architecture is suited to this specifics of the OTTA setting.

The insights provided by this chapter's results further supports the broader thesis of this dissertation.
Specifically, WorldCloner's improvments to adaptative efficiency demonstrate that efficient adaptation requires selectively preserving prior knowledge from $MDP_{\mathrm{source}}$ to be used in $MDP_{\mathrm{target}}$.
We built on the insights of this chapter in our final chapter, Chapter~\ref{chapt:cbwm}, where we examine how we can produce similar positive impacts on adaptive efficiency by preserving prior knowledge in neural-based models.
Instead of examining how black-box neural model adaptation efficiency can be complemented and improved using symbolic models, Chapter~\ref{chapt:cbwm} answers the question: "How can use ideas from symbolic modeling to ground and control the latent knowledge in neural models?"

%% file: chapters/7_cbwm_chapt.tex
\chapter{Concept Bottleneck World Models}
\label{chapt:cbwm}

Reinforcement learning (RL) policies that use world models show great promise in efficient task transfer. 
The ability to update a world model through a single exposure to environmental changes, followed by policy adaptation using the updated model without additional environmental interaction, offers significant potential for rapid learning and adaptation~\cite{balloch2023worldcloner,sarathy2021spotter}. 
However, the ``black box'' nature of neural network-based world models presents a challenge for understanding the relationship between knowledge and adaptation performance.

A common approach to working with interpretable ``white box''~\cite{chen2020concept} representations is to associate high-level ``concepts'' with parts of an agent's decision-making process~\cite{chen2020concept,das2023state2explanation}. 
``Concepts'' are symbolic predicates for grounding latent neural network features to human-parsable facts related to the task or domain in question~\cite{das2023state2explanation}.
Concepts are most frequently used in the context of classification systems; for example, a ``beak shape'' might be a concept used in bird classification. 
However, even if a conventionally learned neural representations contain the same knowledge that could be associated with a discrete concept, the knowledge within the neural representation cannot be reliably grounded to specific neurons or sequences of neurons (often called ``circuits'')~\cite{olah2020zoomcircuits}.
Even if one was able to identify a discrete concept in the weights of a particular neural network, that concept will manifest differently in every other neural network ~\cite{locatello2019challenging,carbonneau2022measuring,eaton2024icbinb}. 
Neural network interpretability literature refers to this phenomenon of concept inseparability as ``entanglement.''~\cite{locatello2019challenging,carbonneau2022measuring,nanda_2022, rauker2023toward}
Given the association of a neural circuit with a discrete concept, prior literature finds that the primary cause of entanglement is neural ``polysemanticity,''~\cite{nanda_2022}, meaning that multiple circuits associated with unrelated concepts share neurons.
Polysemanticity makes it impossible to separate concepts using neurons in the layers of a conventionally trained network. 


As a result of concept entanglement and polysementicity, each time a trained network is adapted to a new task, gradient updates modify all of these circuits and thereby all of the concepts, regardless of whether an entangled concept needed to change. 
The impact of this on online test-time adaptation (OTTA) is that while the knowledge embedded in the agent's model of $MDP_{\mathrm{source}}$ may directly apply to $MDP_{\mathrm{target}}$, any adaptation updates will modify those, weights thereby discarding the reusable knowledge.  
Updates to the world model resulting from novel, previously unseen scenarios will inadvertently impact concepts unrelated to the novelty. 

To illustrate this problem, consider an autonomous, holonomic drone trained to follow \textit{non-holonomic} vehicles such as cars on roads. 
If this drone is suddenly tasked with following the movements of \textit{holonomic} vehicles---such as other drones---that do not use roads the process of adapting its policy will likely unnecessarily erode its understanding of roads.
This unnecessary loss of knowledge is adaptive work that did not need to happen and makes future adaptation harder. 
If our drone agent is reassigned and adapted a second time, this time to follow drones that survey infrastructure, it must relearn valuable concepts about roads despite the fact that this knowledge did not need to be forgotten in the first place. 

Some world model methods like Dreamer attempt to avoid overwriting reusable knowledge using methods like Beta-VAE~\cite{higgins2017betavae} to encourage disentanglement through constraints on the loss function. 
However, prior work~\cite{locatello2019challenging} has shown that such methods, while helpful in some circumstances, rarely have an impact on a broad set of learning problems and data distributions. 

In this Chapter we introduce \textbf{Concept Bottleneck World Models (CBWMs)}, a model-based reinforcement learning approach where an internal layer of the world model is constrained to encode human-interpretable concepts related to the task. 
The CBWM architecture incorporates an internal layer constrained to encode human-interpretable concepts, forcing the pre-bottleneck weights to map from the input to the concepts.
These concepts are learned with additional loss terms during the model-learning process that predict concept values and force concept embeddings to be dissimilar. 
This concept layer ``bottlenecks'' the downstream reward, discount, and observation prediction tasks by replacing the majority of the latent state with this concept prediction information. 

By explicitly representing interpretable concepts within the world model, CBWM aims to mitigate the problem of concept entanglement and unnecessary knowledge loss during task transfer, forcing the model to be ``right for the right reasons''~\cite{ross2017right}. 
This approach not only preserves task-relevant information but also enhances the model's adaptability and efficiency in scenarios involving multiple or evolving tasks. 
Our work focuses on examining how concept bottlenecks can improve adaptation to novel scenarios never seen during training by preserving knowledge of concepts that do not need to change.

In the following sections, we will detail the CBWM architecture, present our methodology, and demonstrate its effectiveness through a series of experiments and comparative analyses.

\section{Preliminaries}
\label{sec:cbwm:prelims}

\subsection{Concept Bottleneck Models}
\label{sec:cbwm:prelims:cbm}

Concept Bottleneck Models (CBMs) are neural network architectures that incorporate an intermediate layer of human-interpretable concepts \cite{koh2020concept}. In a CBM, the network is structured such that information must pass through a "concept bottleneck" layer before reaching the output. Formally, given input $x$, concepts $c$, and output $y$, a CBM learns functions $f$ and $g$ such that:

\begin{equation}
    c = f(x), \quad y = g(c)
\end{equation}

By structuring the architecture of the model such that the output task is strictly constrained to being conditioned on the input task, $y=g(f(x))$, the model can be forced to make decisions on the output that are consistent with the concepts. 
The distinction between a concept bottleneck representation such as this and similar discretized representations like ``codebooks''~\cite{van2017vqvae,eaton2024icbinb} and slot attention~\cite{locatello2020slot} is while those methods rely on the structure alone to facilitate separability, the concepts in concept bottlenecks are supervised directly.
As a result, concept embedding model allows for more flexible concept representations while maintaining interpretability.

The implementation of the concept embedding model (CEM)~\cite{espinosa2022concept}, an improved implementation of a CBM, involves learning a mapping from input space to a continuous concept embeddings space, rather than discrete concept predictions.
The CEM contains $k$ concept \textit{networks} corresponding to predefined human-interpretable concepts, plus an additional network for encoding unknown or \textit{residual} concepts. 
For each concept, the CEM maintains two embedding vectors in representing active and inactive states of that concept. 
A context network maps the input to concept-specific embeddings through two functions that produce these active and inactive state embeddings. 
A probability network then predicts concept activation probability from the joint embedding space, producing values between 0 and 1 that can then be supervised with cross entropy loss. 
The final context embedding for each concept is computed as a weighted mixture of the active and inactive embeddings based on this predicted probability.
The key difference between the performance of the CEM and the original CBM is that by supervising predictions from those embeddings, but passing the embeddings instead of the predictions to the downstream task, the CEM still bottlenecks the downstream task while allowing a more rich representation.

In the context of RL and world models, integrating CBMs can provide a way to structure the learned representations around interpretable concepts, potentially improving adaptability of the models to novelty.
Additionally, model-based reinforcement learning is well suited to the addition of a concept bottleneck architecture. 
Model-free reinforcement learning focuses solely on policy learning, developing only enough understanding of the environment state to optimize the policy. 
However, intermediate interpretable policy features such as ``skills'' and state features strictly necessary for control are difficult to precisely define~\cite{StulpSigaud2013skill}. 
As model-based techniques learn a more complete approximation of the environment state, features such as objects, relative object position, and those related to environment modeling are easier to define and supervise, and these features can be useful to interpreting and grounding policies in addition to world models.

\begin{figure}[t]
    \centering
    \includegraphics[width=0.95\textwidth]{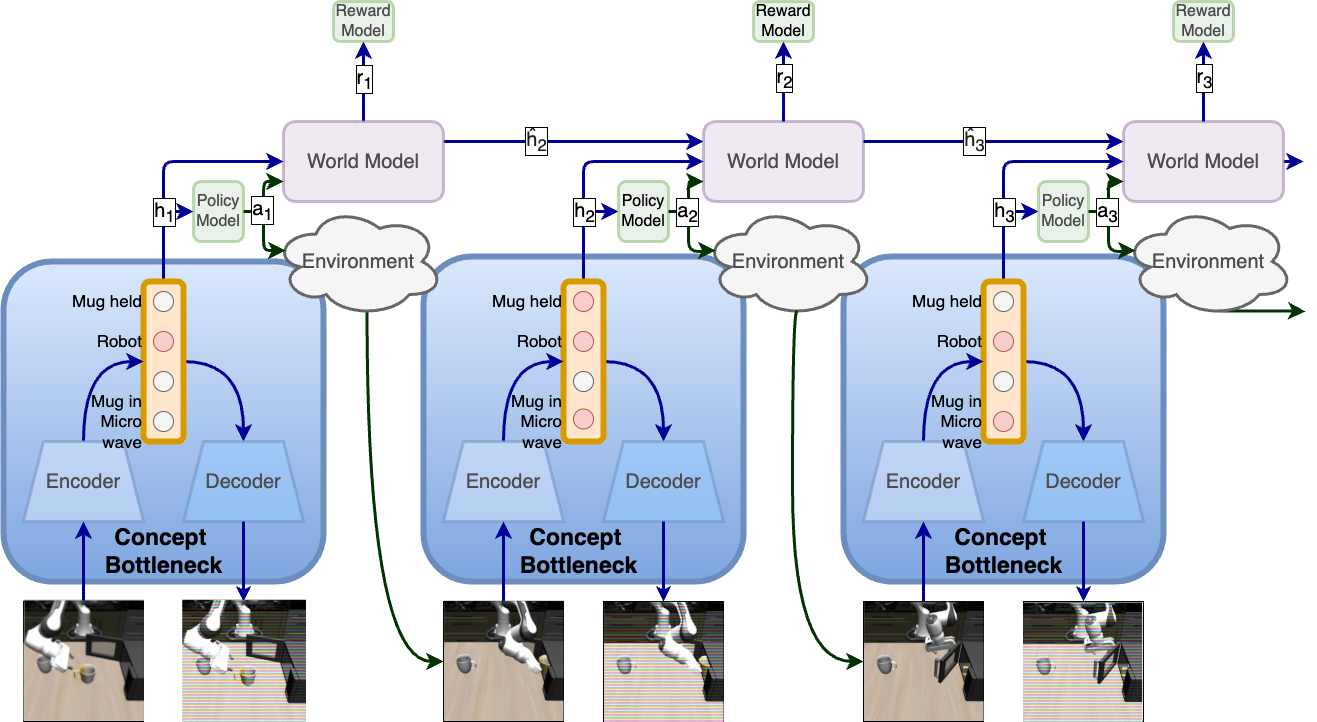}
    \caption{This shows our novel CBWM architecture as it interacts with the agent and environment over three time steps. 
    The bottleneck model, highlighted in blue, is unique as it uses both the stochastic and deterministic components of the world model latent as input. 
    The concept bottleneck itself is represented as the orange vector, where values in the bottleneck predict individual concept predicates, such as whether a robot is present or whether a mug is in the microwave. 
    We indicate the change in the state of the concepts for each new time step as red, meaning the concept is true, or grey, meaning it is false. 
    }
    \label{fig:cbwm:cbwm}
\end{figure}

\section{Concept Bottleneck World Models}
\label{sec:cbwm:cbwm}

We integrate the CBM architecture into a world model-based reinforcement learning framework. 
Our overall architecture, illustrated in Figure \ref{fig:cbwm:cbwm}, consists of three main components: (1) the pre-concept bottleneck network, which processes the input state; (2) the CBM layer itself; and (3) the post-concept bottleneck network, which predicts future states and rewards. 
While the pre- and post-bottleneck networks are specific to the particular world model architecture employed, the CBM layer remains consistent across different world model implementations. In the following subsections, we detail the CBM layer's architecture, associated loss functions, and a procedure for intervening in the world model's predictions.

\subsection{Model Architecture}
\label{sec:cbwm:cbwm:arch}

We adapt the Concept Embedding Model (CEM) layer proposed by  Zarlenga et al. \cite{espinosa2022concept} to the world model-based reinforcement learning context. 
However, unlike in supervised learning tasks where the concept set is often assumed to be near-complete \cite{koh2020concept, espinosa2022concept}, it is unrealistic to expect pre-defined human-understandable features---that is, concepts---to exhaustively capture all relevant aspects of the environment in a reinforcement learning setting. 
For instance, in a task involving manipulation, while we might have available concepts such as the objects that exist in the scene, there may be additional factors influencing those object concept representations, like whether they are visible or serve an addition function like being the target of the reward.

To address this challenge, we extend the CEM layer to incorporate both the stochastic and deterministic components of the world model latent as input. 
The modified world model, depicted operating in a sequence of three steps in Figure \ref{fig:cbwm:cbwm}, represents how the CBWM architecture interacts with sequential concept and environment data.
For a given time step CBWM takes an observation $o_t$ in from the environment and encodes it---shown in the blue block. 
This image encoding $z_t$ and the output of the recurrent model from the prior step $h_t$ form the stochastic and deterministic parts of the latent, respectively.
In the CBWM, we further encode this latent with a CEM layer, signified in orange in Figure~\ref{fig:cbwm:cbwm}, and described in Section~\ref{sec:cbwm:prelims:cbm}. 
The concept-latent is constructed from the concatenated embeddings from all concepts, including the concept residual, creating an $m(k+1)$-dimensional latent that is then passed to the prediction modules of the world model, shown in purple, and the policy model for decision making. 

This complete architecture is trained end-to-end using a composite loss function that combines The typical Dreamer world model loss with a binary cross-entropy loss on concept predictions and an orthogonality constraint between known and unknown concept embeddings.

\subsection{Training Objective}
\label{sec:cbwm:cbwm:objective}

We train Concept Bottleneck World Models (CBWMs) in an end-to-end manner by jointly minimizing the following loss function:

\begin{equation}
    \mathcal{L}_{total}(\phi) = \mathcal{L}_{task}(\phi) \textcolor{blue}{\bf +\beta_{con} \mathcal{L}_{con}(\phi) + \beta_{orth} \mathcal{L}_{orth}(\phi)}
\end{equation}

\begin{equation}
\begin{aligned}
\mathcal{L}_{task} = \mathbb{E}_{\phi}[ \sum_{t=1}^T \mathcal{L}_{obs}(\phi) +& \mathcal{L}_{rew}(\phi) + \mathcal{L}_{dis}(\phi) + \beta_{KL} \mathcal{L}_{KL}(\phi)
] 
\end{aligned}
\end{equation}

\noindent where the the elements in {\bf black}, $\mathcal{L}_{obs}(\phi)$, $\mathcal{L}_{rew}(\phi)$, and $\mathcal{L}_{dis}(\phi)$, are the prediction losses from conventional Dreamer world model learning, $ \mathcal{L}_{KL}(\phi) $ is the latent disagreement loss from conventional Dreamer world model learning. 
The elements in \textcolor{blue}{\bf blue} are losses are added by this work. 
Specifically, $\mathcal{L}_{con}$ is the concept loss, which is a negative log likelihood over one-hot concept predictions,

\begin{equation}
    \mathcal{L}_{\text {con}}=-\sum_{i \in \mathbf{c}} c_i \log{\hat{y}_i}
\end{equation}

and $\mathcal{L}_{orth}$ is the concept orthogonality constraint~\cite{ranasinghe2021orthogonal} that pushes concepts apart. 
\begin{equation}
    \mathcal{L}_{\text {orth }}=\sum_{j \in B} \frac{\sum_{i=1}^{i=k}\left|\left\langle w_i, w_{k+1}\right\rangle\right|}{\sum_{i=1}^{i=k} 1}
\end{equation}
The hyperparameters $\beta$ control the relative importance of the different loss terms, respectively.

\subsection{Offline-to-Online Training}



\begin{figure}[ht]
    \centering
    \includegraphics[width=0.8\textwidth]{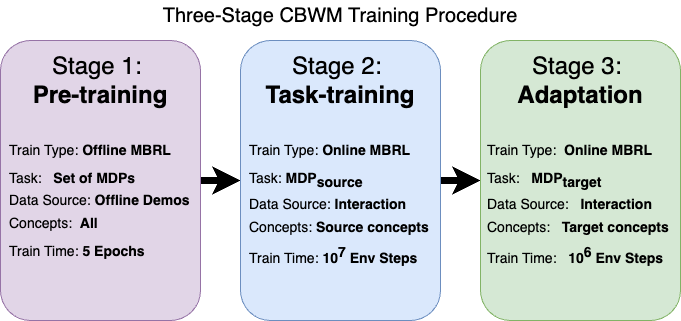}
    \caption{Figure illustrating the three-stage CBWM training procedure for balancing task specific policy learning and adaptation with task-agnostic dynamics and concept knowledge.
    The blue arrow edges represent training processes and are labeled with the task on which the model at the origin point of the arrow is trained.
    The black edges indicate the models that provide the concept representations that are analyzed.}
    \label{fig:cbwm:train_process}
\end{figure}

OTTA performance of an agent can be measured based on the adaptation of the source task 

To compare the impact of knowledge preservation on adaptation efficiency across multiple OTTA settings, CBWM agents trained to have a task-specific policy should still have world models with a task-agnostic understanding of dynamics and concepts.
To balance these two attributes, we train CBWM with an offline-to-online model-based reinforcement learning (MBRL) procedure. 
The training procedure of offline-to-online MBRL generally consists of two stages: offline \textit{pre-training} using a data set of interactions followed by online \textit{source task training} to optimize the policy for the source task and environment. 
In this work, after learning a source task, agents are trained on the target task in a third \textit{online adaptation} stage for measuring OTTA performance. 
We illustrate this training process in Figure~\ref{fig:cbwm:train_process}
By pre-training using task-agnostic data that has no overlap with any source or target tasks, we can maximize the benefits of pre-training while minimizing the risk pre-training benefits some tasks more than others.

While there is no conventional process for training Dreamer-based agents using offline data, with small modifications we are able to pre-train any Dreamer-based models on offline data before online task training.  
At a high level, the offline RL process is similar to the online RL process in Dreamer.
The offline pre-training is performed on a static dataset \( \mathcal{D}_{\text{offline}} \) collected from prior trajectories or demonstrations. 
The world models samples a batch of learning data from the dataset as it would from a replay buffer in reinforcement learning. 
As the latent actor-critic agent used by Dreamer-based approaches operates in the ``imagined'' latent state-space of the world model, the actor-critic is also pre-trained in this phase.

The modifications we made to train Dreamer with offline RL primarily address the challenge of a lack of diversity in offline demonstration data due to lack of ability to explore and the lack of state-action coverage in expert demonstrations.   
As with all offline learning, the key difference from the typical online RL is that learning from offline data provides no means for the agent to examine ``counterfactual actions''~\cite{levine2020offline} to the actions in the dataset. 
Data with counterfactual actions help the agent learn to distinguish optimal and suboptimal actions by their differing value. 
In online RL, agents encounter both optimal actions and suboptimal actions through the use of exploration.  
Because offline agents cannot diversify its learning data through exploration of unfamiliar states and actions, offline RL operates under the false assumption that the offline trajectories form a representative sample over the distribution of potentially optimal states and actions. 
Using expert demonstrations for offline RL also worsens data diversity. 
This is the reason why many popular offline reinforcement learning benchmark data sets, such as Atari100k and D4RL, are collected from the learning history of an online reinforcement learning process.
Compared to expert demonstration, learning from a learning history will provide more diversity through the rollouts from early, unconverged policies.

This lack of diversity in offline data causes a distribution shift when transitioning to online training, making errors in bootstrapping and function approximation~\cite{sutton2018reinforcement} even more pronounced than with online RL. 
Specific to Dreamer, the lack of diversity in expert demonstrations also negatively impacts world-model learning.
World model learning from expert data risks overfitting the world model to $p_{\pi^*}(s)$, the state distribution following the optimal policy $\pi^*$.
While learning in the latent space of a well-trained world model can help overcome issues with distribution shift in actor-critic learning, overfitting the world model can lead to serious problems predicting the transitions resulting from off-policy actions.

The two main steps we take in this work to alleviate distribution shift caused by offline-to-online transfer are (1)~supplementing the offline demonstration with random behavior data and (2)~data augmentation during training. 
To add diversity to the offline demonstration data, we used a simple random policy to collected a dataset of rollouts in the same environments that used to collect the demonstration data. 
Deploying a random agent on each task in the pre-training dataset, we collected 10 sequences of the same length as the average sequences in original pre-training dataset of random interaction with the environment. 
While a dataset of random behavior in a continuous control environments like Robosuite covers very little of state-action space, the added diversity from the random policy data still helps prevent overfitting in the world model during pre-training.
In addition to preventing overfitting, training the world model on off-policy actions near the initial state is helpful because it eases the distribution shift experienced by the latent actor-critic when transitioning to online training. 
Unconverged policies, either from adaptation or weight resetting, take suboptimal actions near the initial state. 
Whereas an overfit world model would be biased towards predicting optimal next states even for suboptimal actions, familiarity with random behavior near the initial state enables the world model to aid policy adaptation. 

For data augmentation, we were primarily interested in what approaches would best regularize the world model and policy to not overfit to the pre-training tasks. 
Data selection and augmentation have been shown to have a large impact on the tractability of pre-training~\cite{laskin2020reinforcement,laskin2020curl}.
During training, we built on the findings of the   APT~\cite{stooke2021decoupling,wolczyk24finetuning} method and used ``random shifting,'' which is a pad followed by a random cropping back to the original size, to prevent overfitting. 
For random cropping and translation, we looked at shift sizes of 2 and 4 pixels of the standard sized 64x64 Dreamer visual inputs.  
While we investigated temporal augmentations like frame stacking and temporal masking, our primarily investigations confirmed findings of prior work that these types of temporal augmentations yield inconsistent results across tasks and environments~\cite{laskin2020curl,stooke2021decoupling}. 

Prior work suggests that diverse offline RL pre-training of an agent can greatly improve the efficiency of online task-training with RL~\cite{schwarzer2021pretraining}.
However, this is not consistently found in all offline-to-online RL research; other prior works show that the effectiveness of pre-training is highly task dependent and can even negatively influence online task-training. 
As the primary motivation for pre-training the CBWM is to establish a task-agnostic base understanding of concepts, it is important that pre-training not have any negative impacts specifically on task training. 
As such, in addition to these modifications described above, we also experimented with alleviating distribution through partial model transfer, weight resetting, weight freezing, and other modifications. 
In order to focus this Chapter on CBWM as a means of knowledge preservation, we refer the reader to the Appendix for our theoretical justification for partial model offline-to-online transfer and further details and experiments regarding transfer, resetting, freezing, and other considerations specific to offline-to-online RL.

\section{Experiments}
\label{sec:cbwm:exp}

In our experiments, we sought to test the following hypotheses:
\begin{enumerate}
    \item Concept bottlenecks can be implemented in challenging, sequential robot learning scenarios and still show critical bottleneck attributes,  balancing concept and downstream task performance and concept intervention leading to impact on that concept in the downstream tasks.
    \item In reinforcement learning problems, grounding knowledge about task-independent semantic concepts in a bottleneck has minimal negative impact on final performance while adding interpretability and intervention capability.
    \item Enforcing the learning of task-independent semantic concepts in a bottleneck improves online test time adaptation (OTTA) to novelty in reinforcement learning through knowledge preservation.
\end{enumerate}

\noindent We used three separate experimental procedures to validate each of these hypotheses about concept bottlenecks in model-based reinforcement learning for OTTA. 
Here we describe the setup of these procedures and the baselines against which we evaluate CBWM.

\subsection{Concepts in Reinforcement Learning Environments}
\label{sec:cbwm:exp:training}

To examine the impact of a concept bottleneck on model-based reinforcement learning, we run experiments in realistic robot learning OTTA scenarios in Robosuite~\cite{robosuite2020} based on the LIBERO manipulation settings. 
The Robosuite simulation framework, based on the MuJoCo simulator~\cite{todorov2012mujoco}, is a realistic environment for vision-based robot manipulation. 
LIBERO is a dataset collected from humans that provides expert demonstrations of ``pick-and-place'' manipulation tasks. 
Both are motivated as tools to help transfer robot policies from simulation to the real world~\cite{zhao2020sim}. 
Vision-based pick-and-place is a relevant problem for combining prior knowledge with test time adaptation---while LIBERO demonstrated the effectiveness of imitation learning if you had access to task data before hand, none of the RL baselines  in the original Robosuite effort could fully solve the pick-and-place tasks.
This means that vision-based pick-and-place manipulation problems are a likely case where pre-training may be necessary to learn strong RL policies on previously unseen tasks.

Prior to this work, the vast majority of concept bottleneck research has been applied to simpler scenarios such as image classification. 
As the original LIBERO dataset does not include any notion of ``concepts'' as defined in this work, we derived a set of concepts comprised of 
the multi-hot encoding of whether objects are or are not present in the scene.
Loss balancing terms and training hyperparameters were selected in part to keep the final concept prediction accuracy over 80\%.
This ensures that the bottleneck concept predictions had a high impact on downstream tasks.

As in LIBERO, all of our tasks are pick-and-place tasks, where a target object from a variety of objects is manipulated to a goal location using a simulated model of the Franka Panda robot arm.  
The tasks we select are motivated to have the novel change between two pick-and-place tasks fall into one of three high-level categories: 
\begin{enumerate}
    \item Changes in the target objects and target object initial location
    \item Initial robot pose that varies in distance from the target object
    \item No change in the robot, target object, or goal location, but changes to the surrounding environment that vary in how much the policy is affected. 
\end{enumerate}
\noindent These task changes enable analysis of the CBWM approach for reinforcement learning tasks in the context of the novelty-ontology formulated in Chapter~\ref{chapt:novgrid} of this dissertation. 
Novelties where the changes in the target objects and target object initial location are designed to represent delta novelties.
The policy for grasping an object at a slightly different initial position or different objects in the same location are similarly difficult task to learn from scratch.
However, when the novelty changes the target object location or type of target object, the adaptation process will need to retain an understanding of the high level task similarities while adapting to these small changes. 

Novelties where the initial robot position is closer to or father from the target objects are designed to represent shortcut and barrier novelties. 
Due to the continuous action space, in pick-and-place manipulation tasks, learning from scratch to control a robot arm over longer distances is more challenging from a credit assignment and exploration standpoint.
However, when the novelty changes the initial distance to the target object location, the adaptation process will need to retain an understanding that traveling to the target object is not important compared to the the distance traveled. 

Novelties where the changes to the surrounding environment not including the target object or robot have varying impact on the optimal policy, and therefore examines all three types: shortcut, delta, and barrier.
If there is a change to objects that have no impact on the  policy, this is a delta novelty (that Boult et. al.~\cite{boult2021towards} refers to as a ``nuisance''), and the agent must be able to update the world model without impacting the policy. 
If there is a change to environment objects that block the path of the policy in one place, this is a barrier novelty and the agent must be able to update the policy and world model without discarding the unimpacted parts of the policy. 
Conversely, if the change to environment objects removes an object from the direct path to the target, this is a shortcut novelty and the agent must be able to update the policy to no longer avoid an object that is no longer there. 

A complete list and description of tasks and task pairs can be found in Appendix~\ref{app:cbwm:tasks}.

\subsection{Offline Pre-training Implementation Details}

With the exception of \textit{tabula rasa} baselines, all RL tasks are trained starting from world models pre-trained offline model-based RL. 
Recent research suggest the positive impact of diverse pre-training on efficiency in RL~\cite{schwarzer2021pretraining} (For further background on offline pre-training for model-based RL fine tuning, see the Appendix). 
To ensure that the world model, agent, and particularly the concept bottleneck are pre-trained with a diverse, task-agnostic prior, all models are pre-trained on the LIBERO\_90 dataset~\cite{liu2024libero}.
LIBERO\_90 is a collection of offline decision-making data from 90 unique pick-and-place tasks, where the data for each task contains 50 expert demonstrations per task of at least 120 frames per demonstration. 
The pretraining phase is designed solely to enable the world model to learn basic, task-independent dynamics, observation reconstruction, and the general connection between action and reward. 
As such, no differentiating semantic information about the different tasks is provided to the agent during pretraining, and none of the fine-tuning tasks exist in the pretraining dataset. 

One challenge of offline-to-online reinforcement learning in general, but especially in the case of learning from expert demonstration data, is the lack of diversity afforded by exploration~\cite{levine2020offline,nair2021awac} (for a deeper discussion of the importance of diversity in exploration for all transfer learning see Chapter~\ref{chapt:transx}). 
This is the reason many popular offline reinforcement learning benchmark datasets like Atari100k and D4RL are collected from an online reinforcement \textit{learning process}, not only an expert demonstration. 
However, it is not always the case that someone would have sufficient access to an environment to train interactively.
As such, to complement the expert demonstration data we collected a simple random policy dataset in our target environments in Robosuite. 
In our experiments, we study the importance of mixing this random data in with the expert data.

We utilize 64x64x3 resolution ``agentview'' images as is standard in training Dreamer models. 
To ensure all experiments were trained with diverse expert data, all models are pre-trained on the LIBERO\_90 dataset, which has offline RL data for 50 expert demonstrations of at least 120 frames for 90 unique pick-and-place tasks.
We split the data sets into a 95\%-5\% split for validation and training, where we validate four times per epoch. 
This same split was used when pre-training using our random behavior dataset. 
Pre-training lasts for 5 epochs of offline model-based RL. 
For our data augmentation experiments we padded these images with ``reflection padding,'' and then randomly cropped the image back down to 64x64. 
Each model is fine-tuned with online model-based RL in the Robosuite environment for an additional 500,000 environment steps, allowing for a thorough adaptation to the online setting. 
The models are fine tuned on multiple different Robosuite tasks that aims to assess how well pre-trained representations transfer to new tasks within Robosuite generally, capturing the impact of pre-training in boosting performance on the target domain.

Task training then involves fine-tuning the pretrained model with reinforcement learning on a single task. 
The inputs and outputs to the model are entirely the same, except that the observation, observation target, reward, and concept values come from the environment, and the action is sampled from the CBWM actor.

\begin{toexclude}

\subsection{Knowledge Preservation in Model-Based RL and its Impact on OTTA}
\label{sec:cbwm:exp:taskpairs}

Taken together, these previous efforts allow us to 
examine how well the CBWM can learn concepts and tasks, and more specifically whether it affects common failure modes in task transfer. 
Furthermore, we are interested in a key question that has been suggested by  prior works~\cite{balloch2023worldcloner,eaton2024icbinb} and directly pertains to this thesis: \textbf{to what extent does disentanglement of concepts in a neural network allow the network to transfer unchanged concepts between tasks?}

To investigate these phenomena, we construct a novel training procedure that uses the tasks selected for this work as pairs (representing a novelty), and retrains multiple times using these tasks to examine the difference between these models. 
Specifically, tasks are arranged into Task Pairs, $(T_1, T_2)$, where Task 1 ($T_1$) represents the task that the pretrained model, ``$M_0$,'' is fine-tuned on, and Task 2 ($T_2$) is a separate task that is used to further tune the Task 1 model, ``$M_1$''. 
We call the model that results from Task 2 fine-tuning ``$M_2$.''
Finally, the ``$M_2$'' model is then trained \textit{again} on Task 1, resulting in a fourth model ``$M_1'$.'' 
This process is illustrated in the diagram in Figure~\ref{fig:cbwm:process}.

\begin{figure}[ht]
    \centering
    \includegraphics[width=0.8\textwidth]{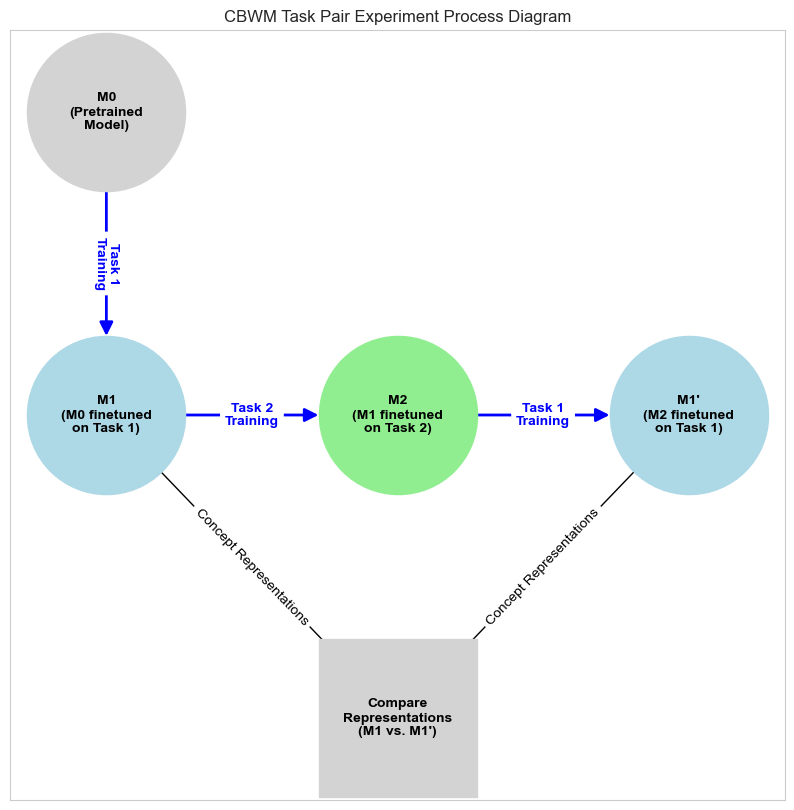}
    \caption{Figure illustrating the novel CBWM training process for investigating OTTA and the relationship between OTTA performance and concept knowledge preservation. 
    The blue arrow edges represent training processes and are labeled with the task on which the model at the origin point of the arrow is trained.
    The black edges indicate the models that provide the concept representations that are analyzed.}
    \label{fig:cbwm:process}
\end{figure}

This training procedure enables us to examine two phenomena critical to better understanding how knowledge preservation impacts OTTA performance. 
Firstly, this experiment sheds light on the impact of repeated fine-tuning of a pretrained model on both task performance and adaptation performance. 
By comparing the first $T_1$ training process to the second $T_1$ training process we can see whether adaptability improves or deteriorates in the most ideal setting: when returning to a previous known task. 
Furthermore, by comparing the task performance of model $M_1$ to multi-step fine-tuned models like $M_1'$ we can examine how the additional fine-tuning impacts final task performance of these models. 

This training procedure allows us to examine how the intermediate knowledge representations change when adapting to novelties, the degree to which supervision impacts these changes, and the relationship between knowledge preservation and adaptation performance. 
Each pair of tasks was designed so that changes during adaptation in the task are related to a known subset of concepts expected to change. 
For example, if a ``butter'' object at location $p$ in the environment is the target object in $T_1$, and then in $T_2$ the only difference is that target object is now a ``milk'' object, having replaced butter as the object at location $p$, we can expect that, in an ideal disentangled representation, the latent embedding of the ``butter'' and ``milk'' concepts should be different between $M_1$ and $M_2$, and the other concepts should be nearly identical. 
Furthermore, if we compare $M_1$ to $M_1'$, for ideally disentangled models we should see no difference in the concepts at all. 
While we do not expect any of these models to be ideally disentangled as neural network concept entanglement is highly complex, this learning setup with these baselines allows us to study the possibilities posed by our hypotheses. 

\end{toexclude}

\subsection{Model-Based RL Baselines}
\label{sec:cbwm:exp:baselines}

As baselines for comparison in all of these experiments, we trained
DreamerV3 and two variants of the CBWM architecture: \textit{BWM} and \textit{BWM+O}.
DreamerV3, as the base of our CBWM method, will allow us to examine how much the concept bottleneck framework impacts the overall task performance.
\textit{BWM} uses the CBWM bottleneck architecture, but without concept supervision or an orthogonality loss. 
By having no additional losses the gradient flow during training should be identical to that of the original DreamerV3 with extra linear units, enabling us to examine the impact of the bottleneck architecture alone on knowledge retention.
\textit{BWM+O} also uses the CBWM bottleneck architecture without concept supervision, but includes orthogonality loss. 
With the addition of the orthogonality loss but not the concept supervision we can determine the effectiveness of knowledge separation on knowledge retention.

The algorithm implementations are based on the PyTorch implementation of DreamerV3 in the SheepRL~\cite{EclecticSheep_and_Angioni_SheepRL_2023} repository. 
To accommodate the unique complexity of visual manipulation tasks, we use a custom layer configuration based on the DreamerV3 \texttt{medium} 100M parameter configuration,  replacing the configuration of all MLPs with a stack of 3 layers with 640 hidden parameters.
These values were selected based on prior results from pre-training~\cite{seo2022pretraining,wolczyk24finetuning} similar models on similar visual manipulation environments.
For all other hyperparameters in all baselines, we use the default DreamerV3 values as tuned for DMLab.


\section{Results}
\label{sec:cbwm:results}

In this section, we examine and analyze the results of these experiments and describe how the results relate to our hypotheses about CBWMs as well as the thesis as a whole. 

\subsection{Learning with Concept Bottleneck Models on Sequential Robot Data}
\label{sec:cbwm:results:prelim}

Looking at the performance of all the concepts in the LIBERO-90 dataset, we found that the accuracy of the concepts was high, even though the concepts were challenging.
Objects in the scene were often small and occluded, sometimes entirely, with only some concepts affected by robot manipulation.
In spite of these challenges, the mean accuracy across all concepts was 91.9\%
This shows that difficult, non-curated concepts can still be learned and used for this means of adding concepts.     
The mean per-concept prediction accuracy is shown in Figure~\ref{fig:cbwm:base_acc}.

\begin{figure}[t]
    \centering
    \includegraphics[width=0.8\textwidth]{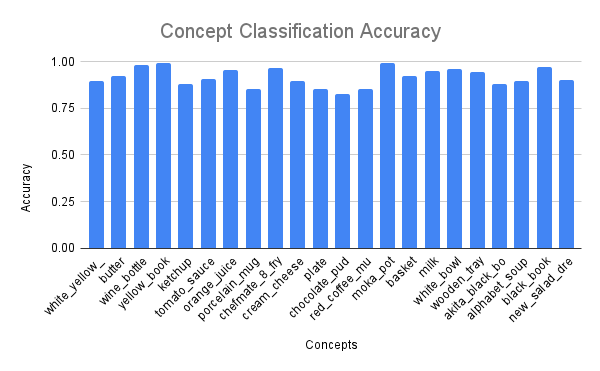}
    \caption{Concept classification accuracy across different object and state concepts in the LIBERO-90 dataset. 
    Each bar represents the accuracy of the concept bottleneck model in predicting the presence/absence of a specific concept (e.g., objects like bowls, cups, and wine glasses, or states like 'grasped'). 
    The model achieves consistently high accuracy ($>$75\%) across most concepts, with a mean accuracy of 91.9\%, demonstrating that the concept bottleneck can effectively learn and represent diverse task-relevant concepts despite the challenges of partial occlusion and varying object sizes in manipulation scenarios. Concepts are measured on the validation split of the dataset after pre-training.}
    \label{fig:cbwm:base_acc}
\end{figure}


When looking at the model's predicted observations---the downstream task that is bottlenecked---we see little qualitative deterioration of performance. 
In Figure~\ref{fig:qual}, which shows the predicted outputs, the predicted observation is able to reproduce the salient objects in the scene faithfully although there are some differences.
This shows that the concept module of the CBWM does not have a significant impact on the downstream world model modules like the observation prediction model. 

\begin{figure}[ht]
    \centering
    \includegraphics[width=0.8\textwidth]{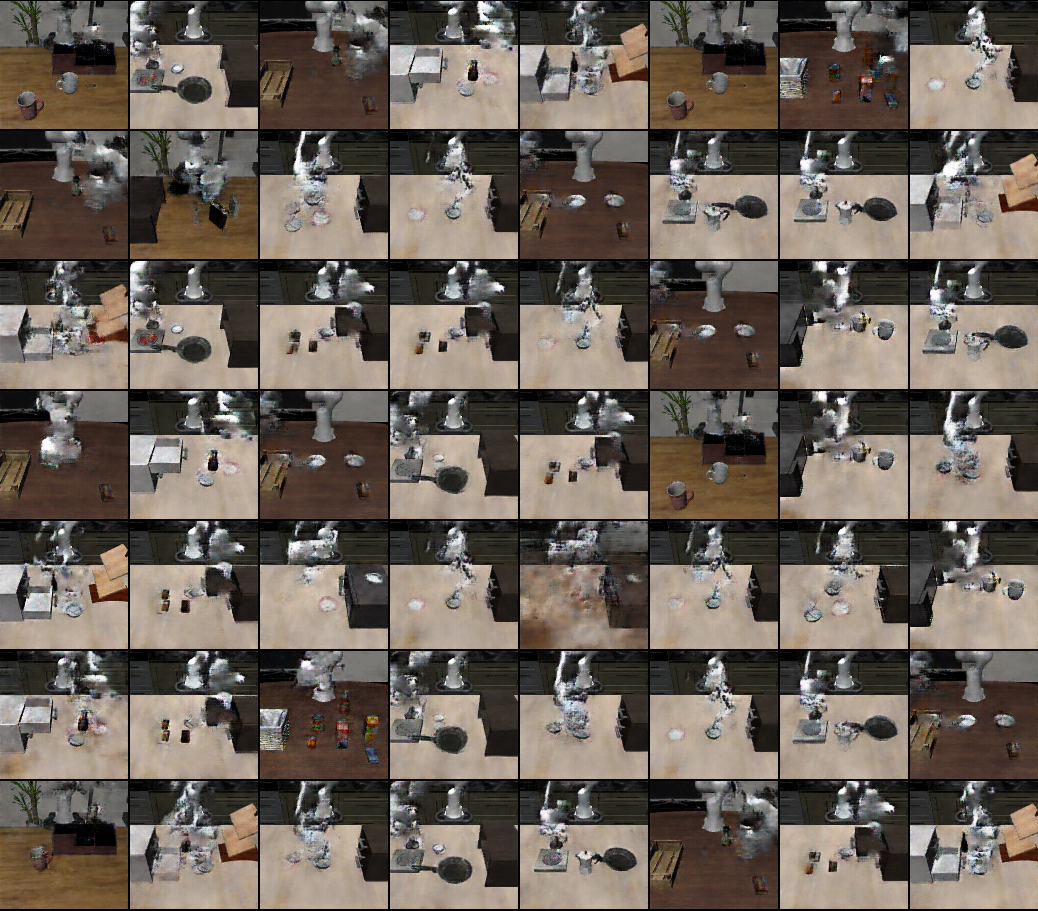}
    \caption{This shows the observation predictions of the LIBERO space. While the image fidelity varies across samples, we see that the objects, which are supported by the concepts, are very clearly predicted.
    }
    \label{fig:qual}
\end{figure}


\begin{figure}[t]
    \centering
    \includegraphics[width=0.8\textwidth]{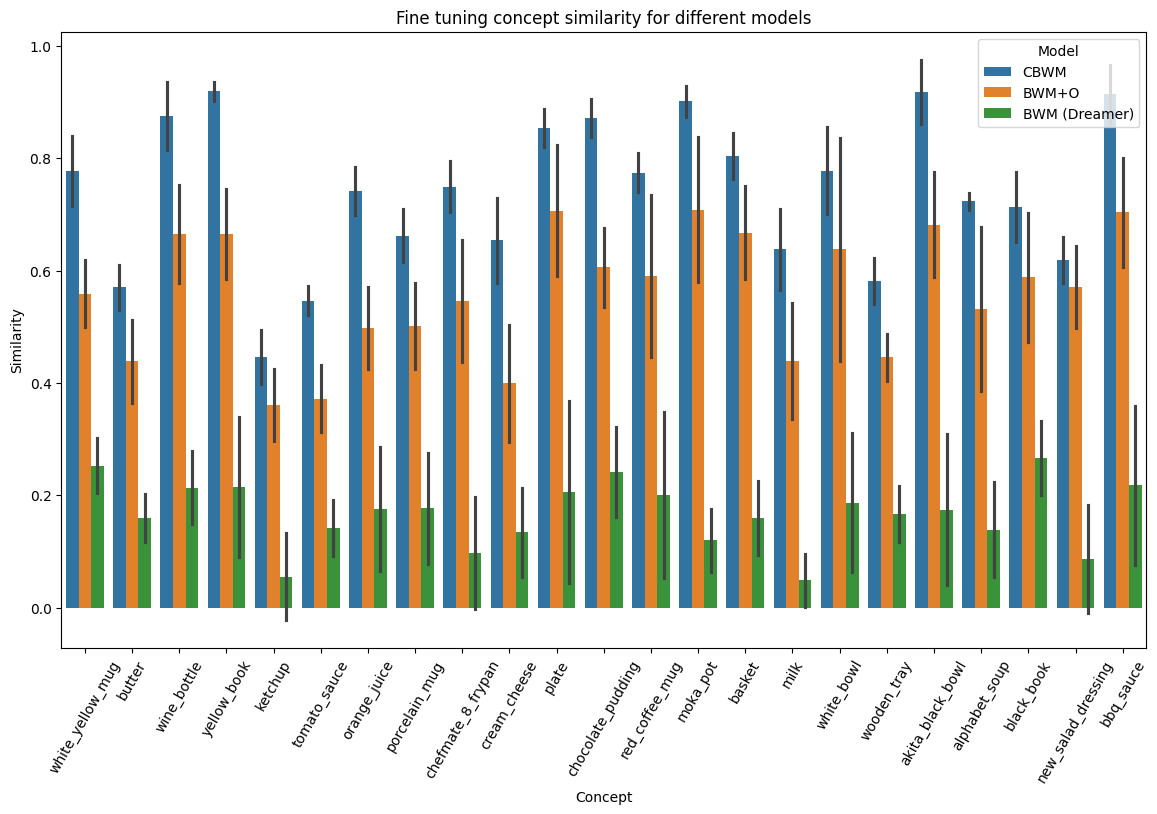}
    \caption{Plotting concept cosine similarity for individual concepts for the BWM, BWM+O, and CBWM models.  
    This demonstrates that CBWM is vastly superior at retaining concept information across adaptation. Interestingly, the orthogonality loss also exhibits strong concept similarilty across adaptation.
    This suggests that there may be a path forward for unsupervised concept discovery using orthogonality loss.}
    \label{fig:cbwm:con_acc}
\end{figure}

\begin{toexclude}

\jb{Placeholder: plot figure of average task training for all baselines, or perhaps just ones that show training well}

\begin{table}[ht]
  \centering
  \begin{tabular}{|c|c|c|c|c|c|}
    \hline
    & \multicolumn{5}{|c|}{\textbf{Mean Converged Task Reward} $\uparrow$} \\ 
    \cline{2-6} \textbf{Tasks} & DV3 & P2E & CBWM & BWM & BWM+O  \\ 
    \hline
    \hline
    Task 1 & $\pm$ & $\pm$ & $\pm$ & $\pm$ & $\pm$ \\ 
    Task 2 & $\pm$ & $\pm$ & $\pm$ & $\pm$ & $\pm$ \\ 
    Task 3 & $\pm$ & $\pm$ & $\pm$ & $\pm$ & $\pm$ \\ 
    Task 4 & $\pm$ & $\pm$ & $\pm$ & $\pm$ & $\pm$ \\ 
    Task 5 & $\pm$ & $\pm$ & $\pm$ & $\pm$ & $\pm$ \\ 
    Task 6 & $\pm$ & $\pm$ & $\pm$ & $\pm$ & $\pm$ \\ 
    Task 7 & $\pm$ & $\pm$ & $\pm$ & $\pm$ & $\pm$ \\ 
    Task 8 & $\pm$ & $\pm$ & $\pm$ & $\pm$ & $\pm$ \\ 
    Task 9 & $\pm$ & $\pm$ & $\pm$ & $\pm$ & $\pm$ \\ \hline
  \end{tabular}
  \caption{This table shows the mean final performance of the DreamerV3, Plan2Explore, CBWM, BWM, and BWM+O models on all tasks after 1x$10^6$ training steps for each task.}
  \label{tab:cbwm:trainingperf}
\end{table}

\begin{table}[ht]
  \centering
  \begin{tabular}{|c|c|c|c|c|c|}
    \hline
    & \multicolumn{5}{|c|}{\textbf{Adaptive Efficiency} $\downarrow$} \\ 
    \cline{2-6} \textbf{Tasks} & DV3 & P2E & CBWM & BWM & BWM+O  \\ 
    \hline
    \hline
    Task 1 & $\pm$ & $\pm$ & $\pm$ & $\pm$ & $\pm$ \\ 
    Task 2 & $\pm$ & $\pm$ & $\pm$ & $\pm$ & $\pm$ \\ 
    Task 3 & $\pm$ & $\pm$ & $\pm$ & $\pm$ & $\pm$ \\ 
    Task 4 & $\pm$ & $\pm$ & $\pm$ & $\pm$ & $\pm$ \\ 
    Task 5 & $\pm$ & $\pm$ & $\pm$ & $\pm$ & $\pm$ \\ 
    Task 6 & $\pm$ & $\pm$ & $\pm$ & $\pm$ & $\pm$ \\ 
    Task 7 & $\pm$ & $\pm$ & $\pm$ & $\pm$ & $\pm$ \\ 
    Task 8 & $\pm$ & $\pm$ & $\pm$ & $\pm$ & $\pm$ \\ 
    Task 9 & $\pm$ & $\pm$ & $\pm$ & $\pm$ & $\pm$ \\ \hline
  \end{tabular}
  \caption{This table shows the mean adaptive efficiency of the DreamerV3, Plan2Explore, CBWM, BWM, and BWM+O models on all tasks.}
  \label{tab:cbwm:adaptperf}
\end{table}

\subsection{Task Pair Multi-Step Training Demonstrates Improved Knowledge Preservation}
\label{sec:cbwm:results:taskpairs}

The task pair training process showed very interesting results. \jb{ANALYSIS EXPECTED RESULTS:}
Given the fine-tuned models whose results were described in Section~\ref{sec:cbwm:exp:training}, each of these models was fine-tuned on pair tasks as described in Section~\ref{sec:cbwm:exp:taskpairs} to produce nine sets of $(M_1,M_2,M_!')$ models for the CBWM, BWM, and BWM+O methods. 

In the cases of all pairs across all algorithms, we observed a monotonic improvement in adaptive efficiency as models were repeatedly fine-tuned. 
This confirms the widely held belief that even entangled knowledge is helpful in adapting models; however, as we discuss in Chapter~\ref{chapt:transx}, specifically with the discussion of challenging barrier novelties in Section~\ref{sec:transx:disc}, this could also be a product of the fact that the MDPs joined by this novelty are in all cases very similar. 
Repeated fine-tuning likely did remove some of the general knowledge from pretraining that served as a good initialization for training $M_1$, which has been observed concretely in the training of LLMs~\cite{mukhoti2023fine}. 
We expect that efficiency consistently improved in spite of this because the $M_2$ had retained more of the knowledge of the task $T_1$ from $M_1$ than was originally present in $M_0$, allowing it to ``code switch''\MarkLeft{This is a sensitive term, and maybe consider avoiding} to $T_1$ more easily when training $M_1'$. 
\julia{if this holds in the final data, this is a very cool finding!!!}
While this phenomenon is almost certainly sensitive to model size, we believe these data suggest that there may be a determinable minimum number of task switches between pairs of tasks after which the fine tuned model will simply learn to be a multi-task model.

\begin{table}[ht]
  \centering
  \begin{tabular}{|c|c|c|c|}
    \hline
    & \multicolumn{3}{|c|}{\textbf{Adaptive Efficiency} $\downarrow$} \\ 
    \cline{2-4} \textbf{Models} & CBWM & BWM & BWM+O  \\ 
    \hline
    \hline
    $M_1$ & $\pm$ & $\pm$ & $\pm$ \\ 
    $M_2$ & $\pm$ & $\pm$ & $\pm$ \\ 
    $M_1'$ & $\pm$ & $\pm$ & $\pm$ \\ \hline
  \end{tabular}
    \begin{tabular}{|c|c|c|}
    \hline
    \multicolumn{3}{|c|}{\textbf{Mean Converged Task Reward} $\uparrow$} \\ 
    \hline
    CBWM & BWM & BWM+O  \\ 
    \hline
    \hline
    $\pm$ & $\pm$ & $\pm$ \\ 
    $\pm$ & $\pm$ & $\pm$ \\ 
    $\pm$ & $\pm$ & $\pm$ \\ \hline
  \end{tabular}
  \caption{This table shows the mean adaptive efficiency of CBWM, BWM, and BWM+O models for sets of $(M_1,M_2,M_!')$ averaged over all task pairs. \jb{This is one of the few scenarios where I truly believe the data could be better represented as a bar graph. There should be 3 chunks of bars, one for each of M1,M2, and M1', and in each chunk there would be 6 bars, 2 for each method where the efficiency bar would be solid and the task reward would be lined,  and each method would have a different color. I added placeholder Figure~\ref{fig:cbwm:taskpairbar} to demonstrate}}
  \label{tab:pm_table}
\end{table}

\begin{figure}[t]
    \centering
    \includegraphics[width=1\textwidth]{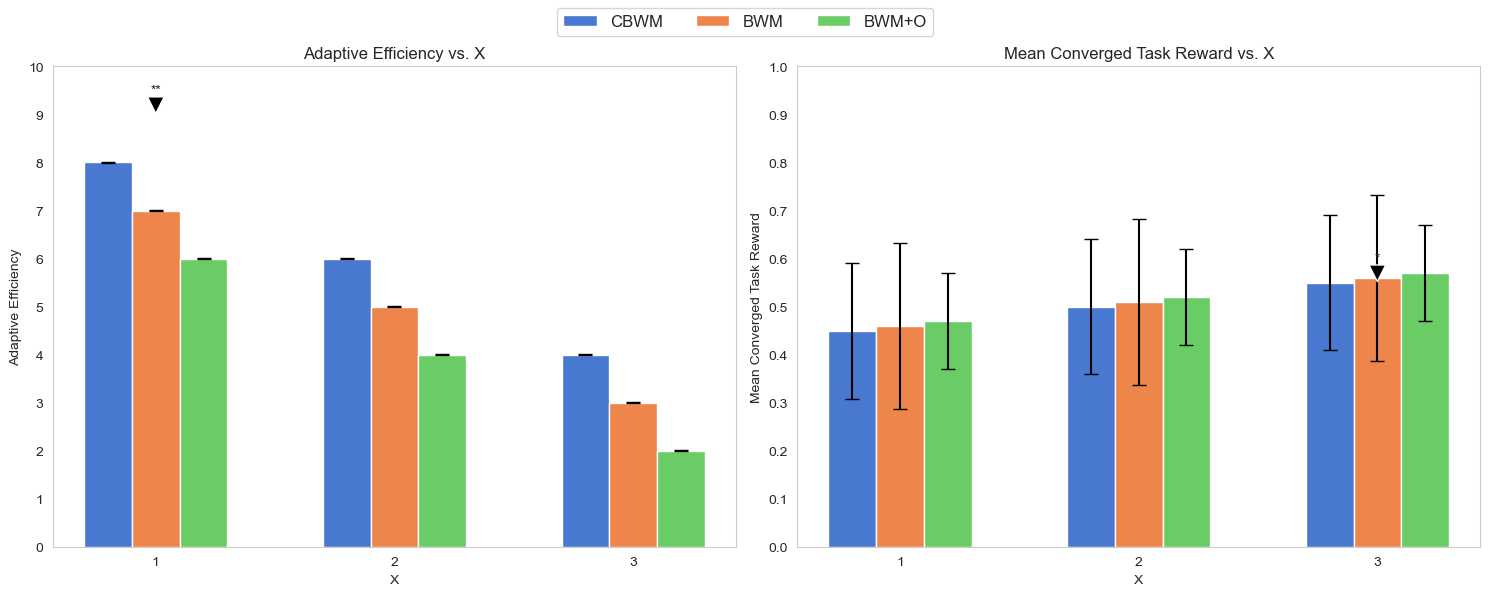}
    \caption{PLACEHOLDER PLACEHOLDER PLACEHOLDER DO NOT LEAVE IN DELETE}
    \label{fig:cbwm:taskpairbar}
\end{figure}

CBWM showed superior adaptative efficiency to BWM and BWM+O. Interestingly, however, BWM+O also consistently outperformed BWM. 
While not as impactful as the concept supervision loss, this shows that, consistent with the literature around bases in mechanistic interpretability of neural networks~\cite{elhage2022superposition}, an embedding that is nearly orthogonal has strong representation power~\cite{johnson1984extensions}.

\end{toexclude}

\subsection{Embedding Concepts Helps Knowledge Preservation and Adaptation}

To study the models' ability to preserve knowledge during adaptation, we revisit novelties we selected to study as described in Section~\ref{sec:cbwm:exp}. 

For adaptation to novelties that changes in target object initial locations, we observe a similar ability to adapt in BWM, BWM+O, and CBWM.
The fact that there is little implicit capture of unsupervised location information in concept bottleneck models is consistent with the findings in prior work.~\cite{raman2023do} 
For adaptation to grasping a different target object, on the other hand,  BWM and BWM+O underperformed CBWM.
This is likely because while the knowledge associated with the target object impacted all other concepts in BWM and BWM+O, in CBWM the network impact of the distribution shift is eased by the consistency of the concept loss being present in source task training and adaptation. 


 \begin{figure}[t]
    \centering
    \includegraphics[width=0.8\textwidth]{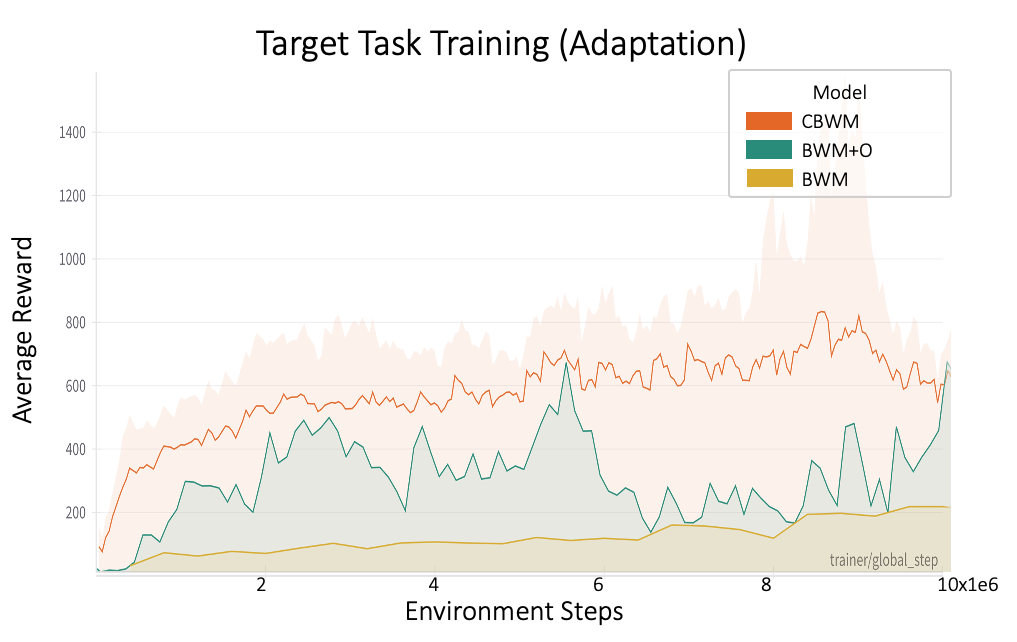}
    \caption{The OTTA learning curves averaged over all tasks for CBWM, BWM+O, and BWM when transferring from source to target task. 
    The speed with which the average return increases for CBWM and BWM+O, in addition to the final performance after 10 million steps, shows that concept and orthogonality losses help transfer reusable concept knowledge.
    While BWM+O shows the efficacy of the orthogonality loss without concept supervision, the high variance shows the instability of this approach.
}
    \label{fig:cbwm:adaptive_eff}
\end{figure}

We see similar patterns emerge when comparing the adaptation performance of BWM, BW+O, and CBWM models in initial robot position novelties and novelties that change the surrounding environment not including the robot and target objects.
Adaptive performance for initial robot position novelties is similar for BWM, BWM+O, and CBWM models. 
This echoes the comparative performance of adaptation to target object initial location changes and suggests that concepts related to object location or the agents policy may be necessary to improve performance in novelties involving differences in spatial distribution. 
In contrast, for novelties where changes to the surrounding environment cause an object to block or unblock the path of the policy, BWM and BWM+O had less efficient adaptation than CBWM. 
This can be explained because the change to the world model and policy largely concerns a small subset of object concepts, and the only model in which those concepts are learned to be separated is the CBWM.

To measure the amount of change in concepts before and after adaptation to the target task as the cosine distance between two models' embeddings of that concept. 

\begin{equation*}
    \text{Cosine Distance} = 1 - \frac{\mathbf{e_{c,1}} \cdot \mathbf{e_{c,2}}}{\|\mathbf{e_{c,1}}\| \|\mathbf{e_{c,2}}\|}
\end{equation*}

\noindent where $e_{c,i}$ is the embedding vector of concept $c$ in model $M_i$. 
In comparing concepts, because the original Dreamer approach has no separated knowledge, we compare the concept and adaptive performance of the BWM, BW+O, and CBWM models. 
Across all of these tasks, we examined the preservation of concept information.
This data is shown for all individual concepts in Figure~\ref{fig:cbwm:con_acc}
For all tasks, the CBWM approach most clearly exhibits a change in task-dependent concepts that was greater than in task-independent concepts. 
This shows that concept supervision indeed helps to preserve knowledge. 


\begin{toexclude}

For these results, we compared the embeddings of both $(M_1, M_2)$ and $(M_2, M_1')$ model pairs. 
Generally we observed there was a greater magnitude of change overall for $(M_1, M_2)$ than $(M_2, M_1')$. 

Beyond solely preserving knowledge, our results also confirm our hypothesis that better preservation of knowledge correlates with faster task adaptation time.
\jb{todo this is a perfect place for a scatter plot. x axis is adaptation efficiency, y axis is *total amount* of change in an embedding space the sum of cosine distances for all concepts in a pair. Different icons/colors for different methods}. 
As the Figure~\ref{fig:cbwm:scatterplot} shows, where we plot all seeds of all model pairs to compare the total embedding change to the task adaptation performance, there is a clear trend that less embedding change can be adapted more quickly. 
This is a logical conclusion; after all, less change should be easier to optimize for. 
However, the plot also shows that the CBWM clearly produces the largest number of high efficiency, low difference embeddings, demonstrating that the CBWM is better at producing these low-change embeddings.

Finally, we compare pairs of models trained for the same task but with different numbers of fine-tuning processes $(M_1$ $M_1')$, where for ideally disentangled models we would expect to see no difference in the concepts at all. 
\jb{these are the results I am least confident in so I will keep the ``anticipated analysis'' brief}.
Interestingly, we find that while the CBWM does produce the most similar $(M_1$ $M_1')$ embedding pairs, the difference from the BWM and BWM+O methods is less pronounced than in the other embedding results. 
We believe that this is because, while the BWM and BWM+O concept embedding spaced are more entangled, the concept models for a given task are fairly consistent and can return to something close to the original embedding if the embedding space is not changed too extremely, as is the case in all of these adaptation scenarios.

These results suggest that embeddings may have something analogous to a ``yield point'', i.e. a point beyond which additional changes through fine tuning prevent it from returning to its original distribution through gradient descent. 
Given such a point-of-no-return, we believe these results suggest that concept bottleneck disentanglement can help prevent embeddings from changing irreparably.

\end{toexclude}

\section{Key Takeaways}

In this chapter, we introduced Concept Bottleneck World Models (CBWMs), a novel approach to model-based reinforcement learning that incorporates human-interpretable concepts into the world model architecture. 
Our approach addresses the challenge of entangled concept knowledge in neural network-based world models, which can lead to unnecessary loss of information during task transfer and adaptation to novel scenarios.
Through our experiments with the LIBERO dataset, we have demonstrated several key findings using the CBWM approach:
\begin{enumerate}
    \item High concept accuracy: Despite the challenging nature of the concepts in the LIBERO-90 dataset, including small and often occluded objects, our model achieved a mean accuracy of 91.9\% across all concepts. This demonstrates that CBWMs can effectively learn and utilize difficult, non-curated concepts. 
    \item Preserved observation prediction quality: The integration of the concept bottleneck did not significantly deteriorate the model's ability to predict observations. Our qualitative results show that the model can faithfully reproduce salient objects in the scene, indicating that the concept bottleneck effectively structures the learned representations without compromising performance.
    \item Preservation of knowledge: when adapting from source to target task, concept cosine similarity is improved by enforcing concept orthogonality in the latent world model bottleneck and significantly improved when also supervising the concepts.
    \item Knowledge preservation improves adaptive efficiency: By constraining the information in the bottleneck module with the concept and orthogonality losses, CWBM avoids overwriting of reusable concept knowledge.
\end{enumerate}

\noindent These results suggest that CBWMs offer a promising approach to improving the adaptability, interpretability, and efficiency of model-based reinforcement learning systems. 
By explicitly representing interpretable concepts, CBWMs can help mitigate the problem of unnecessary knowledge loss during task transfer and enhance the model's ability to adapt to novel scenarios.

The work presented in this Chapter, as with all work, has limitations that provide opportunities for future investigation, specifically in the sourcing, interpretation, and preservation of concepts in OTTA.
Critically, CBWMs simply model the environment as a linear embedding space of selected concepts by the researchers with knowledge of the task. 
This of course begs the question: what if you do not have prior knowledge on what concepts are well suited to the task or OTTA problems, or have no concept supervision at all?
It is important, therefore, that CBWMs are extended to ``discover'' concepts.
By modeling the concepts in the bottleneck with unsupervised object-centric modeling techniques, such as slot attention, and grounding techniques, such as CLIP, CBWMs would be able to bottleneck and reason over knowledge. 
In addition, one of the important attributes of concept bottleneck models is the addition of interpretability in an otherwise uninterpretable black box.
As both adaptation and reinforcement learning are generally avoided in critical systems as the process is uninterpretable, studying the interpretability inherent to CBWMs can be a path to added human trust in sequential decision making agents. 
Lastly, knowledge preservation in this work is mainly guide by the concept bottleneck's loss functions.
However, there is a long history of excellent work studying very similar problems in continual learning and active learning research. 
By extending the knowledge preservation properties of CBWMs with continual learning and active learning techniques, the knowledge preservation and adaptive efficiency properties of CBWMs can be further improved.

The key insight provided by the contributions in this final Chapter is that efficiently adapting agents benefit from an understanding of the knowledge in the network and how that knowledge relates to available data.
In CBWMs, concept label supervision is not just another downstream task, as is common in many neural network approaches where input data have rich label information; if we had instead simply had another concept head, we may be able to benefit from that data for \textit{tabula rasa} learning, but it would not be beneficial to adaptation. 
By understanding that there is an underlying relationship between a concept-base decomposition of the agent's environments, using concept labels as a means of constraining and grounding the latent instead of just predicting the concepts independently from the other downstream tasks, we are able to improve adaptation with data that otherwise may not be very beneficial.

The ability of CBWMs to efficiently adapt without sacrificing performance on the challenging task of vision-based manipulation provides strong validation for the thesis.
By applying the principles revealed by using symbolic models in Chapter~\ref{chapt:knowledge} to neural architectures, CBWM demonstrates that neural models can also improve adaptation if knowledge is preserved.
The results demonstrate that by enforcing concept representations in the bottleneck architecture, CBWMs can disentangle polysemantic knowledge and thereby preserve some concepts during adaptation while allowing other concepts to update appropriately. 
The success of CBWMs shows that by selectively preservating important prior knowledge through constraints on latent world model representations we can increase the efficiency adaptation. 
The finding that concept supervision and orthogonality constraints improved adaptation more than architectural changes alone (as shown by comparison with BWM and BWM+O baselines) confirms our thesis that explicitly regulating how prior knowledge is preserved and updated is critical for efficient adaptation. 
These results demonstrate that structured approaches to knowledge representation and preservation, as proposed in our thesis, can significantly improve an agent's ability to adapt to environmental changes while maintaining important prior capabilities.

%% file: chapters/8_conclusion_chapt.tex
\chapter{Conclusions}
\label{chapt:conclusion}

\section{Contributions}

In this dissertation, we investigated methods to improve the efficiency and overall performance of reinforcement learning agents when adapting at test time to unexpected and previously unseen changes in the environment.
This dissertation has advanced the understanding of rapid adaptation to novelty and the role played by data sampling and transferring prior knowledge through the formulation of testing frameworks, clearer definition of the problem of online test time adaptation to novelty in reinforcement learning, and rigorous experimentation and evaluation, . 

Specifically, this dissertation shows that two critical yet separate components of effective OTTA in reinforcement learning are (1) exploration and (2) knowledge preservation. 
We show that within the large set of exploration methods already proposed to improve traditional reinforcement learning, stochasticity and diversity play a key role in the adaptation of on-policy model-free reinforcement learning (Chapter~\ref{chapt:transx}). 
Considering the sample efficiency potential of model-based reinforcement learning, this dissertation then demonstrates how the findings of Chapter~\ref{chapt:transx} can not only be applied to MBRL as well, but how the use of exploration is linked to the higher-level issue of sampling in reinforcement learning. 
We show in Chapter~\ref{chapt:dops} that, since MBRL approaches like the Dreamer family of models separate the environment interaction and data sampling processes, improving the efficiency of data sampling for policy and world model learning is as critical for OTTA as it is the environment interaction with exploration. 

Knowledge preservation is a widely desired attribute in all of transfer learning, and the complexity and concept entanglement inherent to deep neural networks is a hindrance to updating only incorrect prior knowledge. 
However, unlike in transfer learning problems where the most important outcome is strong converged performance, OTTA solutions seek to minimize drop in performance over the period of learning as well. 
This makes catastrophic forgetting of prior knowledge a more critical problem. 
Our work in Chapter~\ref{chapt:knowledge} addresses this problem explicitly in MBRL by using a grounded symbolic representation and learning method for the world model, while using a neural network for the policy model. 
We show that this hybrid approach allows us to improve OTTA performance by maintaining high task performance with the policy while avoiding forgetting in a rapidly-updating world model.
Finally, we apply the lessons learned in Chapters~\ref{chapt:transx},~\ref{chapt:dops}, and~\ref{chapt:knowledge} to build an end-to-end learnable neural model with grounded representations in the concept bottleneck world model. 
In Chapter~\ref{chapt:cbwm}, we show that the addition of a grounded bottleneck in world model learning adds interpretability and improved adaptation efficiency, while also providing a direct link between knowledge preservation and adaptive efficiency in OTTA scenarios.

Beyond advancing the understanding of improving the efficiency and performance of OTTA in RL, the work described in this dissertation has meaningful implications for the application of RL in real-world scenarios. 
Even as we are living in what many call an ``AI Revolution,'' and even with reinforcement learning from human feedback playing a critical role in that ``revolution''~\cite{ouyang2022training}, the application of RL to solve real world problems is limited~\cite{casper2023rlhf,levine2020offline,rlroboticssurvey,Dulac-Arnold2021rwrl}.
Real-world decision-making problems can rarely be well defined as a closed system that experiences no changes through time, and as such need to be able to adapt. 
Whether it is the deterioration of warehouse robots or HVAC systems managed by a behavior controller, language changing as new slang is introduced, or shifting behavior of online crowds modeled by a recommender system, change in real world machine learning applications is inevitable.
The work in this dissertation represents a critical step toward enabling reinforcement learning agents to adapt on the fly to unexpected changes such as these, and hopefully an improvement to our world as a result. 

\section{Key Takeaways}

Readers of this dissertation should take away a few critic lessons from this work.
Adaptation to non-stationarity is a highly-relevant but challenging problem for real-world applications of reinforcement learning. 
To make practical progress in adaptation to non-stationarity we must examine ways to simplify the characterizations of non-stationary phenomena so we can in turn develop efficient adaptive solutions.
NovGrid, the ontology of novelties, and the proposed metrics for measuring OTTA performance provide a starting point for researchers to test the OTTA performance of existing methods and develop novel solutions.
As non-stationarity is an undeniable reality all real world agents will face, this work provides a means of characterizing adaptive response, but also a template for how online test time adaptation can be measured and investigated in other domains.

First, it is critical to consider the data from which RL agents adapt. 
As data for adaptation is acquired by agent interaction in OTTA, the work in Chapter~\ref{chapt:transx} shows that the exploration of RL agents in OTTA settings must take into consideration the relationship between the problem setting and the agent's capacity to explore. 
The data sampled through exploration is critical for adaptation, and, as the results in Chapter~\ref{chapt:transx} show, exploration methods ideal 
for pre-novelty policy convergence 
are not necessarily best suited to adaptation.
By considering exploration as something that depends on the characteristics of the environment and potential novelties agents can adapt more capably either through a single ideal exploration method or exploration specifically selected according to a novelty. 

The other side of the learning data coin is how to choose which data to use for adaptation.  
The findings in Chapter~\ref{chapt:dops} show that sampling strategies in model-based RL must be tailored to the distinct learning objectives they serve. 
The success of DOPS suggests that the conventional approach of using identical sampling distributions for world model and policy learning may be fundamentally limiting. 
Instead of conceptualizing end-to-end architectures as monolithic, researchers should consider how gradients from different objectives impact different parts of an architecture, and train with data that balances the needs of the overall architecture with specific model parts.
In the same way that we designed DOPS by first examining the distinctions between the learning, designers of all neural architectures with multiple objectives or ``heads''---not just deep RL---should consider how to properly handle parts of an architecture that are differently affected by these objectives. 
By combining these insights with the insights on exploration from Chapter~\ref{chapt:transx}, future work can continue to progress toward real-time adapting RL agents.

Second, it is critical to weigh the value of prior knowledge already available to agents when they are adapting. 
The key insight from Chapters~\ref{chapt:knowledge} and~\ref{chapt:cbwm} is that architecture design of data-driven OTTA agents has an outsized impact on what prior knowledge can be preserved and therefore what prior knowledge needs to be updated. 
In the specific case of WorldCloner from Chapter~\ref{chapt:knowledge}, this is demonstrated with improved agent adaptation using a symbolic world model and rule learner that does not use gradients that impact the entire representation.
In Chapter~\ref{chapt:cbwm}, the concept bottleneck in the world model latent space affords CBWMs the ability to efficiently adapt without sacrificing performance on the challenging task of vision-based manipulation.
Both share a common attribute: the changes do not manifest as just another downstream task, as is common in many neural network approaches but instead fundamentally modify the agents internal representation of the MDP.
When considering whether parts of an agent's architecture ought to be use symbolic, neural, or other representations, we cannot maximize the adaptive efficiency of the agent without consideringabout how that particular architecture is suited to this specifics of the OTTA setting.

Taken together, these insights reflect the thesis of this dissertation:
to efficiently adapt online to changes in the environment, reinforcement learning agents must (1) use exploration and sampling strategies that prioritize task-agnostic interactions and learning data to reduce distribution shift, and (2) identify and selectively preserve reusable prior knowledge in symbolic and learned representations.

\section{Future Work}

This dissertation also sets the stage for future scientific inquiries that would expand the reach and effectiveness of the techniques and problems described herein.

\subsection{An Extended Definition of Online Test-Time Adaptation to Novelty}

An important next step in the work of online test-time adaptation to novelty---as presented in Chapters~\ref{chapt:novgrid}~and~\ref{chapt:transx}---is \textbf{a more complete and precise definition of novelty}. 
Specifically, before applying OTTA solutions to real world problems, we must:
\begin{enumerate}[(a)]
    \item more precisely quantify how and how much environments change,
    \item extend the definitions of OTTA and novelty to include \textit{continuous change}, and
    \item investigate solutions for adapting to behavior change in multi-agent settings.
\end{enumerate} 
New works on measuring task complexity in deep reinforcement learning~\cite{pmlr-v139-furuta2021fim} and quantifying disentanglement in deep neural networks~\cite{carbonneau2022measuring} represent a promising starting point for quantifying novelty. 
In addition, the large body of theoretical work on solving non-stationary processes touched on briefly in Chapter~\ref{chapt:background} can serve as a foundation for solutions to more precise and complete novelty definitions.
By combining this with new ideas such as measuring a model's potential to accommodate change Lipschitz bounds~\cite{lecarpentier2021lipschitz}, future work can develop OTTA solutions with RL that apply to a broader set of tasks while having more specific expectations of behavior. 

In the multi-agent setting, there already exists prior work investigating non-stationary agent behavior, as behavior distributions are only stationary in a small set of situations~\cite{torrey2013teaching,vinyals2019grandmaster,kejriwal2021multi,sarathy2021spotter}. 
Knowledge preservation and concept modeling could help by modeling concepts of agent behavior, perhaps even initialized trivially from self-play.  
Moreover, taking action to investigate changes in external agent behavior gives a more complex meaning to ``exploration'' of novelty---resembling how children learn by eliciting behaviors of adults~\cite{marcus1975child,yurkovic2021multimodal}---which presents a highly impactful direction of research. 

\subsection{Learning from Safe Exploration of Specific Phenomena}

For autonomous agents in safety-critical situations, such as a robot surveying the site of a natural disaster, the promise of OTTA is very attractive. 
Usually, if an autonomous agent is being used for a task, it is often because human participation is dangerous or undesirable and real-time intervention is not possible. 
However, in safety-critical situations the way autonomous agents react to and interact with novel phenomena depends heavily on the specific nature of the novelty, and could mean the difference between success and failure. 
That said, there is reason to believe this is a learnable skill; studies in animal behavior show that in biological intelligence~\cite{reale2001temperament} exploration of novel phenomenon and safety are inextricably linked. 

The work discussed in Chapters~\ref{chapt:transx}~and~\ref{chapt:dops} provides a first step toward understanding the connection between agent exploration, data sampling, and the use of latent space ``imagination'' in model-based RL. 
In addition to RL methods such as DOPS described in Chapter~\ref{chapt:dops} and its baseline Curious Replay~\cite{kauvar2023curious}, many research areas are concerned with, given a learning goal, identifying the most effective interactions and data for learning. 
Prior solutions developed for active learning and streaming learning settings often must find the best way to sample from a pool or stream of data to maximize notions of coverage and efficiency and remove bias from sampling~\cite{pmlr-v108-shui20deepactive,chen2021active,ren2021survey}. 
The next step will be to combine these exploration and sampling approaches with methods on novelty characterization and safety-critical systems.
One approach enabled by recent work is the use of task-agnostic sources of general knowledge such as Large Pretrained Models (LPM). 
While LPMs cannot be expected to distinguish safe and unsafe novelties in every scenario, compared to a task-specific agent, models trained on a massive set of task-agnostic data are likely to provide a less biased prior about whether a novelty is safe and how confident it should be in a given safety estimate. 

\subsection{Latent Concepts for Agent Introspection and Interpretability}

Lastly, the work in Chapters~\ref{chapt:knowledge}~and~\ref{chapt:cbwm} opens up a wide range of new research directions in the use of concepts---and intermediate representations in general---for reinforcement learning. 
Most pressing, I believe, are more human-focused studies on the added practicality of the interpretability and utility of a concept bottleneck in reinforcement learning. 
Like the disentanglement research that preceded it, concept bottlenecks offer grand promises of neural model interpretability, but rarely test this in human studies. 
Reinforcement learning agents, as with all decision-making systems, are designed primarily to take action without the intervention or involvement of a human agent. 
For human-centric or historically human-controlled systems, this handover of decision-making power requires \textit{trust}. 
While some of that trust will come with exposure to AI systems over time, it is critical for trust and adoption that RL agents can provide human decision makers with explanations and interpretations of why a decision was made~\cite{schmidt2019quantifying,ehsan2021expanding,li2022interpretable,ehsan2023charting}, and models like CBWM provide a starting point to investigate this in reinforcement learning systems. 

Beyond interpretability, there are many questions that still need to be answered about concept bottlenecks before they are used more widely in MBRL, both for adaptation and in general. 
Most specific to the work in this thesis, the question remains: what should be done to more actively preserve prior concept knowledge given an understanding of weight importance and semantic understanding of concepts? 
Work on continual learning provides a strong foundation for this research direction, especially given the recent interest in continual reinforcement learning in general~\cite{parisi2019continual,khetarpal2022towards,abel2024continual}. 
Using concept bottlenecks to disentangle and specifically force a dependence of downstream tasks on grounded intermediate representations is fertile ground to reconsider how continual reinforcement learning techniques can be used to improve adaptation to changing environments.

More broadly, what should constitute a ``concept'' in reinforcement learning and how can concepts best constrain downstream tasks? 
The work in this dissertation assumes, like most concept bottleneck work, that there is a single bottleneck where all concepts are predicted, and that concepts are easily grounded, perceivable phenomena.  
This is a reasonable approach as concepts can be assumed to depend on a shared encoding of the input observation; 
however, if we consider machine learning models designed solely to model concept phenomena, we see a wide range of differences in architecture and learning method, from transformers modeling language~\cite{vaswani2017attention}, to diffusion models for visual data~\cite{croitoru2023diffusion}, to CNNs and state space machines for audio and raw signal data~\cite{gu2022efficiently}. 
Given these differences, it begs the question: is the shared-encoder, single-bottleneck approach the correct solution, 
or are methods like Capsule Networks from Sabour, Frost, and Hinton~\cite{sabour2017dynamic} better suited to the task? 
Moreover, what if concept phenomena are not independent, such as hierarchical concepts? 
Would representation of the bottleneck as, for example, a graph affect performance or induce leakage~\cite{havasi2022addressing}, effectively making the bottleneck obsolete? 
For us to realize the full potential of concept bottlenecks---which I believe could be the most impactful technique for adding interpretability and adaptability to black-box machine learning methods like neural networks---these are all questions that need to be investigated. 

\subsection{Symbolic Concept Relationships for Offline-to-Online Reinforcement Learning}

An interesting future direction for this work would be to address the problem of offline-to-online model-based RL using concept learning to address both objective mismatch~\cite{lambert2020mismatch} and the issues of diversity in offline RL data. 
Given some prior knowledge of the local dynamics that relate a concept and action---for example how a joint angle changes as a result of the joint velocity controlled by the policy---local symbolic dynamics could be used to supervise or simulate the neural world model's prediction of the next state. 
This single, simple relationship would contribute in multiple ways.

Firstly, unlike every other loss in the Dreamer learning algorithm, this would apply in both world model learning and policy learning.
If so desired, this would even allow policy gradients to be directly propogated back to the learning of the world model, fully removing the separation between world model and policy learning as separate properties.
The shared loss would serve to better align the gradients of the world model and policy learning objectives, and therefore avoid performance degradation resulting from objective mismatch in model-based RL.

Secondly, the relationship between concept dynamics and actions could be used to diversify the data used for offline training. 
Many local dynamics relationship hold in most if not all situations; if you can frame a localized subpart of the dynamics as a closed system its much easier to define simple rules that govern that system~\cite{precup1997multi,khetarpal2021partial,alver2024partial}. 
As a result, without much risk of error, synthetic counterfactual action data can be generated with respect to actions from an offline dataset because it is known how the the dynamics concepts ought to evolve given alternative actions. 
In addition, concepts could be used to generate artificial data to diversify world model learning.
Given a single offline state-action sequence, noise can be added to the actions that can then be compensated for in the concept layer by intervening on concepts with the known value that should result from the changed action. 

Although there is no direct analog to predicting reward and critic values, one way to use local dynamics knowledge for value learning could be to use synthetic data generation to learn more ``robust'' reward functions with inverse reinforcement learning~\cite{abbeel2004apprenticeship}.
Given known reward states from offline data, synthetic trajectories could be generated using bidirectional search over the configuration space to identify variations in the actions and concept dynamics that would result in a new trajectory with the same start state, end state, and final reward. 
While neural networks generalize well on their own, because the synthetic trajectories connect novel states to known rewards, they would form a conservative lower bound for reward prediction, which would ease issues with critics overvaluing out-of-distribution states navigated to by offline RL policies~\cite{laroche2019safe,kumar2020cql} 
Augmenting the offline data with the synthetic variational data for the reward function learning, one could train a reward function that was more likely to be accurate off-policy, resulting in stronger critic training.

%% file: chapters/appendix.tex

\begin{theappendices}

\section{Transfer Exploration: Algorithmic Instantiation Exploration Characteristics}
\label{app:transx:chars}

\textbf{Algorithmic instantiation} 
characterizes the mechanism within the reinforcement learning process that alters the typical greedy mechanisms.
Fundamentally the reinforcement learning process can be thought of as cycle with two directions: ``forward,'' where the agent interacts with the environment, receives reward, and collects samples for learning, and ``backward,'' where the agent's models are updated according to the update function based on reward and a loss is calculated and applied based on the reward and update. 
We consider three means of algorithmic instantiation. 
(1) Exploration-based \textit{environment sampling}.
Different means of sampling non-greedily, for example randomly or for explicity diversity, affect the forward process, making the data distribution more amenable to finding the optimum. 
(2) A modification of the \textit{update function}.
Modifying the reinforcement update process, including but not limited to the loss function, affects the forward process by propagating incentives to the agent that are not greedy reward maximization. 
(3) The addition of an \textit{intrinsic reward}. 
Intrinsic motivation is a quality of exploration methods that incentivize visitation of sub-optimal transitions by reweighting the rewards experienced by the agent at those transitions. Intrinsic reward is unique because it is not definitively part of the forward or backward processes: exploration can just as easily sample states and actions according to an intrinsic reward and alter the agent update with intrinsic reward. 

We do not report many interesting findings on algorithmic instantiation, partially because our results show that in general algorithmic instantiation does not have an outsized impact on the final results.
While NoisyNets with an update function instantiation is consistently high performing in different transfer problems, so is RE3 using an intrinsic reward. 
Moreover, considering a within-group evaluation of all of the intrinsic reward algorithms, we can see that there is a very high variance over average performance across all metrics; ICM consistently performing poorly, RE3 and REVD consistently performing well, and many of the others performing inconsistently with respect to one another.
Maybe most critically, however, we do not think it wise to generalize over conclusions about algorithmic instantiation from this work because of all of the characteristic categories, algorithmic instantiation is the most unbalanced. 
The vast majority of the algorithms evaluated in this paper are intrinsic reward, while only one, NoisyNets, has a modified update function, and even DIAYN, while altering the environment sampling process by a policy conditioned on a random skill vector, still uses an intrinsic reward as well. 
This imbalance is accidental, but not unexpected; the vast majority of modern exploration algorithm that generalize to different problems like we used here use intrinsic reward. 
An important direction of future work will be to construct fair means of comparison with offline algorithms and algorithms only suited for continuous control or discrete control so that more methods like $\epsilon$-greedy~\cite{sutton2018reinforcement}, maximum entropy RL~\cite{hazan2019provably, haarnoja2018soft}, and replay methods like hindsight experience replay~\cite{andrychowicz2017hindsight} can also be compared.

\newpage

\section{Transfer Exploration: Algorithm Descriptions}\label{app:transx:algos}

\textbf{RND:} Random Network Distillation is an exploration algorithm that uses the error of a randomly generated prediction problem as an intrinsic reward for the agent. The prediction problem is set up with two neural networks: a randomly initialized fixed target network and a predictor network that is attempting to approximate the target network. Both networks take an observation and output a $k$-dimensional latent vector. The predictor network is trained on observations collected from the agent using gradient descent to minimize the MSE between the outputs of the two neural networks. This MSE loss is used as the intrinsic reward, which will be higher when the predictor network and target network have not been trained on an observation enough to learn the latent yet.

\textbf{REVD:} Rewarding Episodic Visitation Discrepancy is an exploration method that uses intrinsic rewards to motivate the agent to maximize the discrepancy between the set of states visited in consecutive episodes. The discrepancy between consecutive episodes is measured by an estimate of the Renyi divergence using samples from the two episodes. The intrinsic reward is calculated by using the term in the divergence estimate that has to do with the current state, incentivising the agent to visit states that will increase the divergence estimate between the current episode and the previous one.

\textbf{RE3:} Random Encoders for Efficient Exploration is an exploration method that sets the intrinsic reward to an estimate of state entropy. To estimate state entropy, the method applies a k-nearest neighbor entropy estimator in a low-dimensional space the observations are mapped to using a randomly initialized fixed convolutional encoder. The encoder does not need to be trained and instead relies on the convolutional structure of the network, making the algorithm computationally efficient.

\textbf{RIDE:} Rewarding Impact-Driven Exploration is an exploration method that uses intrinsic rewards to incentivize the agent to take actions that lead to large changes in a learned state representation. The learned state representation comes from an encoder that allows for learning of both the forward and inverse models (taken from ICM). The learning problems the state representation is used for only incentivizes the encoder to retain features of the environment that are influenceable by the agent’s actions. Thus, the intrinsic reward is defined as the difference in said state representation, allowing the agent to experience a diverse set of states.

\textbf{ICM:} Intrinsic Curiosity Module is an exploration method that uses the prediction error of a forward model that acts on state embeddings as the intrinsic reward. The state embeddings are learned by using these embeddings to learn an inverse model to predict the action that takes a state embedding to the state embedding in the next time step. These state embeddings are learned to only contain information relevant to the inverse model, effectively solving the noisy-tv problem. The prediction error of the forward model as an intrinsic reward motivates the agent to explore states that it has a poor estimate of the forward dynamics, which should correlate with states the agent has observed less.

\textbf{NGU:} Never Give Up is an exploration algorithm that constructs an intrinsic reward to strongly discourage revisiting the same state more than once within an episode and discourage visiting states that have been visited many times before. These goals are achieved by an episodic novelty module and a life-long novelty module respectively. These use the embedding networks trained in the same manner as ICM to generate a meaningful lower dimensional state representation. The episodic novelty module uses episodic memory and a k-nearest neighbors pseudo-count method to calculate the intrinsic reward. The life-long novelty module uses the same method as RND. Then these two values are combined using multiplicative modulation for the final intrinsic reward.

\textbf{NoisyNets:} Noisy Networks is an exploration algorithm that applies parametric noise to the weights to introduce stochasticity in the agent's policy. This method adds very little overhead since all it requires is a few extra noise parameters in a few layers of the network. This added stochasticity in the weights propagates to the agent's policy to lead to the agent exploring more unknown states instead of only acting greedily.

\textbf{GIRL:} Generative Intrinsic Reward Learning is an exploration algorithm that motivates the agent to visit areas in which a separate model attempting to model the conditional state distribution performs poorly. The method does this by adding an intrinsic reward of the reconstruction error of each state to the extrinsic reward from the task. The model used to model the state distribution is a conditional VAE conditioned on the previous state and a latent variable.

\textbf{RISE:} Renyi State Entropy Maximization is an exploration algorithm that uses intrinsic rewards to maximize the estimate of intra episode Renyi state entropy. This estimate is calculated on latent embeddings of the states within an episode, where the latents are taken from a VAE trained to reconstruct the states. Further, the algorithm automatically searches the different possibilities for the value of k used in the KNN for the Renyi state entropy estimation that guarantees estimation accuracy. Lastly, RISE uses the distance between each state and its k-nearest neighbors as an estimate for entropy and sets the intrinsic reward to this value. The goal of this reward is to motivate the agent to visit a diverse set of states that increases the entropy of the agent’s state visitations. This method is computationally efficient and does not require any additional memory or networks to backpropagate through.

\textbf{DIAYN:} Diversity Is All You Need is an exploration pre-training method that learns a skill-conditioned policy with the goal to produce diverse skills. This is done by setting the reward to something correlated with the performance of a discriminator model that attempts to predict the skill by using the current state as input. Each episode a new skill is sampled for the policy to use, and the discriminator must attempt to predict the skill. Theoretically, this should lead to the policy attempting to make the job of the discriminator as easy as possible by creating diverse skills. Note that in the original paper this reward and skill-conditioned policy was used before any task reward was introduced. Then, these diverse skills were used to learn a task. However, in our work, we adapt DIAYN to be an online algorithm where this reward is trained simultaneously with the task reward. This motivates the agent to both solve the task while keeping the discriminator's job easy by ensuring different skills cover different areas of the state space. This online adaptation of DIAYN works as a traditional exploration algorithm by motivating the agent to take diverse paths throughout training by sampling different diverse skills to use each episode.

\subsection{A note on ``online'' DIAYN}

The effectiveness of explicit diversity and stochasticity methods is consistent throughout our results; however, this does not mean that adding diversity or stochasticity to any algorithm in any way will guarantee improvement to that algorithm's efficiency in novelty adaptation. 
The fundamental design of an algorithm to succeed in a specific RL problem, such as online task transfer, is as important as the selection of exploration principle and instantiation. For example, online DIAYN has average efficiency in both pre and post-novelty for all tasks we tested it on. However, based on the fact that it blends stochasticity with diverse skills could be interpreted to mean that it ought to have performed better post-novelty. In reality, DIAYN's absence of better performance is more likely due to its implementation; as an algorithm originally designed for reward-free pretraining, naive conversion to an online algorithm, while consistent with the original work and able to learn, is a handicap that cannot be solely compensated for by the potential of its exploration approach. A more transfer-appropriate version of DIAYN—as with all of these algorithms—can be designed from scratch and would likely outperform even the best exploration method investigated here. However, this level of algorithmic design ought to be carefully done with the learning problem in mind and is beyond the scope of this work.

\newpage

\section{Transfer Exploration: Additional Results}
\label{app:transx:results}

\begin{table}[ht]
    \centering
    \footnotesize
    \begin{tabular}{|c|c|c|c|c|c|}
        \hline
          & \multicolumn{5}{|c|}{Convergence Efficiency $\downarrow$} \\ 
         \cline{2-6} 
        Exploration & DoorKeyChange & LavaNotSafe & LavaProof & CrossingBarrier & ThighIncrease \\ 
        Algorithm & ($10^{6}$) & ($10^{5}$) & ($10^{6}$) & ($10^{5}$) & ($10^{6}$) \\ \hline 
        \hline
        None (PPO) & 2.56 $\pm$ 0.584 & 0.707 $\pm$ 0.35 & 1.7 $\pm$ 0.683 & 5.43 $\pm$ 1.69 & 7.99 $\pm$ 1.09 \\
        \hline
        NoisyNets & 2.45 $\pm$ 0.908 & 1.02 $\pm$ 0.911 & 1.31 $\pm$ 1.14 & 4.92 $\pm$ 2.13 & 7.17 $\pm$ 1.72 \\
        ICM & 2.12 $\pm$ 0.595 & 0.604 $\pm$ 0.0966 & 1.8 $\pm$ 1.46 & 4.66 $\pm$ 1.14 & 7.34 $\pm$ 1.02 \\
        DIAYN & 2.19 $\pm$ 0.808 & 0.707 $\pm$ 0.265 & 3.44 $\pm$ 1.57 & 5.47 $\pm$ 2.37 & \textbf{6.87 $\pm$ 2.41} \\
        RND & 2.41 $\pm$ 0.956 & 0.635 $\pm$ 0.0893 & 0.976 $\pm$ 0.803 & 5.11 $\pm$ 0.95 & 7.5 $\pm$ 2.04 \\
        NGU & 2.14 $\pm$ 0.289 & 0.768 $\pm$ 0.291 & 2.34 $\pm$ 3.38 & 5.43 $\pm$ 2.03 & 7.72 $\pm$ 1.52 \\
        RIDE & 2.39 $\pm$ 0.975 & \textbf{0.563 $\pm$ 0.0687} & \textbf{0.73 $\pm$ 0.293} & 5.65 $\pm$ 2.11 & 8.24 $\pm$ 1.24 \\
        GIRL & 2.4 $\pm$ 0.855 & 0.676 $\pm$ 0.173 & 2.43 $\pm$ 1.69 & 4.63 $\pm$ 0.979 & 7.61 $\pm$ 1.99 \\
        RE3 & 2.14 $\pm$ 0.616 & 0.604 $\pm$ 0.107 & 1.86 $\pm$ 0.669 & 5.42 $\pm$ 1.37 & 7.78 $\pm$ 0.642 \\
        RISE & 2.32 $\pm$ 0.764 & 0.614 $\pm$ 0.145 & 3.14 $\pm$ 1.89 & \textbf{4.29 $\pm$ 0.788} & 8.55 $\pm$ 0.441 \\
        REVD & \textbf{2.12 $\pm$ 0.891} & 0.635 $\pm$ 0.188 & 1.72 $\pm$ 1.66 & 4.8 $\pm$ 1.12 & 8.73 $\pm$ 0.934 \\
        \hline
    \end{tabular}
    \caption{This table shows the convergence efficiency on the pre-novelty task. It is computed by calculating the number of steps from the start of training until convergence on the first task. Thus, lower numbers are better here. Only runs that converged on the first task are taken into account for this metric.}
    \label{tab:transx:convergence_efficiency}
\end{table}

\begin{table}[ht]
    \centering
    \footnotesize
    \begin{tabular}{|c|c|c|c|c|c|}
        \hline
         & \multicolumn{5}{|c|}{Adaptive Freq $\uparrow$} \\ \cline{2-6} 
        Exploration & DoorKeyChange & LavaNotSafe & LavaProof & CrossingBarrier & ThighIncrease \\ 
        Algorithm &  & ($10^{-1}$) &  &  &  \\ \hline 

        \hline
        None (PPO) & \textbf{1.0 $\pm$ 0.0} & 6.0 $\pm$ 4.9 & \textbf{1.0 $\pm$ 0.0} & \textbf{1.0 $\pm$ 0.0} & \textbf{1.0 $\pm$ 0.0} \\
        \hline
        NoisyNets & 0.889 $\pm$ 0.314 & \textbf{8.0 $\pm$ 4.0} & 0.714 $\pm$ 0.452 & \textbf{1.0 $\pm$ 0.0} & \textbf{1.0 $\pm$ 0.0} \\
        ICM & 0.889 $\pm$ 0.314 & 3.0 $\pm$ 4.58 & 0.875 $\pm$ 0.331 & \textbf{1.0 $\pm$ 0.0} & \textbf{1.0 $\pm$ 0.0} \\
        DIAYN & \textbf{1.0 $\pm$ 0.0} & 3.0 $\pm$ 4.58 & \textbf{1.0 $\pm$ 0.0} & \textbf{1.0 $\pm$ 0.0} & \textbf{1.0 $\pm$ 0.0} \\
        RND & \textbf{1.0 $\pm$ 0.0} & 3.0 $\pm$ 4.58 & 0.714 $\pm$ 0.452 & \textbf{1.0 $\pm$ 0.0} & \textbf{1.0 $\pm$ 0.0} \\
        NGU & \textbf{1.0 $\pm$ 0.0} & 4.0 $\pm$ 4.9 & \textbf{1.0 $\pm$ 0.0} & 0.9 $\pm$ 0.3 & \textbf{1.0 $\pm$ 0.0} \\
        RIDE & \textbf{1.0 $\pm$ 0.0} & 6.0 $\pm$ 4.9 & \textbf{1.0 $\pm$ 0.0} & \textbf{1.0 $\pm$ 0.0} & \textbf{1.0 $\pm$ 0.0} \\
        GIRL & \textbf{1.0 $\pm$ 0.0} & 2.0 $\pm$ 4.0 & \textbf{1.0 $\pm$ 0.0} & \textbf{1.0 $\pm$ 0.0} & \textbf{1.0 $\pm$ 0.0} \\
        RE3 & \textbf{1.0 $\pm$ 0.0} & 2.0 $\pm$ 4.0 & \textbf{1.0 $\pm$ 0.0} & \textbf{1.0 $\pm$ 0.0} & \textbf{1.0 $\pm$ 0.0} \\
        RISE & 0.857 $\pm$ 0.35 & 3.0 $\pm$ 4.58 & \textbf{1.0 $\pm$ 0.0} & \textbf{1.0 $\pm$ 0.0} & \textbf{1.0 $\pm$ 0.0} \\
        REVD & \textbf{1.0 $\pm$ 0.0} & 5.0 $\pm$ 5.0 & \textbf{1.0 $\pm$ 0.0} & \textbf{1.0 $\pm$ 0.0} & \textbf{1.0 $\pm$ 0.0} \\
        \hline
    \end{tabular}
    \caption{This is the frequency that the agent converges on the second task using this exploration algorithm conditioned on the fast it converged on the first task. Higher numbers are better.}
    \label{tab:transx:adaptive_freq}
\end{table}

\newpage

\section{Transfer Exploration: Additional Analysis}
\label{app:transx:analysis}


\subsection{Difference in course-target task performance between continuous and discrete action spaces}

Beyond exploration characteristics, one of the biggest differences between novelty adaptation in discrete vs continuous control is the loose correlation between pre- and post-novelty performance. 
While the Tr-AUC metric is motivated by the presumption that poor performance on the source task will lead to deceptively good performance on the target task, we find in our continuous control environment that the opposite is true.
Based on this finding, we suggest that the fundamental knowledge of continuous control is perhaps more inherently transferable. 
Adapting to suddenly long legs forces the agents to forget some of their prior policies, however much of the challenge in continuous control is learning that relationships between action and effect is broadly applicable; moving one joint with an effort of $E$ will be more similar to moving a different joint with the same effort than comparing any two actions in discrete environments. 
Thus, the relationship between action and exploration, as we saw in the characteristic analysis, seems to be far more tightly bound for continuous control than discrete control.
As a result, inductive biases from separate objectives and controllability assumptions are less problematic, and characteristics that remove time dependence and favor knowledge preservation are more useful. 

\subsection{Shortcut novelties}

We also examined the shortcut LavaProof novelty as compared to the other novelties, and we see some interesting behavior very specific to the notion of a shortcut. 
As identified in prior work, shortcuts can be notoriously hard exploration problems for transfer learning because the novelty is injected and the learner's prior optimum is undisturbed. 
As we have noted, if we used exploration decay in our algorithm implementations, as is common in single-task RL, there is a chance most or even all of the algorithms in this study would ignore the new shortcut and continue with the sub-optimal solution. 
Even without exploration decay, NGU, GIRL, and ICM all fail to consistently identify the shortcut over the safe lava in spite of learning how to safely navigate around it. 
Atypically, NoisyNets also performed poorly and was unable to consistently find the novelty. 
Of those that performed well, in addition to RE3, DIAYN and RIDE performed unusually well. 
These observations together serve as strong evidence that the main difference in characteristic importance for shortcuts is an even stronger emphasis on the importance of explicit diversity. 
For a shortcut, the critical steps are to (1) identify that a shortcut exists, and (2) consider it worth exploring. 
Although intuitively the stochastic nature of NoisyNets may thrive at shortcut identification, it is less likely that a time-independent method like NoisyNets would be able to value exploring something just because it was novel. 
In this way, the lack of temporal locality in NoisyNets overcomes its potential for exploring the novelty.
Interestingly, the reverse happens for DIAYN. 
DIAYN's core motivation is to learn separable distinguishable policy skills, which for a single task learning problem becomes progressively harder as the policy converges. 
When a shortcut is identified, there is a novel opportunity for DIAYN to suddenly learn more diverse separable skills. 
As a result, the DIAYN's specific implementation of explicit diversity is able to overcome its time-independent exploration nature.

\newpage

\section{Transfer Exploration: Implementation Details}
\label{app:transx:impl}
\subsection{Hyperparameters}
\label{app:transx:impl:hyperparams}


We sweeped through the hyperparameter configurations for each exploration algorithm using Bayesian hyperparameter optimization. We ran a minimum of 10 hyperparameter configurations (using more for the algorithms with many parameters), each with six runs (three seeds on MiniGrid-DoorKey-8x8-v0 and three seeds on MiniGrid-SimpleCrossingS9N2-v0), for each algorithm. Each successive configuration was calculated using the weights and biases Bayesian sweep method within reasonable preset range around parameters pulled from prior work. The metric optimized for to minimize the average (over the 6 runs) number of steps needed for the \href{https://stable-baselines3.readthedocs.io/en/master/guide/callbacks.html#stoptrainingonrewardthreshold}{StopTrainingOnRewardThreshold} callback from \href{https://stable-baselines3.readthedocs.io/en/master/}{stable-baselines3} to stop the run with a reward threshold set to $0.35$ (capped at 3M steps). Once the sweeps were finished we chose reasonable hyperparameters that followed the trends of the other runs in the sweep to ensure the chosen parameter configuration was not just an outlier.

Here is a table consisting of the ranges of hyperparameters we sweeped through and our final chosen value for them based on the (limited) number of runs we used. The distribution type column refers to the distribution parameter provided to the \href{https://wandb.ai/}{wandb} sweep agent. For specifics about what each parameter does see the individual papers or the implementations in \href{https://github.com/balloch/rl-exploration-transfer/tree/noisy-net-implementation/rlexplore}{our codebase}. Note that latent\_dim, batch\_size, and learning\_rate parameters refer to networks trained specifically for exploration and have nothing to do with the parameters used for policy training.

\begin{table}[ht]
    \centering
    \footnotesize
    \begin{tabular}{|c|c|c|c|c|}
        \hline
        \textbf{Algorithm} & \textbf{Parameter Name} & \textbf{Distribution Type} & \textbf{Range} & \textbf{Final Value} \\
        \hline
        \multirow{1}{*}{\textbf{PPO}} & learning\_rate & q\_uniform & [0.0003, 0.0008] & 0.00075 \\
        \hline
        \multirow{2}{*}{\textbf{RE3}} & beta & q\_log\_uniform\_values & [0.00001, 0.1] & 0.01 \\
        & latent\_dim & categorical & [16, 32, 64, 128, 256] & 64 \\
        \hline
        \multirow{2}{*}{\textbf{RIDE}} & beta & q\_log\_uniform\_values & [0.00001, 0.1] & 0.001 \\
        & latent\_dim & categorical & [16, 32, 64, 128, 256] & 128 \\
        \hline
        \multirow{2}{*}{\textbf{RISE}} & beta & q\_log\_uniform\_values & [0.00001, 0.1] & 0.002 \\
        & latent\_dim & categorical & [16, 32, 64, 128, 256] & 64 \\
        \hline
        \multirow{5}{*}{\textbf{RND}} & beta & q\_log\_uniform\_values & [0.00001, 0.1] & 0.002 \\
        & learning\_rate & q\_log\_uniform\_values & [0.0001, 0.01] & 0.0003 \\
        & batch\_size & categorical & [16, 32, 64] & 64 \\
        & latent\_dim & categorical & [16, 32, 64, 128, 256] & 128 \\
        \hline
        \multirow{1}{*}{\textbf{Noisy Nets}} & num\_noisy\_layers & categorical & [1, 2, 3] & 2 \\
        \hline
        \multirow{5}{*}{\textbf{NGU}} & beta & q\_log\_uniform\_values & [0.0001, 0.5] & 0.0005 \\
        & learning\_rate & q\_log\_uniform\_values & [0.0001, 0.01] & 0.0006 \\
        & batch\_size & categorical & [16, 32, 64] & 64 \\
        & latent\_dim & categorical & [16, 32, 64, 128, 256] & 128 \\
        \hline
        \multirow{3}{*}{\textbf{ICM}} & beta & q\_log\_uniform\_values & [0.00001, 0.1] & 0.0003 \\
        & learning\_rate & q\_log\_uniform\_values & [0.0001, 0.01] & 0.0003 \\
        & batch\_size & categorical & [16, 32, 64] & 64 \\
        \hline
        \multirow{4}{*}{\textbf{GIRL}} & beta & q\_log\_uniform\_values & [0.00001, 0.1] & 0.0005 \\
        & learning\_rate & q\_log\_uniform\_values & [0.0001, 0.01] & 0.002 \\
        & lambda & q\_log\_uniform\_values & [0.001, 0.1] & 0.05 \\
        & latent\_dim & categorical & [32, 64, 128] & 64 \\
        \hline
        \multirow{3}{*}{\textbf{REVD}} & beta & q\_log\_uniform\_values & [0.00001, 0.1] & 0.00005 \\
        & latent\_dim & categorical & [16, 32, 64, 128, 256] & 64 \\
        \hline
        \multirow{3}{*}{\textbf{RIDE}} & beta & q\_log\_uniform\_values & [0.00001, 0.1] & 0.001 \\
        & latent\_dim & categorical & [16, 32, 64, 128, 256] & 128 \\
        \hline
        \multirow{3}{*}{\textbf{RISE}} & beta & q\_log\_uniform\_values & [0.00001, 0.1] & 0.002 \\
        & latent\_dim & categorical & [16, 32, 64, 128, 256] & 64 \\
        \hline
    \end{tabular}
    \caption{Hyperparameter Sweeps for Exploration Algorithms.}
    \label{tab:transx:hyperparameters}
\end{table}

For the continuous control task (Walker), we ran a targeted sweep on CartPole, mainly tuning parameters that were important to our results such as beta and other exploration algorithm specific parameters. We used prior work, results from our MiniGrid sweep, and other heuristics to estimate the ranges to sweep for different parameters. The main parameters that changed relative to the table above were the beta's for each algorithm as the reward scale is very different in walker as opposed to any MiniGrid tasks.

\subsection{Experimental Setup}
\label{app:transx:impl:setup}

For a valid comparison, all the experiments were run using PPO with the same PPO hyperparameters (listed below). Further, the experiments use the \href{https://stable-baselines3.readthedocs.io/en/master/guide/custom_policy.html#default-network-architecture}{default MLP policy} network shapes from the stable-baselines3 PPO class for the experiments and any hyperparameters not specified below were left as default. 

\begin{table}[ht]
    \centering
    \begin{tabular}{|c|c|}
        \hline
        \textbf{Parameter} & \textbf{Value} \\
        \hline
        learning\_rate & 0.00075 \\
        n\_steps & 2048 \\
        batch\_size & 256 \\
        n\_epochs & 4 \\
        gamma & 0.99 \\
        gae\_lambda & 0.95 \\
        clip\_range & 0.2 \\
        ent\_coef & 0.01 \\
        vf\_coef & 0.5 \\
        max\_grad\_norm & 0.5 \\
        \hline
    \end{tabular}
    \caption{PPO Configuration}
    \label{tab:transx:ppo_rl_alg_kwargs}
\end{table}

Each experiment on MiniGrid used 10 seeds with 5 parallel environments each to ensure reliable results, logging all results to wandb for future aggregation and analysis.

Each experiment on Walker used 5 seeds with 10 parallel environments.

For each of the environments, we ran the experiments with a number of steps that led to a high convergence rate with the implemented algorithms so fair comparisons between algorithms could be used on the task two results.

\begin{table}[ht]
    \centering
    \footnotesize
    \begin{tabular}{|c|c|c|c|}
        \hline
        \textbf{Environment Name} & \textbf{Pre Novelty Steps} & \textbf{Post Novelty Steps} & \textbf{MiniGrid Size} \\
        \hline
        \textbf{door\_key\_change} & 5M & 3M & 8x8 \\
        \textbf{simple\_to\_lava\_crossing} & 2M & 3M & 9x9 \\
        \textbf{lava\_maze\_safe\_to\_hurt} & 500,000 & 5M & 8x8 \\
        \textbf{lava\_maze\_hurt\_to\_safe} & 5M & 2M & 8x8 \\
        \textbf{walker\_thigh\_length} & 10M & 10M & N/A \\
        \hline
    \end{tabular}
    \caption{Environment Details}
    \label{tab:transx:environment_details}
\end{table}

We used a few observation wrappers on the environments in the experiment to set the observation space to be the flattened observed image (to work with simple MLP policies).

\newpage
\section{Extended Related Work}
\label{app:related}

\subsection{Plasticity and Replay Ratios in Deep Reinforcement learning}
\label{app:related:plasticity}

Monolithic prior knowledge can be helpful but is not always a good ``warm start'' for learning in deep neural networks~\cite{ash2020warm}.  
Parameterized models like neural networks have a tendency to be overly-influenced by the early training process, often described as 
\textit{plasticity loss}~\cite{achille2017critical}. 
This is particularly problematic in reinforcement learning, and doubly so in OTTA, where the data distribution shifts throughout the learning process. 
Dohare et al. \cite{dohare2023loss}
shows that in deep continual learning, plasticity correlates with low weight magnitude and density---i.e. avoiding ``dead neurons''--- and that L2-regularization and weight randomization techniques like shrink-and-perturb~\cite{ash2020warm} do much to mitigate plasticity loss. 
Recent research describes this phenomenon in reinforcement learning as 
\textit{primacy bias}~\cite{nikishin2022primacy}.  
Even in the typical formulation of reinforcement learning with a stationary MDP, the evolving data distribution that comes from sampling using an evolving policy makes it a more challenging learning environment than supervised learning~\cite{pmlr-v162-fan22c}. 
For effective online test time adaptation in reinforcement learning, primacy bias and loss of plasticity take on even greater importance, as it is clear the models must update to succeed in the new environment, but it is not clear what old model parameters or data should be preserved. 

One simple solution to the issues of plasticity loss and primacy bias---especially if you have the ability to learn from a replay buffer---is simply periodically resetting the weights of the neural network.
In the continual learning setting, partially resetting the network parameters has been shown to consistently improve learning performance~\cite{ash2020warm}, and in reinforcement learning partial and hard resets have been used to both improve sample efficiency and final performance~\cite{ash2020warm,dohare2021continual,nikishin2022primacy,doro2023sampleefficient}. 
We take advantage of research around parameter resetting to apply notions of increased plasticity to our OTTA setting.  

\subsection{Offline-to-Online Reinforcement Learning}
\label{app:related:offline}

The ability of deep neural networks to scale learning performance with data is one of the reasons why pre-training with non-task data is so effective~\cite{erhan2010pretraining}. 
Motivated by the widespread interest in using deep reinforcement learning for agents learning from visual observations, the most straightforward way to improve agent performance using offline data is to pre-train only the observation encoder~\cite{stooke2021decoupling,parisi2022pretrain}. 
While pre-training the visual encoder avoids many of the challenges and complexities of sequential decision making, the impact on agent performance is unreliable~\cite{shah2021rrl}. 

Offline pre-training reinforcement learning agent behavior can be largely divided into three types of approaches: imitation learning, task-agnostic exploration, or offline reinforcement learning. 
Imitation learning-based pre-training, using techniques like behavior cloning to establish a baseline policy, is often used to improve the sample efficiency of learning policies for complex problem spaces, such as the manipulation of arbitrary objects~\cite{kalashnikov2018scalable}. 
As a result, accumulating datasets of internet-scale demonstrations with different tasks shows great promise for GPT-like policy foundation models~\cite{openxembodiment2024}.
Video Pre-Training (VPT)~\cite{baker2022video} demonstrates that an inverse control-prediction model trained on a small set of demonstrations can be used to supervise imitation learning on a significantly larger dataset of action-free sequence data.
Imitation learning-based pre-training, however, makes the assumption that the state action space is sufficiently smooth such that expert demonstrations cover all situations an agent will encounter, which is not always true. 

Task agnostic exploration assumes a phase during which the reinforcement learning agent is not necessarily given access to the task on which it will be evaluated, but is given access to the environment for interaction. 
Also referred to in off-policy reinforcement learning as ``warm starting,'' unsupervised exploration pre-training usually attempts to cover the state-action space by exploring the space by maximizing surrogate objectives~\cite{andrychowicz2017hindsight,liu2021apt} 
(for a more detailed overview of exploration techniques see Chapters~\ref{chapt:background} and~\ref{chapt:transx}).
The most prominent examples of this employed in pre-training Dreamer are Plan2Explore~\cite{sekar2020planning} in which the policy maximizes uncertainty as the reward, and LEXA~\cite{mendonca2021lexa} which (provided a pre-trained world model) trains an exploration policy that learns to reach novel states. 
However, pre-training with task-agnostic exploration rarely reduces the total amount of environment interaction needed to learn a policy, and as such does not overall improve the interactive sample efficiency of reinforcement learning agents.
Moreover, if the environment is available for pre-training, there are many situations in which environment access is limited or too slow and difficult for interactive pre-training.

For offline reinforcement learning-based pre-training, offline model-free reinforcement learning has seen more research interest than model based methods. 
Two popular methods that address issues inherent in offline algorithms for fine tuning are Conservative Q-Learning (CQL)~\cite{kumar2020cql}, which addresses the distributional shift that occurs when an offline policy encounters novel states during online adaptation, and Batch-Constrained Q-learning (BCQ)~\cite{fujimoto2019bcq}, which constrains the policy in online training to actions close to those in the offline dataset. 
Fewer efforts have been made to formulate offline model-based reinforcement learning approaches~\cite{he2023surveyofflinemodelbasedreinforcement}. 
Notably, Model-Based Offline Reinforcement Learning (MOReL)~\cite{kidambi2020morel} learns a pessimistic MDP to provide lower-bound performance guarantees. 
Although all these approaches are theoretically sound, in practice they can be overly conservative and use suboptimal heuristics and hyperparameters~\cite{lu2022revisiting}, limiting performance on some tasks, especially when fine tuned. 

Fine tuning offline pre-trained RL agents is sufficiently challenging~\cite{lee2022offline,wolczyk24finetuning} that a recent body of work has focused on improving the two-step process of ``offline-to-online reinforcement learning.'' 
The difficulties that arise during the offline-to-online conversion are usually attributed to the sudden shift from the offline state-action distribution to the online state-action distribution (which can lead to bootstrapping errors in fine tuning), overfitting to on-policy demonstrations leading to problems reasoning over off-policy dynamics, and the impact of non-stationarity of rewards on the training of the value function~\cite{levine2020offline}.
One of the first efforts in offline-to-online RL is Advantage-Weighted Actor Critic (AWAC)~\cite{nair2021awac}, which effectively separates learning into supervised learning of the policy and reinforcement learning of the critic to mitigate the negative effects of shifting from the offline data distribution to the online data distribution.
Building off of the work by AWAC, methods reduce the impact of the offline-to-online distribution shift by adding uncertainty to reduce policy exploitation~\cite{rafailov23moto}, smoothing the transition using a blend of offline and online data~\cite{lee2022offline,mao2022moore}, and by adding constraints such as behavior cloning losses~\cite{wang2024o2ac} that can be gradually reduced over training.

There have been a small number of attempts to reformulate model-based RL methods like Dreamer to work both offline and online. 
The simplest method of doing this is to follow the training process of Ha and Schmidhuber~\cite{ha2018recurrent}, where instead of training Dreamer in an ``interleaved'' fashion as designed, where for every step the agent takes in the environment the world model \textit{and} the agent are updated, the agent is trained in ``phases.'' First, the world model is trained on interaction data, and then the world model is frozen and behavior training phase begins~\cite{lu2023challenges, wang2024making}. 
By splitting the training into phases, each phase can be executed offline. 
However, as highlighted by prior work on exploration for Dreamer agents~\cite{sekar2020planning,mendonca2021lexa}, the phase-based approach only works for tasks where exploration is trivial or unimportant. 
Moreover, this makes the assumption that the world model does not require the agent learning process to learn the task effectively, which is in some ways the assumption of imitation learning-based pre-training, but applied to an ``expert'' world model instead of policy. 
One work that pre-trains Dreamer using interaction data in the intended interleaved fashion is the APV method~\cite{seo2022pretraining}, which still modifies the process with an additional module to learn ``action-free'' models before fine tuning online.  
As of this writing, there has been no prior work focused on how to effectively pre-train an unmodified Dreamer agent using end-to-end interleaved training process using both offline and online interaction data.

\subsection{Mechanistic Interpretability}
\label{app:related:mech_int}

This work is similarly motivated to the parallel research area of \textit{mechanistic interpretability}~\cite{nanda_2022,elhage2022superposition} (MI), which studies the interpretation of neural network behavior by constructing an interpretation of a neural network based on an interpretation the internal structures of the network~\cite{olah2020zoomcircuits}. 
This stands in contrast to ``black-box''~\cite{holzinger2022explainable} interpretability approaches such as saliency maps of inputs from outputs~\cite{adebayo2018sanity,jain2019attention,wiegreffe2019attention}, which attempt to produce explanations of relationships between the inputs and outputs without considering the networks' internal mechanisms. 
Interpreting neural networks using codebooks falls in between these methods (in what is sometimes called ``white-box''~\cite{holzinger2022explainable,rauker2023toward} interpretability),  where (as in MI) the internal structures of the neural networks are constrained or investigated to establish what activations of the network are representing while using tools including saliency maps to convey feature importance. 
Indeed: codebooks can be viewed as an ``overcomplete basis'' of codes---a key underlying principles of mechanistic interpretability---over the distribution of latents with concept entanglement aligning with MI notions of superposition~\cite{elhage2022superposition}. 
Overcomplete bases,  while not true bases since they do not exhibit orthogonality, are still powerful concepts as the Johnson-Lindenstrauss (JL) lemma provides mathematical guarantees on the ``near orthogonality'' and therefore representational capacity of sets of vectors larger than the dimensionality of a space~\cite{johnson1984extensions,nanda_2022}. 
Superposition~\cite{elhage2022superposition} is a complementary idea that, given a latent feature space represented by a set of $d$-dimensional vectors $v$, optimization tries to use these vectors  to represent more features than they have the dimensional capacity. 
This leads to phenomena such as representing concepts with groups of features, or individual features alternately representing more than one concept depending on the input, both of which we observe in our work~\cite{elhage2022superposition,nanda_2022}. 
Although MI benefits from strict definitions of tasks, models, and theoretical components such as superposition and the JL lemma to interpret the entire internal structure of neural networks, white-box methods like codebook interpretability provide an attractive utilitarian approach to interpretability. 
By using robust concepts like vector representation and superposition to interpret the behavior of a limited internal part of the network white-box feature interpretation such as this paper's approach can be applied to evaluate the interpretability of almost any network structure without the need to understand all of a network's internal structures.

\newpage
\section{Concept Bottleneck World Models: Additional Methods}
\label{app:cbwm:methods} 

Revisiting the discussion from Chapter~\ref{chapt:background}, the combination of characteristics that most distinguishes the Dreamer~\cite{hafner2019dream} family of model-based reinforcement learning (MBRL) algorithms from other MBRL algorithms is:
\begin{enumerate}
    \item The RSSM world model architecture (Equation~\ref{eq:dreamer:wm}) that models the transition function as a recurrent variational state space,
    \item The use of observation reconstruction as a primary loss in learning the world model,
    \item The formulation of the behavior policy as a latent actor-critic learning only from rollouts in the world model latent space,
    \item Learning both the actor-critic and world model in an interleaved fashion as opposed to in separate phases. 
\end{enumerate}

Maintaining all of these characteristics for both pre-training and fine tuning is challenging and begs the question: ``Why \textit{not} use other methods?'' 
An effort is made to preserve this formulation because recent work postulates that the Dreamer formulation models adaptable human cognition more closely than other approaches~\cite{pearson2019human,mattar2022planning,kudithipudi2022biological,deperrois2024learning}.
Moreover, with the ability of the actor-critic to be trained solely in the ``imagination'' of the world model, Dreamer is well-suited to solving the problem of learning without interacting.
Provided a reasonable approximation of the true reward model and good coverage of the state-actions space, behavior learning theoretically does not require any environment interaction as the actor-critic learns entirely in the world model's embedding space.

Lastly, the Dreamer architecture is difficult to train piece-wise because the input to the policy is an embedded state, and the world model learning depends on agent behavior to explore so as to avoid overfitting the world model to only the dynamics of the optimal path.

\subsection{Impact of Partial Model Transfer on Interleaved Actor-Critic}
\label{app:cbwm:methods:pretrainingtheory}

One of the primary concerns with transferring a pre-trained world model agent is the cascading impact of partial transfers on the actor-critic.
Consider the problem of transferring solely the  world model and not the actor or critic models. 
This is a reasonable approach as prior work in offline-to-online RL caution against transferring overfit actors and critics with poor bootstrapping ability~\cite{wang2024o2ac}, and model-based exploration pre-training works do not transfer the exploration policy when fine tuning~\cite{sekar2020planning,mendonca2021lexa}.  
In actor-critic training, the data used to train the critic is drawn from initial embeddings of real-world states, followed by imagined latent states as generated by the transition predictor, all valued by the rewards from the reward predictor. 
The Dreamer algorithm assumes \textit{tabula rasa} learning, so the distribution of latent states sampled by a random policy $p_{\pi_U}(s)$, with random embeddings and rewards from the initial transition and reward predictors. 
This gives a noisy $\lambda$-error calculation (Equation~\ref{eq:dreamer:tdlambda}), which will produce high gradients that accurately reflect the need for the critic to make large changes to its parameters 

However, assuming instead that the transition and reward prediction models are trained on expert demonstrations, a different outcome emerges. 
When trained transition and reward prediction models are used in policy learning, these models will predict that latent data distributed according to random actions $p_{\pi_U}(s)$ are relatively close in latent space and all have rewards near zero. 
As a result, the critic will predict similar values for all states, leading to a low $\lambda$-error. 
The gradients for much of the early online learning period will be inaccurately small, leading to slower behavior learning than in tabula rasa learning.

\begin{theorem}
Let actor $\pi_{\theta}$ and critic $v_{\psi}$ be randomly initialized models such that their outputs are distributed as $U(s_t)$. If transition predictor $g_\phi^*$ and reward predictor $\mathcal{R}_\phi^*$ are optimized to predict the online task dynamics, then for small $\epsilon$, $\mathcal{L}_0(\psi) - \varepsilon = 0$ and $\mathcal{L}_0(\theta \mid \phi^*) < \mathcal{L}_0(\theta \mid \phi^*)$.
\end{theorem}
\begin{proof}

Let $o_t \sim p_{U}(s)$ be the set of initial observations distributed according to the sequential actions of a random actor and let the embeddings $\{s_1\} e_\phi(o_t)$ serve as initial states.
Batched horizon $H$ length trajectories of states, actions, and rewards,  
$\{s,a,r,s^{\prime}\}_\tau \in B(\theta,\phi), \tau=[t:t+H]$ 
are generated according to learnable models $\pi_{\theta}$, $g_\phi$, and $\mathcal{R}_\phi$.
For \textit{tabula rasa} learning---where by initialization $\pi_{\theta}$ and $g_\phi$ are distributed uniformly in $[-1,1]$, the critic $v_{\psi}$ is a deterministic random projection, and $\mathcal{R}_\phi$ is distributed as a univariate Gaussian with unit variance and $\mu=0.1$---call the data generated by these models $B^0$.
For pre-trained learning---where $g_\phi^*$ and $\mathcal{R}_\phi^*$ are similarly distributed but with optimized parameters---call the data generated by these models $B^{\phi^*}$.


Expanding the terms of the recursion over Equation~\ref{eq:dreamer:tdlambda} and sum from Equation~\ref{eq:dreamer:critic_loss} using the definition of the expected $\lambda$-return, we have: 

\begin{align*}
    V_t^\lambda =& r_t + \hat{\gamma}_t \left((1-\lambda) v_\psi\left(s_{t+1}\right)
    +\lambda V_{t+1}^\lambda\right) \\
    V_{t=H-1}^\lambda =& r_{H-1} + \hat{\gamma}_{H-1} v_{\psi}(s_H)\\
    V_{t=H-2}^\lambda =& r_{H-2} + \hat{\gamma}_{H-2} \left( 
    (1-\lambda) v_{\psi}(s_{H-1}) 
    + \lambda V_{H-1}^\lambda  
    \right) \\
    & ...  \\
    V_{t=2}^\lambda =& r_{2} + \hat{\gamma}_{2} \left( 
    (1-\lambda) v_{\psi}(s_{3}) 
    + \lambda V_{3}^\lambda  
    \right) \\
    V_{t=1}^\lambda =& r_{1} + \hat{\gamma}_{1} \left( 
    (1-\lambda) v_{\psi}(s_{2}) 
    + \lambda V_{2}^\lambda  
    \right) \\
\end{align*}

\noindent First considering $B^0$, given this weighted average---exponential in horizon length---when the  states $s_t$ are distributed uniform randomly the critic as a random projection of that state will result in a diverse set of values $v_{\psi}(s_t)$. Combined with $r_t$ as a stochastic random projection, in the calculation of $\lambda$-error, $v_{\psi}(s_t)$ and $V_t^\lambda$ are not close for any given $t$ leading to high loss. 
On the other hand, if $v_{\psi}(s_H)$ is well-trained, regardless of reward function, $v_{\psi}(s_t)$ and $V_t^\lambda$ will be close, so $\lambda$-error and loss will be low.

However, when considering $B^{\phi^*}$, the states $s_t$ are concentrated as due to the trained transition predictor. 
This is because---especially in continuous control---random actions through a smooth space such as that of the learned transition predictor constitutes a generic random walk, and with distributed actors diffuse Brownian motion. 
As such, given highly similar initial states and a fixed time horizon, in the limit of initial sample size and number of distributed actors the distribution of visited states will follow a compact normal distribution with $\mu=0$ in the transition predictor's embedding space~\cite{arfken2011mathematical} following the equation for the density of particles emanating from a single point:
\begin{equation*}
    p(s_t)_\pi \propto \frac{N}{\sqrt{4\pi H}} \exp\left(-\frac{x^2}{4H}\right)
\end{equation*}
\noindent where, assuming constant diffusivity, $N$ is the number of ``random walks,'' which in our case is a function of number of actors and number of initial states, and $H$ is the horizon.
As our policy has an entropy maximization term and therefore forms a maximal entropy random walk~\cite{duda2012maximal}, this normal assumption forms an upper bound on state distribution density and is a reasonable approximation for small values of $H$.

In addition, $B^{\phi^*}$ has a trained reward function. Even with a dense reward, but especially a with a sparse reward, task design is such that states near the initial position typically have zero reward. 
Returning then to our TD($\lambda$) equations, to consider how the error term changes for a random or optimized critic, we can see that with all of the reward terms near zero, all $\lambda$-error terms for all $t$ become random projections centered on zero, making them zero in expectation. 
When the critic is not randomly initialized, the same phenomenon occurs but with less standard deviation. 
As a result, value loss tends to zero regardless of critic accuracy. 


\end{proof}

\subsection{Transfer Learning and Partial Model Adaptation}
Motivated by this interleaved critic learning issue and the typical issues of distribution shift associated with offline-to-online learning, we conducted systematic experiments on the transferability of various pre-trained weights and submodules from within our world model architecture. 
This is similar to partial model learning in reinforcement learning \cite{precup1997multi,talvitie2008local,khetarpal2021partial,alver2024partial}, where dynamics models are specialized to predict behaviors within localized regions of the state-action space. 
Instead of learning a model specific to a subspace, the experiments in this work reveal the submodules of world model agents that represent the overlapping subspace between the pre-training and fine tuning MDPs.   
Using fully offline pre-trained Dreamer agents we evaluated freezing and reinitialization of different combinations of submodules. 
We focused our weight \textit{freezing} experiments on the key components of the dynamics learning framework: (1) the observation encoder, (2) the recurrent model, (3) the transition prediction model, and (4) the observation prediction model. 
We focused our weight \textit{resetting} experiments on the components of the behavior learning framework: (1) the actor model, (2) the critic model, (3) the transition prediction model, and (4) the discount prediction model. 
The reward prediction model is never frozen or reset because pre-training data was gathered with a sparse reward, while fine tuning used a shaped reward.

\section{Concept Bottleneck World Models: Additional Results}
\label{app:cbwm:results} 

\subsection{Concept Intervention}
\label{app:cbwm:results:intervention} 

One of the most interesting uses of concepts in a bottleneck is the ability to modify downstream outcomes by ``intervention.'' 
Intervention is manually changing a concept to affect change in the downstream task.
We force a scene with positive codes for ``Moka Pot'' to be zero (representing that they are not in the scene), and force the near-zero code ``Pan'' to be one (representing that it is in the scene).
As we can see in Figure~\ref{fig:int}, these changes, respectively, fade the moka pots out and the pan in. 
This shows that the concept bottleneck is forcing the downstream task to utilize concepts to make predictions, rather than ignoring the concepts or using them to represent unpredictable, entangled phenomena (as in non-CBM architectures). 

\begin{figure}
    \centering
    \begin{subfigure}[b]{0.4\textwidth}
        \centering
        \includegraphics[width=\textwidth]{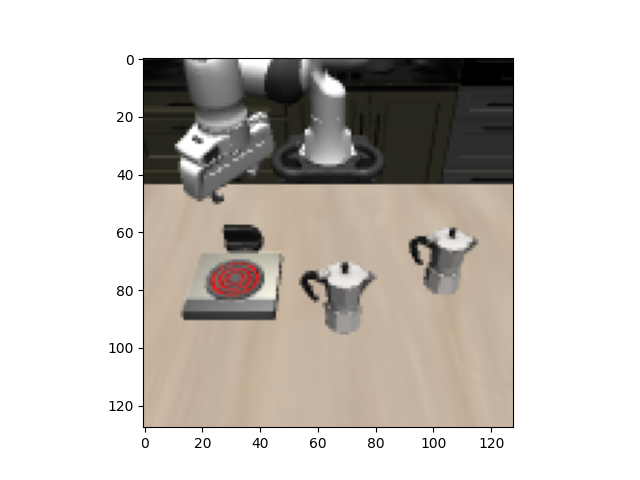}
        \label{fig:subfig1}
    \end{subfigure}
    \hfill
    \begin{subfigure}[b]{0.4\textwidth}
        \centering
        \includegraphics[width=\textwidth]{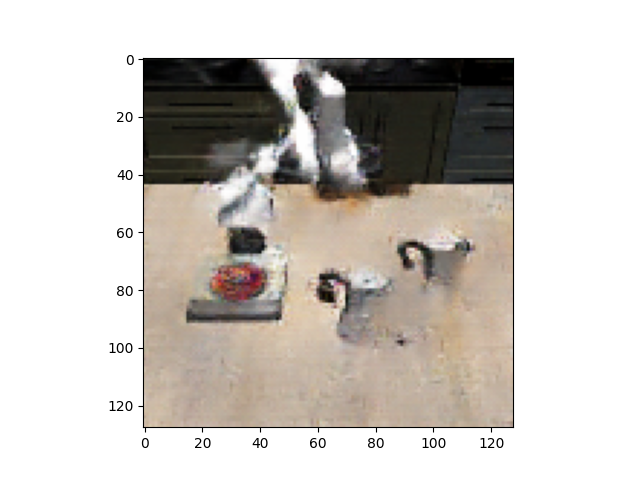}
        \label{fig:subfig2}
    \end{subfigure}
    \hfill
    \begin{subfigure}[b]{0.4\textwidth}
        \centering
        \includegraphics[width=\textwidth]{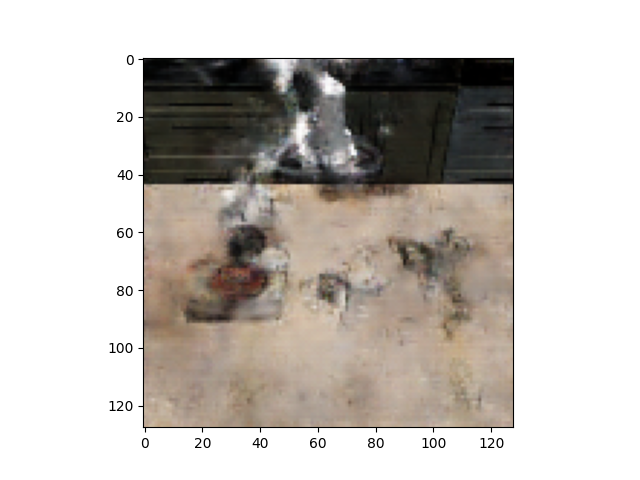}
        \label{fig:subfig3}
    \end{subfigure}
    \begin{subfigure}[b]{0.4\textwidth}
        \centering
        \includegraphics[width=\textwidth]{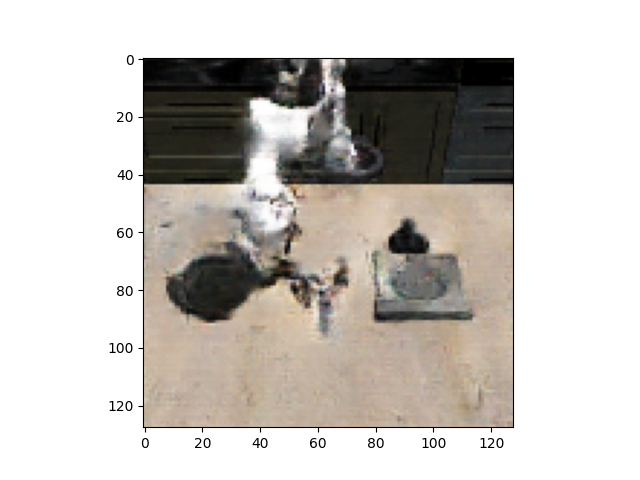}
        \label{fig:subfig4}
    \end{subfigure}
    \caption{In this figure, we see the ground truth observation in (a), followed by the unmodified predicted observation in (b), the moka pots removed in (c), and the pan added in (d).  }
    \label{fig:int}
\end{figure}

\newpage
\section{Concept Bottleneck World Model Tasks}
\label{app:cbwm:tasks}

\begin{figure}[ht]
    \centering
    
    \begin{subfigure}[b]{0.3\textwidth}
        \centering
        \includegraphics[width=\textwidth]{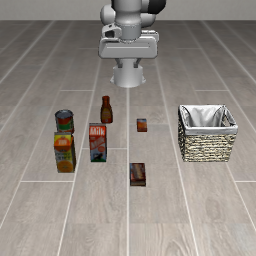}
        \caption{}
        \label{fig:a1}
    \end{subfigure}
    \begin{subfigure}[b]{0.3\textwidth}
        \centering
        \includegraphics[width=\textwidth]{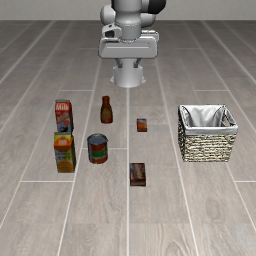}
        \caption{}
        \label{fig:a2}
    \end{subfigure}
    \begin{subfigure}[b]{0.3\textwidth}
        \centering
        \includegraphics[width=\textwidth]{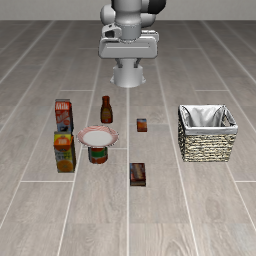}
        \caption{}
        \label{fig:a3}
    \end{subfigure}
    
    \caption*{A)}

    \begin{subfigure}[b]{0.3\textwidth}
        \centering
        \includegraphics[width=\textwidth]{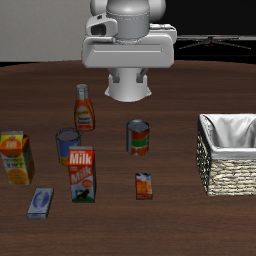}
        \caption{}
        \label{fig:b1}
    \end{subfigure}
    \begin{subfigure}[b]{0.3\textwidth}
        \centering
        \includegraphics[width=\textwidth]{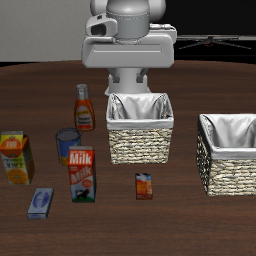}
        \caption{}
        \label{fig:b2}
    \end{subfigure}
    \begin{subfigure}[b]{0.3\textwidth}
        \centering
        \includegraphics[width=\textwidth]{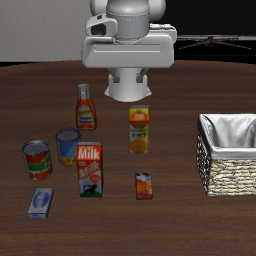}
        \caption{}
        \label{fig:b3}
    \end{subfigure}
    
    \caption*{B)}

    \begin{subfigure}[b]{0.3\textwidth}
        \centering
        \includegraphics[width=\textwidth]{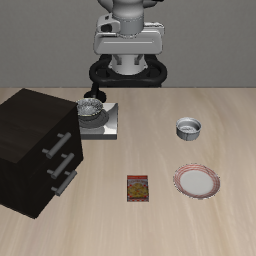}
        \caption{}
        \label{fig:c1}
    \end{subfigure}
    \begin{subfigure}[b]{0.3\textwidth}
        \centering
        \includegraphics[width=\textwidth]{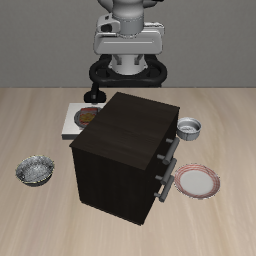}
        \caption{}
        \label{fig:c2}
    \end{subfigure}
    \begin{subfigure}[b]{0.3\textwidth}
        \centering
        \includegraphics[width=\textwidth]{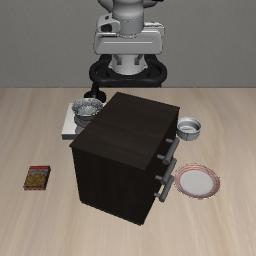}
        \caption{}
        \label{fig:c3}
    \end{subfigure}
    \caption{CBWM tasks designed for testing the impact of concepts on adaptation.}
    \label{fig:3x3}
\end{figure}

Here we describe the tasks implemented in LIBERO's BDDL description language for use in the Robosuite manipulation environment for testing knowledge preservation with Concept Bottleneck World Models.

To evaluate knowledge preservation in online test-time adaptation, we designed a series of test environments in Robosuite based on three core scenes from the LIBERO dataset. Each core scene was modified to create related scenes that introduce specific types of novelties, allowing us to systematically assess how well different approaches preserve and adapt knowledge.

The first set of scenes builds on LIBERO\_OBJECT\_SCENE. The base scene (LIBERO\_OBJECT\_SCENE\_pick\_up\_the\_tomato\_sauce\_and\_place\_it\_in\_the\_basket) tasks the agent with moving a tomato sauce bottle to a basket. We created three variants: LIBERO\_OBJECT\_SCENE\_pick\_up\_the\_milk\_and\_place\_it\_in\_the\_basket, LIBERO\_OBJECT\_SCENE\_pick\_up\_the\_butter\_and\_place\_it\_in\_the\_basket, and LIBERO\_OBJECT\_SCENE\_pick\_up\_the\_tomato\_sauce\_and\_place\_it\_in\_the\_basket\_starting\_with\_plate\_on\_top. The first two variants test adaptation to different target objects while maintaining the same basic task structure. The third variant introduces a barrier novelty by requiring the robot to unstack objects before completing the original task.

The second group derives from LIBERO\_LIVING\_ROOM\_SCENE2, where the base task (LIVING\_ROOM\_SCENE2\_put\_the\_tomato\_sauce\_in\_the\_basket) involves putting tomato sauce in a basket. We developed two variations to test adaptation to increased task complexity: LIVING\_ROOM\_SCENE2\_put\_the\_tomato\_sauce\_in\_the\_basket\_starting\_in\_another\_basket, which requires the robot to navigate around the basket's edges during grasping, and LIVING\_ROOM\_SCENE2\_put\_the\_tomato\_sauce\_in\_the\_basket\_where\_sauce\_spawns\_farther, which tests adaptation to spatial changes by placing the target object farther from the robot's initial position.

The third set extends LIBERO\_SPATIAL\_SCENE. The base task (LIBERO\_SPATIAL\_SCENE\_pick\_up\_the\_black\_bowl\_on\_the\_stove\_and\_place\_it\_on\_the\_plate) requires picking up a black bowl from a stove and placing it on a plate, modified from the original LIBERO scene by removing a duplicate black bowl to avoid ambiguity. We created three variants: LIBERO\_SPATIAL\_SCENE\_pick\_up\_the\_cookies\_on\_the\_stove\_and\_place\_it\_on\_the\_plate replacing the black bowl with cookies, LIBERO\_SPATIAL\_SCENE\_pick\_up\_the\_black\_bowl\_on\_the\_stove\_and\_place\_it\_on\_the\_plate\_blocked\_by\_cabinet adding a cabinet that blocks direct access to the black bowl, and LIBERO\_SPATIAL\_SCENE\_pick\_up\_the\_cookies\_on\_the\_stove\_and\_place\_it\_on\_the\_plate\_blocked\_by\_cabinet combining both changes. These variations test adaptation to both semantic changes (different target objects) and geometric changes (obstacle avoidance) independently and in combination.

This collection of scenes enables us to evaluate how well agents preserve their knowledge across different types of novelties: semantic changes in target objects, increases in task complexity, and modifications to the spatial layout of the environment. Each variant was carefully designed to isolate specific aspects of adaptation while maintaining enough similarity to the base scene to make knowledge transfer beneficial.

\end{theappendices}

%% file: vita.tex

\begin{vita}

Jonathan Clifford Balloch was born in 1989 to parents Hugh and Susan Balloch in New York, NY. 
At the time of writing, Jonathan resides in Atlanta, Georgia with his wife Yelena and their daughters Mariana and Natalia. 

\end{vita}